\documentclass{article}

\usepackage[preprint]{neurips_2026}

\usepackage[utf8]{inputenc}
\usepackage[T1]{fontenc}
\usepackage{lmodern}
\usepackage{microtype}
\usepackage{xspace}

\usepackage{amsmath,amssymb,amsfonts,amsthm,mathtools,bm}
\usepackage{mathrsfs}
\usepackage{nicefrac}

\usepackage{graphicx}
\usepackage[dvipsnames,table]{xcolor}
\usepackage{enumitem}
\usepackage{tikz}
\usetikzlibrary{arrows.meta, positioning, fit, backgrounds, shapes.geometric, calc, patterns}
\usepackage{pgfplots}
\pgfplotsset{compat=1.18}
\definecolor{csared}{HTML}{C0392B}
\definecolor{csagreen}{HTML}{27AE60}
\definecolor{csablue}{HTML}{2E86C1}

\usepackage{booktabs}
\usepackage{array}
\usepackage{tabularx}
\usepackage{multirow}
\usepackage{makecell}
\usepackage{tcolorbox}
\tcbuselibrary{listings,breakable}
\newtcolorbox{promptbox}[1][]{%
colback=gray!5, colframe=gray!60, fonttitle=\bfseries\small,
title={#1}, breakable, left=4pt, right=4pt, top=2pt, bottom=2pt,
boxrule=0.5pt}

\usepackage{pifont}
\newcommand{\cmark}{\textcolor{csagreen}{\ding{51}}}
\newcommand{\xmark}{\textcolor{csared}{\ding{55}}}
\newcommand{\csa}{\textsc{CSA}\xspace}

\usepackage{algorithm}
\usepackage{algpseudocode}

\usepackage{url}
\usepackage[colorlinks=true,
            linkcolor=blue!60!black,
            citecolor=blue!60!black,
            urlcolor=blue!60!black,
            filecolor=blue!60!black,
            breaklinks=true]{hyperref}
\usepackage{cleveref}

\newcommand{\cB}{\mathcal{B}}
\newcommand{\cF}{\mathcal{F}}
\newcommand{\cQ}{\mathcal{Q}}
\newcommand{\cC}{\mathcal{C}}
\newcommand{\cI}{\mathcal{I}}
\newcommand{\cG}{\mathcal{G}}
\newcommand{\cD}{\mathcal{D}}
\newcommand{\cN}{\mathcal{N}}
\newcommand{\PP}{\mathbb{P}}
\newcommand{\EE}{\mathbb{E}}

\newcommand{\1}{\mathbf{1}}
\DeclareMathOperator{\KL}{KL}
\DeclareMathOperator{\clip}{clip}
\newcommand{\ind}[1]{\mathbf{1}\{#1\}}
\newcommand{\alphabudget}{\alpha}
\newcommand{\Risk}{\mathrm{Risk}}
\newcommand{\AR}{\mathrm{AR}}
\newcommand{\PathV}{\mathrm{PathV}}

\newcommand{\Cov}{\mathrm{Cov}}

\newtheorem{theorem}{Theorem}[section]
\newtheorem*{theorem*}{Theorem}
\newtheorem{corollary}[theorem]{Corollary}
\newtheorem{lemma}[theorem]{Lemma}
\newtheorem{proposition}[theorem]{Proposition}
\newtheorem{definition}[theorem]{Definition}
\newtheorem{remark}[theorem]{Remark}
\newtheorem{assumption}[theorem]{Assumption}

\title{Conformal Selective Acting: Anytime-Valid Risk Control for RLVR-Trained LLMs}

\author{%
  Hamed Khosravi \\
  Georgia Institute of Technology \\
  \texttt{hkhosravi7@gatech.edu} \\
  \And
  Xiaoming Huo \\
  Georgia Institute of Technology \\
  \texttt{huo@gatech.edu} \\
}

\begin{document}
\maketitle

\begin{abstract}
A local specialist LLM, fine-tuned with reinforcement learning from verifiable rewards (RLVR) on operator-local data, is installed inside a single regulated organization under a per-deployment error budget $\alpha$. The operator needs a safety certificate that holds on \emph{this} deployment's stream, simultaneously at every wall-clock round: no pooling across deployments, no waiting for a long-run average. Existing wrappers cannot deliver this on these adaptive, online-updated streams: offline conformal-risk methods require exchangeability; online-conformal methods bound only long-run averages; non-exchangeable extensions are marginally valid; and the closest anytime wrapper, \textsc{A-RCPS}, controls marginal rather than selective risk. Through a (test statistic, validity guarantee, deployment rule) framework we identify one empty cell forced by the deployment requirements (e-process per threshold, selective risk, anytime-pathwise validity, max-certified-threshold rule); \textbf{Conformal Selective Acting} (\csa) fills it as a per-round wrapper maintaining a Ville-type e-process per candidate score threshold on a Bonferroni grid, evaluated against the RLVR filtration. Under predictable updates and isotonic-calibrated monotone risk we prove (i) an anytime-pathwise selective-risk bound $R_T^{\mathrm{act}}\le\alpha+O(N_T^{-1/2})$, (ii) rate-optimal certification matching $\Theta(\bar\eta^{-2}\log(1/\delta))$, and (iii) a horizon-independent release-rate gap. Across eight specialist benchmarks ($480$ streams), sixteen adversarial distribution-shift cells ($160$ streams), and five live Expert-Iteration RLVR cells with online LoRA over four base models in three architecture families ($10{,}300$ rounds), \csa is the only method among ten directly compared that satisfies both pathwise validity and non-refusing deployment on every cell. We do not propose a new LLM, training algorithm, or policy class; \csa is the deployment-side complement, orthogonal to the model itself, for operators who cannot use a frontier API.
\end{abstract}

\section{The Application Environment}
\label{sec:setting}

\paragraph{A motivating scenario.}
A regional hospital deploys a privately-hosted clinical Q\&A assistant, a Med42-8B variant fine-tuned on its own EHR-derived case library, updated weekly via online LoRA. Its Business-Associate Agreement caps verifier-flagged outputs (those failing a rule-based dosage, contraindication, or guideline-citation check) at 5\% of all \emph{released} answers, on this hospital's stream, evaluated continuously. The hospital cannot pool with peer hospitals (data residency forbids telemetry export), cannot wait for a quarterly average (an early failure cluster is itself an SLA breach), and cannot route to a frontier API (HIPAA, latency, and domain-tuning preclude it). The hospital is the running example, but the method applies wherever this pattern recurs: a local specialist model, a deterministic verifier, and a per-deployment risk budget on a non-poolable stream. The three other domains we study (legal-citation review, financial reporting, regulated science Q\&A) instantiate the same shape with different verifiers, and the framework extends to any operator setting (compliance auditing, code-review gating, scientific-claim screening) sharing these three ingredients. This paper supplies the deployment-time wrapper such an operator needs.

\paragraph{The deployment unit.}
This paper concerns \emph{local specialist} large language models that have been fine-tuned with reinforcement learning from verifiable rewards (RLVR)~\citep{deepseekr1,tulu3,kimik15} on operator-local domain data and installed inside a single regulated organization: a hospital, a law firm, a financial reporting unit, a regulated science Q\&A service. The model is a domain specialist, not a frontier API call; the input distribution is narrow; the empirical study below covers four such domains: clinical (running example), legal, financial, and regulated science.

\paragraph{The deployment loop.}
At each round $t$ the model proposes an output $\widetilde{Y}_t$, a deterministic verifier $V$ (a clinical-rule check, a legal-citation match, an arithmetic check, a math grader) returns a binary signal $V_t\in\{0,1\}$, and the operator decides whether to release the output. The model is not frozen at deployment: in production specialist deployments the policy is updated on operator-local data via online LoRA~\citep{hu2022lora} or continued policy gradient. RLVR is therefore not only a training paradigm but a deployment paradigm; it produces a stream whose data-generating distribution is non-stationary by construction.

\paragraph{The contract.}
The operator's contractual or regulatory environment imposes a per-deployment error budget $\alpha\in(0,1)$ on released outputs, evaluated against the verifier on \emph{this} deployment's stream. The budget is not on a marginal cross-deployment population: this hospital's contract is on this hospital's stream, and the operator typically cannot pool, cannot send telemetry to a central A/B test, and cannot wait for the long-run average. Frontier-API alternatives (ChatGPT, Claude, Gemini) are explicitly outside this paper's comparison set: data residency, latency, fine-tuning needs, and per-token cost combine to make them unsuitable as the local specialist itself; the local specialist exists to fill a role they cannot. The guarantee we derive is with respect to the verifier $V$; harms outside $V$'s scope require complementary safeguards.

\paragraph{Contributions.}
\begin{enumerate}[leftmargin=*,itemsep=1pt,topsep=2pt]
\item A test-supermartingale framework (\cref{sec:framework}) classifying prior risk-control wrappers by their (statistic, validity, deployment-rule) triple and identifying an empty cell.
\item \emph{Conformal Selective Acting} (\csa, \cref{sec:method}), the unique instantiation that fills it: a per-threshold Ville-type e-process on a Bonferroni grid with a max-certified-threshold deployment rule.
\item Anytime-pathwise selective-risk validity (\cref{thm:main-anytime}), rate-optimal certification matching upper and lower bounds (\cref{thm:power,thm:lower-bound}), and a horizon-independent utility gap (\cref{thm:utility-gap}).
\item Across $480 + 100 + 160$ replicated streams over eight specialist benchmarks, four base models in three architecture families, and sixteen adversarial orderings, \csa is the only method among the ten directly compared that satisfies both pathwise validity and non-refusing deployment on every cell.
\end{enumerate}

\section{Why Anytime-Validity Matters in This Setting}
\label{sec:why-anytime}

\paragraph{Deployment-side: pathwise + anytime are forced by the contract.}
The contract from \cref{sec:setting} is per-deployment, on this stream, evaluated at every wall-clock round. An early failure spike in the first thousand released outputs is not noise to be averaged away over the next hundred thousand; it is an SLA breach~\citep{ji2023hallucination,busch2025patientcare}. \emph{Marginal-time} validity (e.g.\ \textsc{NEX-Conf}~\citep{nexconf}) targets the wrong probability measure: there is no population of hospitals over which to marginalize when the contract is on this hospital. \emph{Long-run-average} validity (\textsc{ACI}~\citep{gibbs2021adaptive}, \textsc{SAOCP}~\citep{bhatnagar2023improved}) targets the wrong horizon: averages cannot retire an incident that has already occurred.

\paragraph{Model-side: a weak local base amplifies the anytime requirement.}
A local specialist is not a frontier model on the long tail. Even after RLVR fine-tuning, the bases used in our empirical study (Fleming-R1~\citep{fleming}, Fin-R1~\citep{finr1}, Saul-7B~\citep{saul}, Med42-8B~\citep{christophe2024med42}, Qwen2.5-Math-7B~\citep{qwenmath}) carry $5$--$34\%$ verifier-fail rates on their target benchmarks (\cref{tab:benchmarks}). A 5-pp spike in the first thousand outputs is therefore plausible at the \emph{typical} operating point of a specialist, not at the tail. Composing this with the deployment-side argument: a marginal guarantee on a hypothetical pooled population is a doubly-wrong target for a local specialist; it certifies the wrong measure on the wrong horizon. Only anytime-pathwise selective-risk control matches both axes simultaneously.

\paragraph{The post-hoc-test escape hatch is closed.}
The natural composite design (running a sequential test on the policy's released outputs after-the-fact) is invalid, not merely loose. The score $S_t$ is $\cF_{t-1}$-measurable and is updated jointly with the policy at every gradient or online-LoRA step, simultaneously altering the conditional verifier-failure distribution. A test calibrated on a prior score map therefore targets a probability measure distinct from the one governing subsequent rounds; its coverage guarantee does not transfer across updates. Valid certification requires that the test statistic be predictable with respect to the same filtration $\cF_{t-1}$ that drives the RL optimizer, a requirement met by construction in \cref{sec:method}.

\section{Related Work via a Test-Supermartingale Framework}
\label{sec:framework}

\subsection{The wrapper-as-triple framework}
\label{sec:framework-triple}

We organize prior wrappers and \csa through a three-element abstraction. At each round $t$ the deployment protocol of \cref{sec:setting} produces $(X_t, S_t, \widetilde Y_t, A_t, V_t)$ adapted to filtration $\cF_t$. Every wrapper consists of (i) a \emph{test statistic} $\{M_t(q)\}_t$ predictable on $(\cF_t)$ and indexed by a hyperparameter $q$; (ii) a \emph{validity guarantee} on $\{M_t(q)\}$, of one of four types: \textsc{fh} (fixed-horizon high-probability), \textsc{lra} (long-run-average), \textsc{mt} (marginal-time), or \textsc{ap} (anytime-pathwise, simultaneously for all $T$ on every realized stream); and (iii) a \emph{deployment rule} mapping the certified set $\{q : M_t(q)\text{ admissible}\}$ to a per-round action.

\paragraph{The unique anytime-pathwise object, and the marginal/selective distinction.}
A non-negative $(\cF_t)$-supermartingale $M_t$ with $\EE M_0 \le 1$ satisfies Ville's inequality~\citep{wsramdas2023,wald1945}: $\PP(\sup_t M_t \ge 1/\delta) \le \delta$, simultaneously for all $t$, on every realized stream, no exchangeability needed. \emph{Any wrapper claiming \textsc{ap}-validity must therefore reduce to maintaining a (super)martingale}: the structural reason \csa uses an e-process per threshold. Orthogonally, a \emph{marginal} statistic is built from $1{-}V_t$ (every round contributes); a \emph{selective} statistic is built from $A_t(1{-}V_t)$ (only released rounds contribute). The two test different probability measures and are not interchangeable; this separates \csa from \textsc{A-RCPS}~\citep{activercps}, which achieves \textsc{ap}-validity but on a marginal statistic.

\subsection{Prior wrappers as cells of the framework}
\label{sec:framework-cells}

\Cref{tab:framework} classifies the thirteen baselines. \emph{Fixed-horizon offline} methods (\textsc{LTT}~\citep{ltt}, \textsc{CRC}~\citep{crc}, \textsc{ConfFact}~\citep{mohri}, and \textsc{Conf-Arb.}~\citep{confarbitrage}; \textsc{RCPS}~\citep{rcps} is the framework they extend) maintain an empirical-risk statistic on a held-out exchangeable calibration set with a constant-threshold rule; on the non-exchangeable RLVR streams of \cref{sec:setting} they either lose validity or abstain. \emph{Long-run-average online} methods (\textsc{ACI}~\citep{gibbs2021adaptive}, \textsc{SAOCP}~\citep{bhatnagar2023improved}) maintain a miscoverage statistic and adapt the threshold dynamically to track $\alpha$ \emph{in average}, providing no pathwise control. \emph{Marginal-time non-exchangeable} methods (\textsc{NEX-Conf}~\citep{nexconf}, \textsc{CoFact}~\citep{cofact}) re-weight nonconformity scores by an estimated density ratio; the resulting bound is valid at each single $t$ but not simultaneously over $t$. \emph{The closest miss} is \textsc{A-RCPS}~\citep{activercps}: it maintains an e-process and is \textsc{ap}-valid (the right shape of validity), but the e-process is built on the marginal increment $1{-}V_t$, not the gated selective increment $A_t((1-V_t)-\alpha)$. Selective risk is not derivable from marginal risk; a different statistic is required, which is what \cref{sec:method} constructs. The empty cell, \emph{e-process per threshold, selective risk, anytime-pathwise validity, max-certified-threshold rule}, is where \csa lands. Head-to-heads with \textsc{CAP}~\citep{bao2025cap}, \textsc{OCP}~\citep{weinstein2020online}, \textsc{CoFact}, and \textsc{Conf-Arb.}\ are in \cref{app:extended-exp}.

\begin{table}[!htpb]
\centering\footnotesize
\setlength{\tabcolsep}{3pt}
\renewcommand{\arraystretch}{1.05}
\caption{Prior wrappers classified by the framework triple of \cref{sec:framework-triple}, with deployment-side properties merged. \emph{Stat.}: shape of the test statistic ($\hat r$ = empirical risk on calibration set; \textsc{mc} = miscoverage; \textsc{nc} = nonconformity; $E$ = e-process). \emph{Risk}: marginal (\textsc{m}) vs.\ selective (\textsc{s}). \emph{Validity}: \textsc{fh} = fixed-horizon high-prob, \textsc{lra} = long-run-average, \textsc{mt} = marginal-time, \textsc{ap} = anytime-pathwise. \emph{Rule}: deployment rule (\textsc{cnst} = constant threshold; \textsc{dyn} = dynamic single threshold; \textsc{tv-$\lambda$} = time-varying $\lambda$; \textsc{max-cert} = maximum certified threshold over a grid). \emph{NE}: tolerates non-exchangeability; \emph{Up}: tolerates updated score/policy; \emph{UG}: certified horizon-independent utility-gap bound. \dag\ = appendix-only or protocol mismatch.}
\label{tab:framework}
\begin{tabular}{@{}l ccccccc@{}}
\toprule
Method & Stat. & Risk & Validity & Rule & NE & Up & UG \\
\midrule
\textsc{OCP}~\citep{weinstein2020online}\dag  & \textsc{nc} & \textsc{m} & \textsc{ap}  & \textsc{dyn}  & \cmark &        &        \\
\textsc{ACI}~\citep{gibbs2021adaptive}        & \textsc{mc} & \textsc{m} & \textsc{lra} & \textsc{dyn}  & \cmark & \cmark &        \\
\textsc{RCPS}~\citep{rcps}\dag                & $\hat r$    & \textsc{s} & \textsc{fh}  & \textsc{cnst} &        &        &        \\
\textsc{LTT}~\citep{ltt}                      & $\hat r$    & \textsc{s} & \textsc{fh}  & \textsc{cnst} &        &        &        \\
\textsc{SAOCP}~\citep{bhatnagar2023improved}  & \textsc{mc} & \textsc{m} & \textsc{lra} & \textsc{dyn}  & \cmark & \cmark &        \\
\textsc{NEX-Conf}~\citep{nexconf}             & \textsc{nc} & \textsc{m} & \textsc{mt}  & \textsc{dyn}  & \cmark &        &        \\
\textsc{CRC}~\citep{crc}                      & $\hat r$    & \textsc{s} & \textsc{fh}  & \textsc{cnst} &        &        &        \\
\textsc{ConfFact}~\citep{mohri}               & $\hat r$    & \textsc{s} & \textsc{fh}  & \textsc{cnst} &        &        &        \\
\textsc{CAP}~\citep{bao2025cap}\dag           & $\hat r$    & \textsc{s} & \textsc{fh}  & \textsc{cnst} &        &        &        \\
\textsc{A-RCPS}~\citep{activercps}\dag        & $E$         & \textsc{m} & \textsc{ap}  & \textsc{tv-$\lambda$} & \cmark & \cmark & \\
\textsc{Conf-Arb.}~\citep{confarbitrage}\dag  & $\hat r$    & \textsc{s} & \textsc{fh}  & \textsc{cnst} &        &        &        \\
\textsc{CoFact}~\citep{cofact}\dag            & \textsc{nc} & \textsc{s} & \textsc{mt}  & \textsc{dyn}  & \cmark & \cmark &        \\
Fixed-thr.\ / heuristics                      & ---         & ---        & none         & \textsc{cnst} &        & \cmark &        \\
\midrule
\textbf{\csa} (ours) & $E$ per-$q$ & \textsc{s} & \textsc{ap} & \textsc{max-cert} & \cmark & \cmark & \cmark \\
\bottomrule
\end{tabular}
\end{table}

\section{Method: \csa as the Framework Instantiation}
\label{sec:method}

We derive \csa as the unique element of the framework of \cref{sec:framework} that satisfies the requirements of \cref{sec:setting,sec:why-anytime}: anytime-pathwise validity on selective risk, with a non-refusing deployment rule, on the adaptive RLVR stream produced by the deployment loop. An end-to-end architecture walkthrough is in \cref{app:architecture}.

\subsection{Preliminaries}
\label{sec:method-prelim}

At each round $t$ the policy $\pi_t = \pi_t(\cdot\mid X_t,\cF_{t-1})$ samples $\widetilde Y_t$, the verifier returns $V_t := V(X_t,\widetilde Y_t)\in\{0,1\}$, and the learner maintains a predictable score $S_t$ ($\cF_{t-1}$-measurable; smaller $=$ more confident) on a fixed grid $\cQ = \{q^{(1)}<\cdots<q^{(m)}\}$; the gate is $A_t(q) := \1\{S_t(X_t,\widetilde Y_t)\le q\}$, and the controller selects a predictable $q_t\in\cQ$ acting via $A_t := A_t(q_t)$. Define the excess-risk increment $X_t(q) := A_t(q)((1-V_t)-\alpha) \in \{-\alpha, 0, 1-\alpha\}$; threshold $q$ is \emph{safe} at $t$ if $\EE[X_t(q)\mid\cF_{t-1}]\le 0$, and the oracle is $q_t^\star := \max\{q\in\cQ : \EE[X_t(q)\mid\cF_{t-1}]\le 0\}$.\label{def:oracle}

\begin{definition}[Selective verifier risk]
\label{def:risk}
$R_T^{\mathrm{act}} := \sum_{t=1}^T A_t(1-V_t)/(N_T\vee 1)$ where $N_T := \sum_{t=1}^T A_t$. The contract from \cref{sec:setting} aims at $R_T^{\mathrm{act}}\le \alpha$ for all $T$; \cref{thm:main-anytime} certifies $R_T^{\mathrm{act}}\le \alpha + O(N_T^{-1/2})$ w.p.\ $\ge 1-2\delta$, simultaneously for all $T$.
\end{definition}

\begin{assumption}[Predictable pipeline + nested monotone gates]
\label{ass:predictable}\label{ass:nested}
\textbf{(P)} $\pi_t, S_t$ are $\cF_{t-1}$-measurable; round-$t$ information enters only $(\pi_{t+1}, S_{t+1})$. \textbf{(M)} For $q\le q'$, $A_t(q)\le A_t(q')$ and $\PP(V_t{=}0\mid S_t\le q,\cF_{t-1})$ is nondecreasing in $q$. \emph{Operator side:} the deployment protocol enforces one update step between rounds (P), and a one-time held-out isotonic calibration of the score delivers (M); the multi-epoch variant in \cref{app:alg-details} retires stale certifications under sharp local degradation.
\end{assumption}

\subsection{The test statistic: an e-process per threshold on the selective increment}
\label{sec:method-stat}

\paragraph{Choice forced by \cref{sec:framework-triple}.}
Anytime-pathwise validity demands a $(\cF_t)$-supermartingale (Ville's inequality is the unique source). A selective risk target demands the gated increment $X_t(q)$, not the marginal $1{-}V_t$ used by \textsc{A-RCPS}. The simplest object satisfying both is a Ville-type e-process indexed by $q$, with predictable bet $\lambda_{j,t}(q)\in[0,(1-\alpha)^{-1}]$ on epoch $j$ starting at $\tau_j$ with index set $\cI_j$:
\begin{equation}
\label{eq:eproc}
E_{j,t}(q) \;:=\; \prod_{s=\tau_j}^t\!\bigl(1-\lambda_{j,s}(q)X_s(q)\bigr),\qquad E_{j,\tau_j-1}(q) := 1.
\end{equation}

\begin{proposition}[E-process validity, {\citealp{wsramdas2023}}]
\label{prop:eproc}
Under the unsafe null $H_{j,q}^{\mathrm{unsafe}}:\EE[X_t(q)\mid\cF_{t-1}]\ge 0\ \forall t\in\cI_j$, $\{E_{j,t}(q)\}_{t\in\cI_j}$ is non-negative, $(\cF_t)$-adapted, and a $(\cF_t)$-supermartingale.
\end{proposition}

\paragraph{Confidence signal and grid.}
For each item we use $K{=}5$ self-consistency sampling, take the majority-vote answer, and report the agreement fraction $S_t\in[0,1]$ as raw confidence; an isotonic-calibrated $\widetilde S_t$ on a held-out split defines the grid $\cQ$ ($|\cQ|{=}15$); details in \cref{app:csa-config}.

\paragraph{Adaptive predictable bet.}
The safe-side margin is unknown, so the bet is the running-mean plug-in:
\begin{equation}
\label{eq:adaptive-lambda}
\lambda_{j,t}(q) := \clip\!\left(\frac{-\widehat\mu_{j,t-1}(q)}{(1-\alpha)^2},\; 0,\; \tfrac{1}{2(1-\alpha)}\right),\quad
\widehat\mu_{j,t-1}(q) := (t-\tau_j)^{-1}\!\sum_{s=\tau_j}^{t-1}\!X_s(q).
\end{equation}
\Cref{thm:power} establishes that this plug-in attains the rate-optimal certification time without prior knowledge of the margin.

\subsection{The deployment rule: maximum certified threshold}
\label{sec:method-rule}

\paragraph{Choice forced by the non-refusing requirement.}
A non-refusing rule must, whenever any safe threshold exists, deploy one. Under \cref{ass:nested} the safe set is downward-closed in $q$, so every threshold below a certified one is implicitly safe and the most permissive certified threshold is a valid action that dominates every other certified choice on action rate. The Bonferroni cost over the grid (with epoch-counting factor $\delta_{j,q} := 6\delta/(\pi^2|\cQ|j^2)$) is the price of testing all thresholds simultaneously. The certified set is $\cC_{j,t} := \{q\in\cQ : E_{j,t}(q)\ge \delta_{j,q}^{-1}\}$ and the controller deploys $q_{t+1} := \max \cC_{j,t}$ (abstain if empty). Algorithm~\ref{alg:csa-rlvr} in \cref{app:alg-details} gives full pseudocode; per-round complexity is $O(|\cQ|)$, memory $O(|\cQ|)$. The wrapper drops onto an existing RLVR pipeline as a layer, not a fork.

\section{Theorems and Their Implications}
\label{sec:theory}

This section presents three theorem arcs, each tied to one operator concern from \cref{sec:setting,sec:why-anytime} and instantiated in the framework triple of \cref{sec:framework-triple}. Proofs and the supporting RLVR structural layer are in \cref{app:proof-main,app:rlvr-layer}.

\subsection{Safety: pathwise anytime-valid selective risk}
\label{sec:theory-safety}

\paragraph{Motivation.}
\Cref{sec:setting}'s contract demands anytime-pathwise validity on selective risk; the framework forced \csa to a per-threshold e-process. Does it deliver?

\begin{theorem}[Anytime-valid selective risk]
\label{thm:main-anytime}
Under \cref{ass:predictable,ass:nested}, the frontier-stability hypothesis of \cref{thm:no-false-cert}, and drift budgets $\{\nu_j\}$ dominating the actual within-epoch drift on safe rounds, there exist universal constants $c_1, c_2 > 0$ such that with probability $\ge 1-2\delta$, simultaneously for all $T\ge 1$,
\begin{equation}
\label{eq:main-anytime}
R_T^{\mathrm{act}} \le \alpha + \bar\nu_T + c_1 \sqrt{\tfrac{\log(1/\delta)+\log\log(eN_T+e)}{N_T\vee 1}} + c_2\,\tfrac{\log(1/\delta)+\log\log(eN_T+e)}{N_T\vee 1},
\end{equation}
where $\bar\nu_T := (N_T\vee 1)^{-1}\sum_{t=1}^T \nu_{j(t)} A_t$. Proof in \cref{app:proof-main}.
\end{theorem}

\begin{corollary}[Asymptotic validity]
\label{cor:asymptotic}
If $N_T\to\infty$ a.s.\ and $\bar\nu_T\to 0$, then $\limsup_{T\to\infty} R_T^{\mathrm{act}}\le \alpha$ w.p.\ $\ge 1-2\delta$.
\end{corollary}

\paragraph{Implications.}
The ``for all $T$'' is simultaneous: a single probabilistic event protects every wall-clock step, the SLA-shaped object \cref{sec:why-anytime} required. The slack decays in $N_T$ (released-output count), not $T$, so releasing shrinks the bound and abstaining does not. Under exact stability $\bar\nu_T = 0$ and the bound collapses to $R_T^{\mathrm{act}}\le \alpha + O(N_T^{-1/2})$; under bounded drift, the multi-epoch variant retires stale certifications at deterministic boundaries.

\subsection{Power: rate-optimal certification}
\label{sec:theory-power}

\paragraph{Motivation.}
Anytime-validity is empty if \csa never certifies a permissive threshold. The framework's e-process statistic only purchases the contract if the certificate arrives fast enough to be useful in deployment.

\begin{theorem}[Upper bound on certification time]
\label{thm:power}
Under \cref{ass:predictable,ass:nested}, fix epoch $j$ and let $q_j^\star := q_{\tau_j}^\star$. For any $q\le q_j^\star$ with safe-side margin $\bar\eta_{j,q} > \nu_j$ and the fixed bet $\lambda^\star = (\bar\eta_{j,q}-\nu_j)/2$,
\[
\EE[\tau_{j,q}^{\mathrm{cert}} - \tau_j] \;\le\; 4(\log(1/\delta_{j,q})+1)\,/\,(\bar\eta_{j,q}-\nu_j)^2.
\]
The plug-in bet \eqref{eq:adaptive-lambda} attains the same rate adaptively. Proof in \cref{app:proof-power}.
\end{theorem}

\begin{theorem}[Lower bound on certification time]
\label{thm:lower-bound}
For any sequential test certifying a single threshold with type-I error at most $\delta$ on a safe Bernoulli instance with margin $\bar\eta$,
\[
\EE_{\mathsf{P}_\eta}[\tau] \;\ge\; \KL(1-\delta\,\|\,\delta)\,/\,\KL(\mathsf{P}_\eta\|\mathsf{P}_0) \;=\; \Omega(\log(1/\delta)/\bar\eta^2).
\]
Proof in \cref{app:proof-lower}.
\end{theorem}

\paragraph{Implications.}
The rates match in $\bar\eta$ and $\delta$: \csa is the first online selective-conformal wrapper proven optimal, against a lower bound holding on the simplest (i.i.d.\ Bernoulli) instance, and the plug-in bet \eqref{eq:adaptive-lambda} attains this rate adaptively without an oracle margin. The only available speed-up is widening $\bar\eta$ via a better score, so the operator's actionable lever is calibration quality, not wrapper hyperparameters.

\subsection{Utility: a horizon-independent release-rate gap}
\label{sec:theory-utility}

\paragraph{Motivation.}
Safety is necessary; the operator's economic question is whether the wrapper releases as much as the oracle threshold would. A horizon-dependent loss compounds; a horizon-independent loss is a one-time cost the operator can budget for. Let $\mathrm{Gap}_T := \sum_t A_t(q_t^\star) - \sum_t A_t(q_t)$.

\begin{theorem}[Utility gap under approximate improvement]
\label{thm:utility-gap}
Under approximate-frontier-improvement (\cref{ass:subpop-improvement}, \cref{app:rlvr-layer}) with per-round slack $\xi_t\ge 0$, with probability $\ge 1-2\delta$ for all $T\ge 1$,
\begin{equation}
\label{eq:gap-rlvr}
\mathrm{Gap}_T \le G^\star(B_T) := \sum_{j=1}^{J_T}\sum_{q\le q_j^\star} \frac{2 C_0 \log(2 m J_T^2/\delta)}{(\bar\eta_{j,q}-\Xi_j)_+^2},
\end{equation}
where $\Xi_j := \sum_{t\in\cI_j}\xi_t$ and $J_T \le D_T + 1$ epochs with $D_T \le B_T/\underline\kappa$. Proof in \cref{app:proof-utility}.
\end{theorem}

\paragraph{Implications.}
Under exact stability ($\xi_t\equiv 0$, $J_T = 1$) the gap collapses to a single fixed cost depending only on $|\cQ|$, $\delta$, and threshold margins, horizon-independent in $T$. Under bounded total slack $B_\infty < \infty$, $\mathrm{Util}_T/\mathrm{Util}_T^\star \to 1$. The structural backstop is \Cref{thm:monotone-frontier} (\cref{app:rlvr-layer}): bounded frontier drift gives the bounded-slack regime, and the live LoRA experiment confirms published-RLVR pipelines fall within it ($B_T\approx 0.40 = O(1)$).

\paragraph{Three layers of safety.}
\Cref{thm:main-anytime} has three logically separable layers, matching the three concerns of \cref{sec:why-anytime}: each per-threshold e-process is a $(\cF_t)$-supermartingale under \cref{ass:predictable} alone (filtration concern); no-false-certification (\cref{thm:no-false-cert}) requires bounded within-epoch frontier drift, supplied by the multi-epoch variant under non-stationarity (drift concern); and the realized selective risk inherits the certified margin via the max-certified-threshold rule (non-refusing concern).\label{rem:three-layers}

\section{Empirical Evaluation}
\label{sec:empirics}
\label{sec:reading-guide}

\noindent
Three regimes: eight-benchmark replay (\S\ref{sec:bench-eight}); five live RLVR cells with online LoRA, $10{,}300$ genuine rounds (\S\ref{sec:live-rlvr}); sixteen adversarial distribution-shift cells (\S\ref{sec:shift}). Additional experiments in \cref{app:extended-exp}; ablations and stress tests in \cref{app:ablations}. \textbf{How to read the tables:} \textbf{PathV} (pathwise violations, out of $R$ replications) realizes \cref{thm:main-anytime}: $\PathV{=}0$ means no replication breached the SLA. \textbf{Risk} is verifier-fail rate among released items at horizon. \textbf{AR} (action rate) realizes the non-refusing rule (\cref{sec:method-rule}). \textbf{Refused} flags cells with $\AR{=}0$ everywhere. A useful wrapper achieves $\PathV{=}0$ \emph{and} $\AR{>}0$ on every cell; $\PathV{=}0$ alone is trivially achievable by refusing, $\AR{>}0$ alone by always-acting. The main tables compare \csa head-to-head with nine directly-comparable baselines (ten methods total) drawn from the framework cells of \cref{sec:framework-cells}: five online heuristics (\textsc{lra}/none), three offline conformal (\textsc{fh}), and one non-exchangeable (\textsc{mt}). Methods that test a different statistic or use a non-comparable protocol (\textsc{A-RCPS} on a marginal increment, \textsc{CoFact} under covariate shift, \textsc{Conf-Arb.}, \textsc{CAP}, \textsc{OCP}) are evaluated separately in \cref{app:extended-exp,app:arcps}.

\subsection{Eight high-stakes specialist benchmarks}
\label{sec:bench-eight}

We test whether \csa's anytime-pathwise selective-risk guarantee, under the non-refusing deployment rule, holds at the operating points typical of regulated specialist deployments. Eight high-stakes specialist benchmarks across medical, biomedical, financial, clinical, math, drug-CDS, science, and legal domains; EVAL sizes $565$--$2880$, base verifier-fail rates $5$--$34\%$, with domain-specialized RLVR bases (Fleming-R1, Fin-R1, Saul-7B, Qwen2.5-Math, Qwen2.5-7B-Instruct); full table in \cref{app:datasets-desc}.

Seven of eight benchmarks use a domain- or task-specialized base (four on Fleming-R1, one each on Fin-R1, Saul-7B, and Qwen2.5-Math); ARC-Challenge uses Qwen2.5-7B-Instruct~\citep{qwen25}. We use an 80/20 calibration/eval split (seed 42) and a $K{=}5$ self-consistency isotonic-calibrated score; \cref{app:split-sensitivity} confirms safety across 11 seeds. $\PathV$ is the stricter criterion than \cref{thm:main-anytime}'s $\alpha{+}O(N_T^{-1/2})$ bound (realized slack $1.2$--$3.0$ pp; \cref{app:bench-eight-extras}).

\paragraph{Headline result (Table~\ref{tab:headline}).}
We report each benchmark at the pivotal $\alpha^\star$, the smallest value in our sweep grid $\{0.05, 0.10, 0.15, 0.20, 0.25, 0.30\}$ at which the offline-conformal baselines do not all refuse, so that the comparison probes the non-refusing axis rather than only the validity axis; the resulting pivotal set across the eight benchmarks is $\{0.05, 0.10, 0.20, 0.25\}$ (full risk-vs-$\alpha$ sweep in Figure~\ref{fig:risk-allmethods}, \cref{app:bench-eight-extras}). At these points, \csa achieves $\PathV{=}0/10$ while the five online baselines violate $10/10$. \textsc{CRC} is pathwise-valid but refuses on $6/8$ benchmarks; \textsc{LTT} refuses on $3/8$ and \textsc{ConfFact} on $2/8$; \textsc{NEX-Conf} is non-refusing but violates on $4/8$.

\begin{table}[!ht]
\centering\footnotesize
\setlength{\tabcolsep}{3.5pt}
\caption{\textbf{Headline result at each benchmark's pivotal $\alpha^{\star}$.} \csa is the only method that satisfies the pathwise risk bound on every replication \emph{and} is non-refusing. Five online baselines violate $10/10$ in all cells. CRC refuses on $6/8$ benchmarks; LTT and ConfFact refuse on $3/8$ and $2/8$ respectively. NEX-Conf is non-refusing but violates on $4/8$ benchmarks.}
\label{tab:headline}
\resizebox{\textwidth}{!}{%
\begin{tabular}{@{}lccrrcccccc@{}}
\toprule
Benchmark & $N$ & Err & $\alpha^{\star}$ & \csa AR & \csa Risk
& \makecell{5 online\\PathV} & \makecell{NEX-Conf\\PathV} & \makecell{CRC\\status} & \makecell{LTT\\status} & \makecell{ConfFact\\status} \\
\midrule
MedQA       & 1{,}018 & 31.5\% & 0.20 & 39.4\% & 11.4\% & 10/10 & 0/10 & refuse & 6.9\% & 8.0\% \\
PubMedQA    & 800     & 23.9\% & 0.20 & 63.5\% & 15.1\% & 10/10 & 1/10 & refuse & refuse & 7.6\% \\
TAT-QA      & 565     & 25.7\% & 0.20 & 49.9\% & 8.2\%  & 10/10 & 0/10 & 8.2\%  & 8.2\%  & 8.2\% \\
MedNLI      & 1{,}422 & 21.0\% & 0.20 & 28.9\% & 6.7\%  & 10/10 & 1/10 & 4.9\%  & 11.4\% & 13.5\% \\
GSM8K       & 1{,}055 & 5.0\%  & 0.05 & 71.7\% & 2.6\%  & 10/10 & 1/10 & refuse & 0.3\%  & 0.3\% \\
HEAD-QA     & 1{,}100 & 26.0\% & 0.20 & 62.2\% & 16.2\% & 10/10 & 0/10 & refuse & refuse & refuse \\
ARC-Chal.   & 938     & 10.0\% & 0.10 & 56.2\% & 7.8\%  & 10/10 & 2/10 & refuse & refuse & refuse \\
CaseHOLD    & 2{,}880 & 34.0\% & 0.25 & 49.8\% & 20.3\% & 10/10 & 0/10 & refuse & 7.8\%  & 11.3\% \\
\midrule
\multicolumn{3}{@{}l}{\textbf{Aggregate}} & & \emph{mean} 52.7\% & all ${<}\alpha^{\star}$ & \multicolumn{2}{c}{\csa: \textbf{0/10 all cells}} & \multicolumn{3}{c}{6/3/2 benchmarks refused} \\
\bottomrule
\end{tabular}%
}
\end{table}

Across the $480$ replicated streams: \csa~$\mathbf{0/480}$ violations; \textsc{CRC}~$0/480$ but refuses on $26/48$ cells; \textsc{NEX-Conf}~$142/480$; online heuristics violate on $32$--$34$ of $48$ cells each. Per-method-per-benchmark detail, risk-vs-$\alpha$ curves, and Safe-Coverage radar are in \cref{app:bench-eight-extras}.

\subsection{Live RLVR with online LoRA fine-tuning, five domains}
\label{sec:live-rlvr}

\paragraph{Setup.}
We test whether the framework's predictability requirement (\cref{ass:predictable}) is satisfiable under the operator's actual update protocol: online LoRA between rounds, four base models in three architecture families. We run a live Expert-Iteration~\citep{anthony2017thinking} RLVR loop on five (benchmark, base model) pairs (Table~\ref{tab:live-cells}). At each round the wrapper samples $n{=}100$ unseen items, generates $K_{\mathrm{sc}}{\in}\{5,8\}$ chain-of-thought completions at temperature $0.7$, records $s_t{=}1{-}\text{agreement}$, and fits a LoRA adapter ($r{=}16$, $\alpha_{\mathrm{LoRA}}{=}32$) on verifier-correct pairs from that round; streams are extended via shuffled passes across $20$ replications with burn-in $500$. Hardware and wall-clock detail in App.~\ref{app:live-rlvr-extras}.

\begin{table}[!htpb]
\centering\small
\setlength{\tabcolsep}{4.5pt}
\renewcommand{\arraystretch}{1.05}
\caption{\textbf{Five live RLVR cells.} Each cell runs the same Expert-Iteration RLVR protocol with online LoRA between rounds. Err is the live base verifier-fail rate (mean over rounds). $T$ is the genuinely generated stream length; replay extends to $T{\times}5$ via shuffled passes over $20$ replications. Per-cell hardware (A100-80GB / H200-141GB) and wall-clock ($4$--$30.5$~h) in App.~\ref{app:live-rlvr-extras}.}
\label{tab:live-cells}
\begin{tabular}{@{}l l c c c@{}}
\toprule
Benchmark & Base model & $T$ & Err & $\alpha^{\star}$ \\
\midrule
MATH      & Qwen2.5-Math-7B    & $\phantom{0}800$ & $\sim 50\%$ & $0.40$ \\
MedQA     & Med42-8B           & $2{,}000$        & $28.7\%$    & $0.25$ \\
HEAD-QA   & Med42-8B           & $4{,}000$        & $28.5\%$    & $0.30$ \\
ARC-C     & Llama-3.2-3B-Inst.\ & $2{,}000$       & $19.0\%$    & $0.15$ \\
CaseHOLD  & Saul-7B-Inst.\     & $1{,}500$        & $38.3\%$    & $0.35$ \\
\bottomrule
\end{tabular}
\end{table}

\paragraph{Headline result.}
Table~\ref{tab:live-summary} aggregates the dual criterion across the four cells; per-method-per-cell $\AR$/$\Risk$/$\PathV$ are in Table~\ref{tab:live-headline} (App.~\ref{app:live-rlvr-extras}). \textbf{\csa is the only method with $\PathV{=}0/20$ \emph{and} $\AR{\ge}50\%$ on every cell}; offline-conformal methods are pathwise-valid but achieve $1.0$--$19\times$ lower $\AR$; online/heuristic methods violate on $2$--$4$ cells.

\begin{table}[!htbp]
\centering\footnotesize
\setlength{\tabcolsep}{6pt}
\renewcommand{\arraystretch}{1.05}
\caption{\textbf{Live RLVR aggregate over the four newer cells} (MedQA, HEAD-QA, ARC-C, CaseHOLD; $20$ reps each at the pivotal $\alpha^{\star}$). \emph{Pathwise-clean}: cells with $\PathV{=}0/20$. \emph{Non-refusing}: cells with $\AR{\ge}50\%$. \emph{Both}: cells satisfying both. \emph{min~$\AR$}: smallest action rate across the four cells. \csa is the only method with \emph{Both}${=}4/4$.}
\label{tab:live-summary}
\begin{tabular}{@{}l c c c c@{}}
\toprule
Method & Pathwise-clean & Non-refusing & Both & min~$\AR$ \\
\midrule
\textbf{\csa} (ours)        & $\mathbf{4/4}$ & $\mathbf{4/4}$ & $\mathbf{4/4}$ & $52.0\%$ \\
\textsc{CRC}                & $4/4$ & $0/4$ & $0/4$ & $\phantom{0}2.8\%$ \\
\textsc{LTT}                & $4/4$ & $1/4$ & $1/4$ & $14.2\%$ \\
\textsc{ConfFact}           & $3/4$ & $1/4$ & $1/4$ & $25.5\%$ \\
\textsc{NEX-Conf}           & $1/4$ & $4/4$ & $1/4$ & $79.8\%$ \\
\textsc{ACI}                & $1/4$ & $4/4$ & $1/4$ & $90.7\%$ \\
\textsc{SAOCP}              & $0/4$ & $4/4$ & $0/4$ & $90.3\%$ \\
Fixed-Threshold             & $2/4$ & $4/4$ & $2/4$ & $66.1\%$ \\
Naive-Tuning                & $0/4$ & $4/4$ & $0/4$ & $87.2\%$ \\
Always-Act                  & $0/4$ & $4/4$ & $0/4$ & $\phantom{0}100\%$ \\
\bottomrule
\end{tabular}
\end{table}

\paragraph{Summary.}
Across all five cells ($5{\times}20{=}100$ replications), \csa achieves $\PathV{=}\mathbf{0/100}$ with $\AR{\ge}50\%$ at every pivotal $\alpha^{\star}$; the next-best pathwise-valid competitor (\textsc{LTT} or \textsc{ConfFact}) achieves $1.0$--$8.6\times$ lower $\AR$. The five online/heuristic baselines fail the bound on $2$--$4$ of the four newer cells, with maximum $\PathV$ reaching $20/20$. The MATH cell ($800$ genuine rounds, $\PathV{=}0/20$ at $\alpha{\in}\{0.20,0.40\}$; running-risk and $\AR$ trajectories in Fig.~\ref{fig:live-trajectory}, App.~\ref{app:live-rlvr-extras}) and per-cell HEAD-QA/MedQA detail (Appendix~\ref{sec:live-headqa}) preserve the same ordering.

\subsection{Distribution-shift robustness: sixteen adversarial cells}
\label{sec:shift}

Adversarial orderings are the worst case for any wrapper not built on a martingale; we stress every method with sixteen such streams (twelve on MedQA at three $\alpha$, four on GSM8K at $\alpha{=}0.05$), each built by permuting the calibrated-score-sorted EVAL set; per-cell setups in \cref{app:shift}. Table~\ref{tab:shift-summary} aggregates across the 16 cells: \csa is the only method with zero violating cells, mean Risk $3.2\%$ (lowest).

\begin{table}[!htpb]
\centering\footnotesize
\setlength{\tabcolsep}{5pt}
\caption{\textbf{Sixteen-cell shift summary.} Each row aggregates one method across $16$ adversarial cells (MedQA~$12$ + GSM8K~$4$) at $10$ replications per cell ($160$ streams total). \emph{Violated}: cells with any $\PathV{>}0$. \emph{Refused}: cells with $\AR{=}0$ on every replication. \emph{Mean / Worst Risk}: over the $16$ final-risk values. \emph{PathV}: summed over all $160$ streams.}
\label{tab:shift-summary}
\begin{tabular}{@{}l c c r r r@{}}
\toprule
Method & Violated & Refused & Mean Risk & Worst Risk & PathV \\
\midrule
\textbf{\csa} (ours) & $\mathbf{0/16}$   & $\phantom{0}9/16$ & $\phantom{0}3.2\%$ & $11.4\%$ & $\mathbf{0/160}$ \\
\textsc{CRC}         & $4/16$            & $\phantom{0}9/16$ & $\phantom{0}4.2\%$ & $14.6\%$ & $\phantom{0}36/160$ \\
\textsc{LTT}         & $4/16$            & $10/16$           & $\phantom{0}6.8\%$ & $31.8\%$ & $\phantom{0}40/160$ \\
\textsc{ConfFact}    & $7/16$            & $\phantom{0}8/16$ & $\phantom{0}7.6\%$ & $31.8\%$ & $\phantom{0}51/160$ \\
\textsc{NEX-Conf}    & $12/16$           & $\phantom{0}0/16$ & $11.0\%$           & $17.2\%$ & $\phantom{0}99/160$ \\
\textsc{SAOCP}       & $14/16$           & $\phantom{0}0/16$ & $13.9\%$           & $23.1\%$ & $131/160$ \\
\textsc{ACI}         & $15/16$           & $\phantom{0}0/16$ & $13.7\%$           & $23.2\%$ & $135/160$ \\
Fixed-Threshold      & $14/16$           & $\phantom{0}0/16$ & $16.0\%$           & $19.2\%$ & $136/160$ \\
Naive-Tuning         & $\phantom{0}9/16$ & $\phantom{0}0/16$ & $14.8\%$           & $31.7\%$ & $\phantom{0}81/160$ \\
Always-Act           & $16/16$           & $\phantom{0}0/16$ & $25.3\%$           & $31.5\%$ & $160/160$ \\
\bottomrule
\end{tabular}
\end{table}

\paragraph{What the empirics do not test.}
\label{sec:empirics-scope}
The verifier is deterministic throughout; imperfect or LLM-judge verifiers are out of scope. The protocol assumes one update step between rounds (\cref{ass:predictable}); mid-batch updates require the multi-epoch variant (\cref{app:alg-details}). Calibration cost is one-time and excluded from per-round complexity.

\section{Conclusion and Discussion}
\label{sec:conclusion}

\paragraph{Summary.}
The framework triple of \cref{sec:framework-triple} admits a single uniquely-shaped cell (anytime-pathwise validity on selective risk under a non-refusing deployment rule) that no prior wrapper occupies. \csa fills it, with rate-optimal certification (\cref{thm:power,thm:lower-bound}) and a horizon-independent release-rate gap (\cref{thm:utility-gap}). Across $480 + 100 + 160$ replicated streams over eight specialist benchmarks, four base models in three architecture families, and sixteen adversarial orderings, \csa is the only method among the ten compared that satisfies both pathwise validity and non-refusing deployment on every cell.

\paragraph{Scope and limitations.}
The guarantee is with respect to the verifier $V$; harms outside $V$'s scope require complementary safeguards. The setting is deterministic-verifier specialist deployments; extension to imperfect or probabilistic verifiers (e.g.\ LLM-judges with calibrated error rates) is open. The sparsely-observed-verifier variant, where $V$ is queried on a $\pi$-fraction of rounds, is treated in \cref{app:sparse-theory}. The deployment protocol requires one update step between rounds; mid-batch updates fall outside \cref{ass:predictable} and require the multi-epoch variant.

\paragraph{What this paper is and is not.}
We do not propose a new LLM, training algorithm, policy class, or stronger reasoning model. The contribution is the deployment-side complement to those design choices: given a trained local specialist and its verifier, the operator can certify at every wall-clock round that the verifier-measured failure rate among released outputs stays below the contractual $\alpha$. The framework of \cref{sec:framework} generalizes beyond RLVR: any wrapper can be classified by its (statistic, validity, deployment-rule) triple, and the cell \csa fills is independently of interest for any anytime-pathwise selective-risk problem.

\paragraph{Reproducibility and broader impact.}
Source code, aggregate result summaries, and end-to-end table- and figure-regeneration scripts are released at \url{https://github.com/HamedKhosravi99/CSA-RLVR}; full reproducibility details (compute, surrogate, data release) are in \cref{app:implementation}; broader scope and impact discussion in \cref{app:scope}.

\small
\bibliographystyle{plainnat}
\bibliography{refs}

@article{deepseekr1,
  title        = {{DeepSeek-R1} incentivizes reasoning in {LLMs} through reinforcement learning},
  author       = {Guo, Daya and Yang, Dejian and Zhang, Haowei and others},
  journal      = {Nature},
  volume       = {645},
  pages        = {633--638},
  year         = {2025},
  doi          = {10.1038/s41586-025-09422-z},
  url          = {https://www.nature.com/articles/s41586-025-09422-z},
}

@article{tulu3,
  title        = {{T\"ulu} 3: Pushing Frontiers in Open Language Model Post-Training},
  author       = {Lambert, Nathan and Morrison, Jacob and Pyatkin, Valentina and Huang, Shengyi and Ivison, Hamish and Brahman, Faeze and Miranda, Lester James V. and Liu, Alisa and Dziri, Nouha and Lyu, Shane and Gu, Yuling and Malik, Saumya and Graf, Victoria and Hwang, Jena D. and Yang, Jiangjiang and Bras, Ronan Le and Tafjord, Oyvind and Wilhelm, Chris and Soldaini, Luca and Smith, Noah A. and Wang, Yizhong and Dasigi, Pradeep and Hajishirzi, Hannaneh},
  journal      = {arXiv preprint arXiv:2411.15124},
  year         = {2025},
  eprint       = {2411.15124},
  archivePrefix= {arXiv},
  primaryClass = {cs.CL},
  url          = {https://arxiv.org/abs/2411.15124},
}

@article{kimik15,
  title        = {{Kimi k1.5}: Scaling Reinforcement Learning with {LLMs}},
  author       = {{Kimi Team}},
  journal      = {arXiv preprint arXiv:2501.12599},
  year         = {2025},
  eprint       = {2501.12599},
  archivePrefix= {arXiv},
  primaryClass = {cs.AI},
  url          = {https://arxiv.org/abs/2501.12599},
}

@inproceedings{hu2022lora,
  title        = {{LoRA}: Low-Rank Adaptation of Large Language Models},
  author       = {Hu, Edward J. and Shen, Yelong and Wallis, Phillip and Allen-Zhu, Zeyuan and Li, Yuanzhi and Wang, Shean and Wang, Lu and Chen, Weizhu},
  booktitle    = {International Conference on Learning Representations},
  year         = {2022},
  url          = {https://arxiv.org/abs/2106.09685},
}

@inproceedings{anthony2017thinking,
  title        = {Thinking Fast and Slow with Deep Learning and Tree Search},
  author       = {Anthony, Thomas and Tian, Zheng and Barber, David},
  booktitle    = {Advances in Neural Information Processing Systems},
  year         = {2017},
  url          = {https://arxiv.org/abs/1705.08439},
}

@inproceedings{activercps,
  title        = {Active, Anytime-Valid Risk Controlling Prediction Sets},
  author       = {Xu, Ziyu and Karampatziakis, Nikos and Mineiro, Paul},
  booktitle    = {Advances in Neural Information Processing Systems},
  year         = {2024},
  url          = {https://arxiv.org/abs/2406.10490},
}

@inproceedings{gibbs2021adaptive,
  title        = {Adaptive Conformal Inference Under Distribution Shift},
  author       = {Gibbs, Isaac and Cand{\`e}s, Emmanuel J.},
  booktitle    = {Advances in Neural Information Processing Systems},
  volume       = {34},
  pages        = {1660--1672},
  year         = {2021},
  url          = {https://proceedings.neurips.cc/paper/2021/hash/0d441de75945e5acbc865406fc9a2559-Abstract.html},
}

@inproceedings{bhatnagar2023improved,
  title        = {Improved Online Conformal Prediction via Strongly Adaptive Online Learning},
  author       = {Bhatnagar, Aadyot and Wang, Huan and Xiong, Caiming and Bai, Yu},
  booktitle    = {Proceedings of the 40th International Conference on Machine Learning (ICML)},
  series       = {Proceedings of Machine Learning Research},
  volume       = {202},
  pages        = {2337--2363},
  year         = {2023},
  publisher    = {PMLR},
  url          = {https://proceedings.mlr.press/v202/bhatnagar23a.html},
}

@inproceedings{ltt,
  title        = {Learn Then Test: Calibrating Predictive Algorithms to Achieve Risk Control},
  author       = {Angelopoulos, Anastasios N. and Bates, Stephen and Cand{\`e}s, Emmanuel J. and Jordan, Michael I. and Lei, Lihua},
  booktitle    = {Proceedings of the 39th International Conference on Machine Learning (ICML)},
  series       = {Proceedings of Machine Learning Research},
  volume       = {162},
  pages        = {986--1011},
  year         = {2022},
  publisher    = {PMLR},
  url          = {https://proceedings.mlr.press/v162/angelopoulos22a.html},
}

@inproceedings{crc,
  title        = {Conformal Risk Control},
  author       = {Angelopoulos, Anastasios N. and Bates, Stephen and Fisch, Adam and Lei, Lihua and Schuster, Tal},
  booktitle    = {Proceedings of the 12th International Conference on Learning Representations (ICLR)},
  year         = {2024},
  url          = {https://arxiv.org/abs/2208.02814},
}

@inproceedings{mohri,
  title        = {Language Models with Conformal Factuality Guarantees},
  author       = {Mohri, Christopher and Hashimoto, Tatsunori},
  booktitle    = {Proceedings of the 41st International Conference on Machine Learning (ICML)},
  series       = {Proceedings of Machine Learning Research},
  volume       = {235},
  pages        = {36029--36047},
  year         = {2024},
  publisher    = {PMLR},
  url          = {https://proceedings.mlr.press/v235/mohri24a.html},
}

@article{nexconf,
  title        = {Conformal Prediction Beyond Exchangeability},
  author       = {Barber, Rina Foygel and Cand{\`e}s, Emmanuel J. and Ramdas, Aaditya and Tibshirani, Ryan J.},
  journal      = {The Annals of Statistics},
  volume       = {51},
  number       = {2},
  pages        = {816--845},
  year         = {2023},
  doi          = {10.1214/23-AOS2276},
  url          = {https://arxiv.org/abs/2202.13415},
}

@article{wald1945,
  title        = {Sequential Tests of Statistical Hypotheses},
  author       = {Wald, Abraham},
  journal      = {The Annals of Mathematical Statistics},
  volume       = {16},
  number       = {2},
  pages        = {117--186},
  year         = {1945},
  publisher    = {Institute of Mathematical Statistics},
  doi          = {10.1214/aoms/1177731118},
}

@article{wsramdas2023,
  title        = {Estimating Means of Bounded Random Variables by Betting},
  author       = {Waudby-Smith, Ian and Ramdas, Aaditya},
  journal      = {Journal of the Royal Statistical Society Series B: Statistical Methodology},
  volume       = {86},
  number       = {1},
  pages        = {1--27},
  year         = {2024},
  doi          = {10.1093/jrsssb/qkad009},
  url          = {https://arxiv.org/abs/2010.09686},
}

@article{ji2023hallucination,
  author  = {Ji, Ziwei and Lee, Nayeon and Frieske, Rita and Yu, Tiezheng and Su, Dan and Xu, Yan and Ishii, Etsuko and Bang, Yejin and Madotto, Andrea and Fung, Pascale},
  title   = {Survey of Hallucination in Natural Language Generation},
  journal = {ACM Computing Surveys},
  volume  = {55},
  number  = {12},
  articleno = {248},
  numpages  = {38},
  year    = {2023},
  doi     = {10.1145/3571730},
  url     = {https://doi.org/10.1145/3571730}
}

@article{busch2025patientcare,
  author  = {Busch, Felix and Hoffmann, Lena and Rueger, Christopher and van Dijk, Elon H. C. and Kader, Rawen and Ortiz-Prado, Esteban and Makowski, Marcus R. and Saba, Luca and Hadamitzky, Martin and Kather, Jakob Nikolas and Truhn, Daniel and Cuocolo, Renato and Adams, Lisa C. and Bressem, Keno K.},
  title   = {Current Applications and Challenges in Large Language Models for Patient Care: A Systematic Review},
  journal = {Communications Medicine},
  volume  = {5},
  number  = {1},
  pages   = {26},
  year    = {2025},
  doi     = {10.1038/s43856-024-00717-2},
  url     = {https://doi.org/10.1038/s43856-024-00717-2}
}

@article{rcps,
  title        = {Distribution-Free, Risk-Controlling Prediction Sets},
  author       = {Bates, Stephen and Angelopoulos, Anastasios N. and Lei, Lihua and Malik, Jitendra and Jordan, Michael I.},
  journal      = {Journal of the ACM},
  volume       = {68},
  number       = {6},
  pages        = {1--34},
  year         = {2021},
  doi          = {10.1145/3478535},
  url          = {https://dl.acm.org/doi/10.1145/3478535},
}

@article{bao2025cap,
  title        = {{CAP}: A General Algorithm for Online Selective Conformal Prediction with {FCR} Control},
  author       = {Bao, Yajie and Huo, Yuyang and Ren, Haojie and Zou, Changliang},
  journal      = {Journal of Machine Learning Research},
  volume       = {26},
  number       = {287},
  pages        = {1--74},
  year         = {2025},
}

@inproceedings{weinstein2020online,
  title        = {Online Control of the False Coverage Rate and False Sign Rate},
  author       = {Weinstein, Asaf and Ramdas, Aaditya},
  booktitle    = {Proceedings of the 37th International Conference on Machine Learning (ICML)},
  series       = {Proceedings of Machine Learning Research},
  volume       = {119},
  pages        = {10193--10202},
  year         = {2020},
  publisher    = {PMLR},
  url          = {https://proceedings.mlr.press/v119/weinstein20a.html},
}

@inproceedings{cofact,
  title     = {{CoFact}: Conformal Factuality Guarantees for Language Models under Covariate Shift},
  author    = {Hu, Zirui and Zhang, Zheng and Wang, Yingjie and Rutkowski, Leszek and Tao, Dacheng},
  booktitle = {Proceedings of the International Conference on Learning Representations (ICLR)},
  year      = {2026},
}

@inproceedings{confarbitrage,
  title     = {Conformal Arbitrage: Risk-Controlled Balancing of Competing Objectives in Language Models},
  author    = {Overman, William and Bayati, Mohsen},
  booktitle = {Advances in Neural Information Processing Systems (NeurIPS)},
  year      = {2025},
  url       = {https://arxiv.org/abs/2506.00911},
}

@article{medqa,
  title   = {What Disease Does This Patient Have? A Large-Scale Open Domain Question Answering Dataset from Medical Exams},
  author  = {Jin, Di and Pan, Eileen and Oufattole, Nassim and Weng, Wei-Hung and Fang, Hanyi and Szolovits, Peter},
  journal = {Applied Sciences},
  volume  = {11},
  number  = {14},
  year    = {2021},
  doi     = {10.3390/app11146421},
}

@inproceedings{pubmedqa,
  title     = {{PubMedQA}: A Dataset for Biomedical Research Question Answering},
  author    = {Jin, Qiao and Dhingra, Bhuwan and Liu, Zhengping and Cohen, William W. and Lu, Xinghua},
  booktitle = {Proceedings of EMNLP-IJCNLP},
  year      = {2019},
  doi       = {10.18653/v1/D19-1259},
}

@inproceedings{tatqa,
  title     = {{TAT-QA}: A Question Answering Benchmark on a Hybrid of Tabular and Textual Content in Finance},
  author    = {Zhu, Fengbin and Lei, Wenqiang and Huang, Youcheng and Wang, Chao and Zhang, Shuo and Lv, Jiancheng and Feng, Fuli and Chua, Tat-Seng},
  booktitle = {Proceedings of ACL},
  year      = {2021},
  doi       = {10.18653/v1/2021.acl-long.254},
}

@inproceedings{mednli,
  title     = {Lessons from Natural Language Inference in the Clinical Domain},
  author    = {Romanov, Alexey and Shivade, Chaitanya},
  booktitle = {Proceedings of EMNLP},
  year      = {2018},
  doi       = {10.18653/v1/D18-1187},
}

@misc{gsm8k,
  title        = {Training Verifiers to Solve Math Word Problems},
  author       = {Cobbe, Karl and Kosaraju, Vineet and Bavarian, Mohammad and Chen, Mark and Jun, Heewoo and Kaiser, Lukasz and Plappert, Matthias and Tworek, Jerry and Hilton, Jacob and Nakano, Reiichiro and Hesse, Christopher and Schulman, John},
  year         = {2021},
  eprint       = {2110.14168},
  archivePrefix= {arXiv},
  primaryClass = {cs.LG},
}

@inproceedings{headqa,
  title     = {{HEAD-QA}: A Healthcare Dataset for Complex Reasoning},
  author    = {Vilares, David and G{\'o}mez-Rodr{\'\i}guez, Carlos},
  booktitle = {Proceedings of the 57th Annual Meeting of the Association for Computational Linguistics},
  pages     = {960--966},
  year      = {2019},
  publisher = {Association for Computational Linguistics},
  doi       = {10.18653/v1/P19-1092},
}

@misc{arc,
  title        = {Think You Have Solved Question Answering? Try {ARC}, the {AI2} Reasoning Challenge},
  author       = {Clark, Peter and Cowhey, Isaac and Etzioni, Oren and Khot, Tushar and Sabharwal, Ashish and Schoenick, Carissa and Tafjord, Oyvind},
  year         = {2018},
  eprint       = {1803.05457},
  archivePrefix= {arXiv},
  primaryClass = {cs.AI},
}

@inproceedings{casehold,
  title     = {When Does Pretraining Help? {A}ssessing Self-Supervised Learning for Law and the {C}ase{HOLD} Dataset of 53,000+ Legal Holdings},
  author    = {Zheng, Lucia and Guha, Neel and Anderson, Brandon R. and Henderson, Peter and Ho, Daniel E.},
  booktitle = {Proceedings of the Eighteenth International Conference on Artificial Intelligence and Law (ICAIL)},
  year      = {2021},
  doi       = {10.1145/3462757.3466088},
}

@inproceedings{lexglue,
  title     = {{L}ex{GLUE}: A Benchmark Dataset for Legal Language Understanding in {E}nglish},
  author    = {Chalkidis, Ilias and Jana, Abhik and Hartung, Dirk and Bommarito, Michael J. and Androutsopoulos, Ion and Katz, Daniel Martin and Aletras, Nikolaos},
  booktitle = {Proceedings of ACL},
  year      = {2022},
  doi       = {10.18653/v1/2022.acl-long.297},
}

@misc{fleming,
  title        = {{Fleming-R1}: Toward Expert-Level Medical Reasoning via Reinforcement Learning},
  author       = {Liu, Chi and Li, Derek and Shu, Yan and Chen, Robin and Duan, Derek and Fang, Teng and Dai, Bryan},
  year         = {2025},
  eprint       = {2509.15279},
  archivePrefix= {arXiv},
  primaryClass = {cs.CL},
  url          = {https://arxiv.org/abs/2509.15279},
}

@misc{finr1,
  title        = {{Fin-R1}: A Large Language Model for Financial Reasoning through Reinforcement Learning},
  author       = {Liu, Zhaowei and Guo, Xin and Yang, Zhi and Lou, Fangqi and Zeng, Lingfeng and Niu, Jinyi and Li, Mengping and Qi, Qi and Liu, Zhiqiang and Han, Yiyang and others},
  year         = {2025},
  eprint       = {2503.16252},
  archivePrefix= {arXiv},
  primaryClass = {cs.CL},
  url          = {https://arxiv.org/abs/2503.16252},
}

@misc{saul,
  title        = {{SaulLM-7B}: A Pioneering Large Language Model for Law},
  author       = {Colombo, Pierre and Pires, Telmo Pessoa and Boudiaf, Malik and Culver, Dominic and Melo, Rui and Corro, Caio and Martins, Andre F. T. and Esposito, Fabrizio and Raposo, Vera L{\'u}cia and Morgado, Sofia and Desa, Michael},
  year         = {2024},
  eprint       = {2403.03883},
  archivePrefix= {arXiv},
  primaryClass = {cs.CL},
  url          = {https://arxiv.org/abs/2403.03883},
}

@misc{qwen25,
  title        = {{Qwen2.5} Technical Report},
  author       = {{Qwen Team}},
  year         = {2024},
  eprint       = {2412.15115},
  archivePrefix= {arXiv},
  primaryClass = {cs.CL},
}

@misc{qwenmath,
  title        = {{Qwen2.5-Math} Technical Report: Toward Mathematical Expert Model via Self-Improvement},
  author       = {Yang, An and Zhang, Beichen and Hui, Binyuan and Gao, Bofei and Yu, Bowen and Li, Chengpeng and Liu, Dayiheng and Tu, Jianhong and Zhou, Jingren and Lin, Junyang and others},
  year         = {2024},
  eprint       = {2409.12122},
  archivePrefix= {arXiv},
  primaryClass = {cs.CL},
  url          = {https://arxiv.org/abs/2409.12122},
}

@inproceedings{vllm,
  title     = {Efficient Memory Management for Large Language Model Serving with {PagedAttention}},
  author    = {Kwon, Woosuk and Li, Zhuohan and Zhuang, Siyuan and Sheng, Ying and Zheng, Lianmin and Yu, Cody Hao and Gonzalez, Joseph E. and Zhang, Hao and Stoica, Ion},
  booktitle = {Proceedings of SOSP},
  year      = {2023},
  doi       = {10.1145/3600006.3613165},
}

@article{christophe2024med42,
  title        = {{Med42} -- Evaluating Fine-Tuning Strategies for Medical {LLMs}: Full-Parameter vs. Parameter-Efficient Approaches},
  author       = {Christophe, Cl{\'e}ment and Kanithi, Praveen K. and Munjal, Prateek and Raha, Tathagata and Hayat, Nasir and Rajan, Ronnie and Al-Mahrooqi, Ahmed and Gupta, Avani and Salman, Muhammad Umar and Gosal, Gurpreet and Kanakarajan, Bhargav and Vassilieva, Natalia and Pimentel, Marco AF and Khan, Shadab},
  journal      = {arXiv preprint arXiv:2404.14779},
  year         = {2024},
  eprint       = {2404.14779},
  archivePrefix= {arXiv},
  primaryClass = {cs.CL},
}

\newpage
\appendix

\addtocontents{toc}{\protect\setcounter{tocdepth}{2}}
\setcounter{tocdepth}{2}
\renewcommand{\contentsname}{Appendix Contents}
\tableofcontents
\vspace{1em}\hrule\vspace{1em}

\section{\csa architecture overview}
\label{app:architecture}

Figure~\ref{fig:architecture} gives a high-level view of \csa as a plug-and-play wrapper on an existing RLVR system: the top (blue) shows the unchanged RLVR pipeline (prompt $\to$ policy $\to$ output $\to$ verifier), and the bottom (green) shows the \csa wrapper (surrogate score $\to$ e-process update $\to$ certification check $\to$ deployed threshold $\to$ act/abstain). The anytime-valid guarantee (red box) holds simultaneously for all $T$ with no feedback from \csa to the RLVR training loop.

\begin{figure}[!htpb]
\centering
\begin{tikzpicture}[
  >=Stealth,
  node distance=0.6cm and 0.8cm,
  box/.style={draw, rounded corners=3pt, minimum height=0.65cm, minimum width=1.4cm,
              font=\small, align=center, thick},
  rlvr/.style={box, fill=blue!8},
  csa/.style={box, fill=green!10},
  mathbox/.style={draw, rounded corners=2pt, fill=yellow!8, font=\footnotesize,
                  align=center, inner sep=4pt, thick},
  decision/.style={draw, rounded corners=2pt, fill=orange!12, font=\small,
                   align=center, inner sep=3pt, thick, minimum height=0.6cm},
  ann/.style={font=\scriptsize\itshape, text=gray!70!black},
  guarantee/.style={draw=red!60!black, rounded corners=3pt, fill=red!5,
                    font=\footnotesize, align=center, inner sep=4pt, thick},
]
\node[rlvr] (prompt) {$X_t$};
\node[rlvr, right=0.6cm of prompt] (policy) {$\pi_t(\cdot\mid X_t)$};
\node[rlvr, right=0.6cm of policy] (output) {$\widetilde{Y}_t$};
\node[rlvr, right=1.0cm of output] (verifier) {$V_t = V(X_t,\widetilde{Y}_t)$};
\draw[->, thick] (prompt) -- (policy);
\draw[->, thick] (policy) -- (output);
\draw[->, thick] (output) -- (verifier);
\node[above=0.01cm of prompt, font=\scriptsize\bfseries, text=blue!60!black] {prompt};
\node[above=0.01cm of policy, font=\scriptsize\bfseries, text=blue!60!black] {policy};
\node[above=0.01cm of output, font=\scriptsize\bfseries, text=blue!60!black] {output};
\node[above=0.01cm of verifier, font=\scriptsize\bfseries, text=blue!60!black] {verifier};
\node[csa, below=1.4cm of prompt, minimum width=2.2cm] (surrogate)
  {$s_t = 1 - \hat{p}_t(X_t, \widetilde{Y}_t)$};
\node[mathbox, right=0.5cm of surrogate, minimum width=3.8cm] (eproc) {%
  $X_t(q) = A_t(q)\bigl((1{-}V_t) - \alpha\bigr)$\\[2pt]
  $E_t(q) = E_{t-1}(q)\cdot\bigl(1 - \lambda_t(q)\,X_t(q)\bigr)$
};
\node[decision, right=0.5cm of eproc, minimum width=1.8cm] (cert) {%
  $E_t(q) \ge \delta_q^{-1}$\,?\\[1pt]
  {\scriptsize certify $q$ safe}
};
\node[csa, right=0.5cm of cert, minimum width=1.6cm] (deploy) {%
  $q_t = \max \cC_t$\\[1pt]
  {\scriptsize deployed threshold}
};
\node[box, fill=green!25, below=0.7cm of deploy, xshift=-1.0cm, minimum width=1.1cm,
      font=\small] (release) {\textbf{Act}};
\node[box, fill=red!10, below=0.7cm of deploy, xshift=1.0cm, minimum width=1.1cm,
      font=\small] (abstain) {Abstain};
\node[font=\footnotesize, above=0.1cm of release, xshift=0.5cm, text=green!50!black]
  {$s_t \le q_t$};
\node[font=\footnotesize, above=0.1cm of abstain, xshift=-0.3cm, text=red!60!black]
  {$s_t > q_t$};
\draw[->, thick] (output.south) -- ++(0,-0.5) -| (surrogate.north);
\draw[->, thick] (surrogate) -- (eproc);
\draw[->, thick] (eproc) -- (cert);
\draw[->, thick] (cert) -- (deploy);
\draw[->, thick] (deploy.south) -- ++(0,-0.25) -| (release.north);
\draw[->, thick] (deploy.south) -- ++(0,-0.25) -| (abstain.north);
\draw[->, thick, dashed, gray] (verifier.south) -- ++(0,-0.65) -| ([xshift=0.5cm]eproc.north);
\node[font=\tiny\itshape, text=black, below=0.05cm of verifier, xshift=-0.3cm] {$V_t$ feedback};
\node[guarantee, below=0.75cm of release, xshift=-4cm, minimum width=5.5cm] (guarantee) {%
  \textbf{Anytime-valid guarantee} (Thm.~\ref{thm:main-anytime}):\\[2pt]
  $\displaystyle R_T^{\mathrm{act}} \le \alpha + O\!\left(\sqrt{\frac{\log(1/\delta)}{N_T}}\right)
  \quad \forall\, T \ge 1$
  \quad w.p.\ $\ge 1{-}2\delta$
};
\begin{scope}[on background layer]
\node[draw=blue!50, dashed, rounded corners=6pt, thick,
      fit=(prompt)(policy)(output)(verifier),
      inner xsep=6pt, inner ysep=10pt] {};
\node[draw=green!50!black, dashed, rounded corners=6pt, thick,
      fit=(surrogate)(eproc)(cert)(deploy)(release)(abstain),
      inner xsep=8pt, inner ysep=8pt] {};
\end{scope}
\end{tikzpicture}
\caption{\csa as a plug-and-play deployment wrapper. Top (blue): the RLVR system is unchanged. Middle (green): \csa computes a surrogate score, updates one e-process per candidate threshold, certifies when the e-process crosses $\delta_q^{-1}$, and gates release via the largest certified threshold. Bottom (red): the anytime-valid guarantee holds simultaneously for all $T$ under predictable updates and monotone risk, without exchangeability and with no feedback to training.}
\label{fig:architecture}
\end{figure}
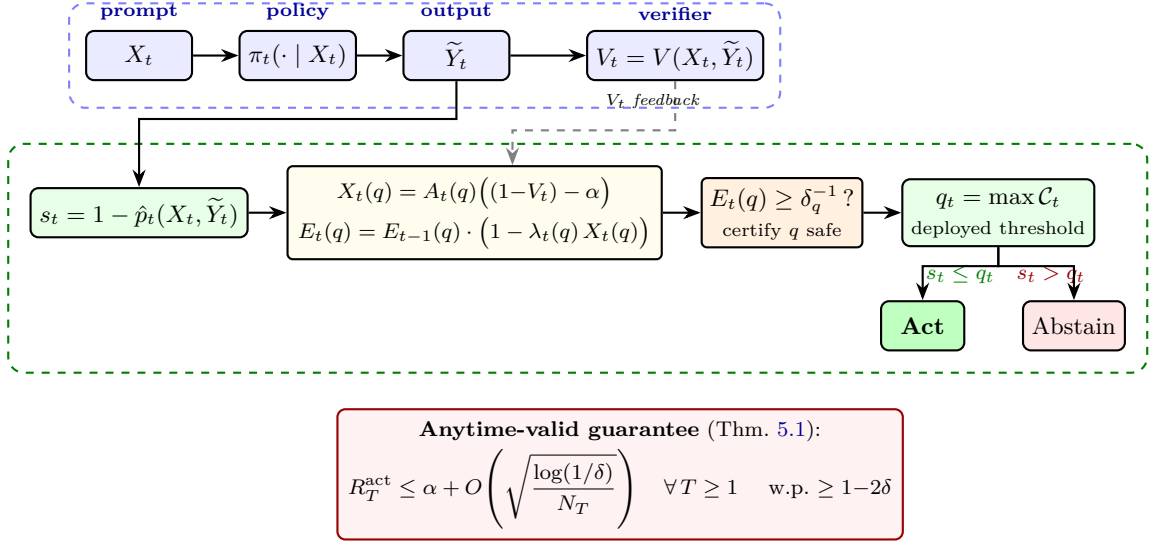

\section{Algorithm details and implementation}
\label{app:alg-details}

\subsection{\csa: single-epoch version}
\label{app:alg-single}

Algorithm~\ref{alg:csa-rlvr} is the deployed variant. It runs a single e-process per threshold, never resets, and uses the equal budget allocation $\delta_q = \delta/(2m)$. Certifications are permanent under the exact-stability condition (Corollary~\ref{cor:single-epoch-sufficiency}).

\begin{algorithm}[!htpb]
\caption{\textsc{CSA-RLVR}: Conformal Selective Acting (single-epoch).}
\label{alg:csa-rlvr}
\begin{algorithmic}[1]
\Require Grid $\cQ = \{q^{(1)} < \cdots < q^{(m)}\}$, target $\alpha\in(0,1)$,
         confidence $\delta\in(0,1)$, initial surrogate $\hat{p}_1$, initial policy $\pi_1$
\State $\delta_q \gets \delta/(2m)$ for all $q\in\cQ$
  \Comment{equal budget split}
\State $\ell_q \gets 0$ for all $q\in\cQ$
  \Comment{log e-process $\ell_q = \log E_t(q)$}
\State $\Sigma_q \gets 0$, $n_q\gets 0$ for all $q\in\cQ$
\State $\textsc{Cert}(q) \gets \textsc{False}$ for all $q\in\cQ$;
       $q_{\mathrm{deploy}}\gets\bot$
\For{$t = 1, 2, \ldots$}
  \State Observe $X_t$; sample $\widetilde{Y}_t\sim\pi_t(\cdot\mid X_t)$;
         $s_t\gets 1-\hat{p}_t(X_t,\widetilde{Y}_t)$
  \If{$q_{\mathrm{deploy}}\ne\bot$ \textbf{and} $s_t \le q_{\mathrm{deploy}}$}
    \State Release $\widetilde{Y}_t$ ($A_t=1$)
  \Else
    \State Abstain ($A_t=0$)
  \EndIf
  \State $V_t \gets V(X_t,\widetilde{Y}_t)$
    \Comment{verifier on every round}
  \For{$q\in\cQ$ with $q \ge s_t$}
    \Comment{thresholds with $A_t(q)=1$}
    \State $X_t^q\gets (1-V_t) - \alpha$
    \If{$n_q > 0$}
      \State $\hat\mu_q \gets \Sigma_q/n_q$;
             $\lambda_t^q\gets\clip\!\bigl(-\hat\mu_q/(1-\alpha)^2,\,0,\,1/(2(1-\alpha))\bigr)$
    \Else
      \State $\lambda_t^q\gets 0$
    \EndIf
    \State $\ell_q \gets \ell_q + \log(1 - \lambda_t^q X_t^q)$;
           $\Sigma_q\gets\Sigma_q + X_t^q$; $n_q\gets n_q+1$
    \If{$\neg\,\textsc{Cert}(q)$ \textbf{and} $\ell_q \ge \log(1/\delta_q)$}
      \State $\textsc{Cert}(q)\gets\textsc{True}$
    \EndIf
  \EndFor
  \State $q_{\mathrm{deploy}}\gets \max\{q\in\cQ:\textsc{Cert}(q)\}$
         (or $\bot$ if none certified)
  \State Add $(X_t,\widetilde{Y}_t,V_t)$ to buffer $\cD_t$; update $\hat{p}_{t+1}$, $\pi_{t+1}$
\EndFor
\end{algorithmic}
\end{algorithm}

\paragraph{Correctness.}
Lines~1--4 set up $E_0(q){=}e^{\ell_q}{=}1$, matching~\eqref{eq:eproc}. The surrogate $\hat p_t$ and policy $\pi_t$ are frozen before $V_t$ is observed, satisfying Assumption~\ref{ass:predictable}. Lines~12--22 update only thresholds $q\ge s_t$ for which $A_t(q){=}1$; when $A_t(q){=}0$ the increment $X_t(q){=}0$ and $E_t(q){=}E_{t-1}(q)$, so skipping these thresholds is exact. The adaptive bet on line~16 is $\cF_{t-1}$-measurable and satisfies $\lambda_t^q\in[0,(1-\alpha)^{-1}]$ by clipping, so Proposition~\ref{prop:eproc} applies. With budget $\delta_q=\delta/(2m)$, Theorem~\ref{thm:no-false-cert} gives $\PP(\text{any false certification})\le\delta/2$.

\subsection{CSA-Epoch: multi-epoch version}
\label{app:alg-epoch}

When the monotone-frontier assumption may be violated (e.g., under distribution shift or policy degradation), Algorithm~\ref{alg:csa-epoch} resets at deterministic epoch boundaries $\tau_1<\tau_2<\cdots$ with budget $\delta_{j,q}=6\delta/(\pi^2 m j^2)$. A restart revokes all certifications and costs $\sim$$100$--$400$ rounds per epoch to re-certify. The no-false-certification guarantee (Theorem~\ref{thm:no-false-cert}) holds within each epoch as long as the frontier does not drop by more than one grid step $\Delta q$ during the epoch; the epoch length should therefore be chosen short enough that the bounded-drift condition (Theorem~\ref{thm:monotone-frontier}) limits within-epoch frontier movement to at most $\Delta q$.

\begin{algorithm}[!htpb]
\caption{\textsc{CSA-Epoch}: Conformal Selective Acting (multi-epoch).}
\label{alg:csa-epoch}
\begin{algorithmic}[1]
\Require Grid $\cQ$, target $\alpha$, confidence $\delta$,
         restart schedule $\tau_1{=}1<\tau_2<\cdots$
\State $j\gets 1$; \Call{ResetEpoch}{$j$}
\For{$t=1,2,\ldots$}
  \If{$t = \tau_{j+1}$}
    \State $j\gets j+1$; \Call{ResetEpoch}{$j$}
  \EndIf
  \State Execute inner loop of Algorithm~\ref{alg:csa-rlvr} with $\delta_q\gets\delta_{j,q}$
\EndFor
\Procedure{ResetEpoch}{$j$}
  \State $\delta_{j,q}\gets 6\delta/(\pi^2\cdot m\cdot j^2)$;
         $\ell_q\gets 0$; $\Sigma_q\gets 0$; $n_q\gets 0$;
         $\textsc{Cert}(q)\gets\textsc{False}$
\EndProcedure
\end{algorithmic}
\end{algorithm}

\subsection{Computational and memory analysis}
\label{app:complexity}

\paragraph{Per-round cost.}
At round $t$: one surrogate forward pass ($C_{\mathrm{surr}}$); a binary search over $\cQ$ to find the largest $k^\star$ with $q^{(k^\star)}\ge s_t$ ($O(\log m)$); $k^\star$ e-process updates each with one log and one addition ($O(k^\star)$); one scan for the new max certified threshold ($O(m)$, or $O(1)$ amortized since $q_{\mathrm{deploy}}$ only increases in the single-epoch version); amortized surrogate retraining ($C_{\mathrm{retrain}}/B$). Total per-round cost $O(m) + C_{\mathrm{surr}} + C_{\mathrm{retrain}}/B$. For $m=50$--$200$ the $O(m)$ overhead is negligible relative to policy inference.

\paragraph{Measured overhead.}
Table~\ref{tab:overhead} gives per-component times on MATH-200 replay with Qwen2.5-Math-7B (4-bit MLX): the entire \csa pipeline adds $0.16$\,ms per round, $<\!0.001\%$ of the $\sim$17\,s generation time.

\begin{table}[h]
\centering\small
\caption{\textbf{Per-round cost breakdown} for \csa on MATH-200 replay with Qwen2.5-Math-7B. Verifier and all \csa machinery together add $<\!0.001\%$ of generation time.}
\label{tab:overhead}
\begin{tabular}{lrl}
\toprule
\textbf{Component} & \textbf{Time} & \textbf{\% of generation} \\
\midrule
LLM generation (Qwen-7B, 249 tokens avg.) & $\sim$17\,s  & $100\%$ \\
Exact-match verification                   & $59\,\mu$s  & $0.0004\%$ \\
Surrogate scoring (logistic regression)    & $96\,\mu$s  & $0.0006\%$ \\
E-process update ($m{=}15$ thresholds)     & $7\,\mu$s   & $<\!0.0001\%$ \\
\midrule
\textbf{Total \csa overhead}               & $\mathbf{0.16}$\,\textbf{ms} & $<\!0.001\%$ \\
\bottomrule
\end{tabular}
\end{table}

\paragraph{Memory.}
Working memory (excluding surrogate and replay buffer): three floats per threshold ($\ell_q,\Sigma_q,n_q$) and one boolean ($\textsc{Cert}(q)$), totaling $O(m)$ floats.

\paragraph{Numerical stability.}
All e-processes are maintained in log-space. The factor $1-\lambda_t(q)X_t(q)$ is guaranteed positive by $\lambda_t(q)\le 1/(2(1-\alpha))$ and $|X_t(q)|\le 1-\alpha$. No numerical issues were encountered across $>3$~million total e-process updates.

\subsection{Verifier evaluation on abstained rounds}
\label{app:verifier-all-rounds}

Algorithm~\ref{alg:csa-rlvr} evaluates $V_t$ on every round, including abstained ones, because the e-process for any $q\ge s_t$ uses $X_t(q)=(1-V_t)-\alpha$ regardless of whether the deployed gate acted. Skipping the verifier on abstained rounds would freeze e-processes at uncertified thresholds and delay or prevent their certification. Two practical modifications preserve validity:
\begin{itemize}[leftmargin=1.3em,nosep]
\item \textbf{Delayed evaluation.} Queue $(X_t,\widetilde Y_t)$ and run the verifier in batch; apply the e-process updates upon completion. Assumption~\ref{ass:predictable} holds as long as $\hat p_{t+1}$ uses only verified examples.
\item \textbf{Partial evaluation.} Run the verifier only on acted rounds. Updates are restricted to $q\le q_{\mathrm{deploy}}$; higher thresholds receive no update and cannot be certified, so the controller stalls at the current threshold but stays valid.
\end{itemize}
For expensive verifiers (LLM judges with $C_{\mathrm{ver}}\!\approx\!C_{\mathrm{gen}}$) the sparse-verifier variant (\S\ref{app:sparse}) is preferable: Bernoulli subsampling at rate $\pi$ preserves full anytime-valid validity and inflates the expected certification-delay bound by at most $1/\pi_{\min}$ (exactly $1/\pi$ when $\pi_t\equiv\pi$).

\subsection{Implementation choices not fixed by theory}
\label{app:impl-choices}

\begin{itemize}[leftmargin=1.3em, nosep]
\item \textbf{Surrogate model class.} Theory requires $\hat p_t$ $\cF_{t-1}$-measurable (Assumption~\ref{ass:predictable}) and, for the RLVR structural results, uniformly convergent (Assumption~\ref{ass:surrogate-calibration}). A natural default is logistic regression on features $\phi(x,y)=[\texttt{difficulty}(x),\log p_\pi(y\mid x),\texttt{len}(y),1]$, which has $\varepsilon_t=O(\sqrt{d\log t/t})$. Richer surrogates (calibrated neural networks, Platt-scaled policy logits) are viable.
\item \textbf{Threshold grid.} We use $m{=}15$ geometrically-spaced points between the $2\%/98\%$ quantiles of the calibrated score. Finer grids approximate $q_t^\star$ more closely (discretization gap $\le\Delta q$) but increase the utility gap because each threshold must be individually certified.
\item \textbf{Surrogate retraining frequency.} Batched every $B{=}500$ rounds.
\item \textbf{Warm start.} The paper uses a calibration phase on the 80\% calibration split before deployment, seeding the e-processes with $n_0$ pre-deployment updates.
\item \textbf{Fallback action.} The theory is agnostic: $R_T^{\mathrm{act}}$ counts only released outputs.
\end{itemize}

\paragraph{Deployment recipe for the multi-epoch variant.}
Practitioners using \textsc{CSA-Epoch} (Algorithm~\ref{alg:csa-epoch}) face three operational choices. We recommend the following defaults and diagnostics: (i)~\emph{Epoch length} $\tau_{j+1}-\tau_j$: set to $1{,}000$--$2{,}000$ rounds for moderate-shift settings (matches the bounded-drift condition $\xi_t{\le}\Delta q/$ epoch length used in our LoRA experiments), and shorten to $200$--$500$ rounds when the policy is updated more aggressively (e.g., per-batch RLHF). (ii)~\emph{Drift budget} $\nu_j$: start at $0$ (tightest bound) and inflate only when the structural-layer estimate $\varepsilon_j$ from Appendix~\ref{app:rlvr-layer} exceeds the running surrogate calibration gap; in our experiments $\nu_j{=}0$ sufficed across all live cells. (iii)~\emph{Stale-certificate diagnostics}: monitor (a) the per-round running risk $R_t^{\mathrm{act}}$ at the deployed threshold against the certified bound $\alpha+\text{slack}(N_T)$, and (b) the surrogate Brier/ECE drift over a sliding window; a sustained Brier increase of $>20\%$ over the warmup baseline is a leading indicator that the frontier may have dropped and an early epoch reset is warranted.

\section{Deferred proofs}
\label{app:proofs}

\subsection{E-process validity (Proposition~\ref{prop:eproc})}
\label{app:proof-eproc}

\begin{proof}
Under $H_{j,q}^{\mathrm{unsafe}}:\EE[X_t(q)\mid\cF_{t-1}]\ge 0$ and predictable $\lambda_{j,t}(q)\in[0,(1-\alpha)^{-1}]$,
\begin{align*}
\EE\!\bigl[E_{j,t}(q)\mid\cF_{t-1}\bigr]
&= E_{j,t-1}(q)\,\EE\!\bigl[1-\lambda_{j,t}(q)X_t(q)\mid\cF_{t-1}\bigr] \\
&= E_{j,t-1}(q)\bigl(1-\lambda_{j,t}(q)\EE[X_t(q)\mid\cF_{t-1}]\bigr)
\;\le\; E_{j,t-1}(q).
\end{align*}
Nonnegativity: since $X_t(q)\in\{-\alpha,0,1-\alpha\}$ and $\lambda_{j,t}(q)\le(1-\alpha)^{-1}$, the factor $1-\lambda_{j,t}(q)X_t(q)\ge 0$.
\end{proof}

\subsection{No false certification (Theorem~\ref{thm:no-false-cert})}
\label{app:proof-no-false}

\begin{theorem}[No false certification]
\label{thm:no-false-cert}
Suppose the deployed epoch schedule is such that the oracle frontier $q_t^\star$ is nondecreasing within each epoch $\cI_j$. This frontier-stability condition is supplied by either the single-epoch algorithm under exact stability ($\xi_t\equiv 0$, via Theorem~\ref{thm:monotone-frontier} and Corollary~\ref{cor:single-epoch}) or any analysis partition at frontier-drop events. (Assumption~\ref{ass:nested} alone is a within-time condition and does \emph{not} imply across-time frontier nondecreasingness.) Let $T_{j,q}:=\inf\{t\in\cI_j:\,q\le q_t^\star\}$ be the first time within epoch $j$ at which $q$ becomes safe. Then
\[
\PP\!\bigl(\exists\,j,\,q\in\cQ,\,t\in\cI_j:\,t<T_{j,q}\ \text{and}\ q\in\cC_{j,t}\bigr)\le\delta,
\]
and consequently $\PP(q_t\le q_t^\star\;\forall t)\ge 1-\delta$.
\end{theorem}

\begin{proof}
Fix epoch $j$ and threshold $q\in\cQ$. Let
\[
T_{j,q}:=\inf\{t\in\cI_j:\,q\le q_t^\star\}
\]
be the first time within epoch $j$ at which $q$ becomes safe (with $T_{j,q}:=\infty$ if this never occurs), and define the stopped process
\[
\widetilde E_{j,t}(q)\;:=\;E_{j,\,t\wedge(T_{j,q}-1)}(q),\qquad t\in\cI_j.
\]
Before $T_{j,q}$ the threshold $q$ is unsafe, so $\EE[X_t(q)\mid\cF_{t-1}]\ge 0$ and Proposition~\ref{prop:eproc} gives $\widetilde E_{j,t}(q)$ as a nonnegative supermartingale. Ville's inequality then yields
\[
\PP\!\Bigl(\sup_{t<T_{j,q}}E_{j,t}(q)\ge\delta_{j,q}^{-1}\Bigr)
=\PP\!\Bigl(\sup_{t\in\cI_j}\widetilde E_{j,t}(q)\ge\delta_{j,q}^{-1}\Bigr)
\le\delta_{j,q},
\]
so with probability at least $1-\delta_{j,q}$, threshold $q$ is not certified before it becomes safe within epoch $j$. Union-bounding over $q\in\cQ$ and $j\ge 1$ with $\delta_{j,q}=6\delta/(\pi^2 m j^2)$,
\[
\sum_j\sum_q\delta_{j,q} = \frac{6\delta}{\pi^2 m}\sum_j\frac{1}{j^2}\cdot m = \frac{6\delta}{\pi^2}\cdot\frac{\pi^2}{6} = \delta.
\]
On the resulting event, every deployed threshold is certified only after it is safe in its epoch, hence $q_t\le q_t^\star$ for all $t$. \qedhere
\end{proof}

\subsection{Exponential supermartingale lemma}
\label{app:lem-expsup}

\begin{lemma}[Exponential supermartingale]
\label{lem:exp-supermg}
Let $\{q_t\}_{t\ge 1}$ be any predictable controller trajectory and let $M_t := A_t(q_t)((1-V_t)-\alpha)$. If $\EE[M_t\mid\cF_{t-1}]\le 0$, then for every $\lambda>0$,
\[
L_T(\lambda) \;:=\; \exp\!\Bigl(\lambda\sum_{t=1}^T M_t - \tfrac{\lambda^2}{8}N_T\Bigr)
\]
is a nonnegative supermartingale. Under declared drift budgets $\{\nu_j\}$ and with $J_t:=\1\{q_t\le q_t^\star\}$,
\[
L_T^{\nu}(\lambda) := \exp\!\Bigl(\lambda\sum_{t=1}^{T}J_t(M_t-\nu_{j(t)}A_t) - \tfrac{\lambda^2}{8}\sum_{t=1}^T J_t A_t\Bigr)
\]
is also a nonnegative supermartingale.
\end{lemma}

\begin{proof}
Conditional on $\cF_{t-1}$, $M_t\in[-\alpha,1-\alpha]$ (a bounded range of length~$1$). By Hoeffding's lemma, $\EE[e^{\lambda M_t}\mid\cF_{t-1}]\le\exp(\lambda\EE[M_t\mid\cF_{t-1}]+\lambda^2/8)\le e^{\lambda^2/8}$ on $\{A_t=1\}$, and $=1$ on $\{A_t=0\}$. Combining, $\EE[e^{\lambda M_t}\mid\cF_{t-1}]\le e^{\lambda^2 A_t/8}$, whence $\EE[L_T(\lambda)/L_{T-1}(\lambda)\mid\cF_{T-1}]\le 1$. The drift-aware variant follows analogously with a safe-step indicator $J_t$ freezing the process on unsafe steps.
\end{proof}

\subsection{Anytime-valid selective risk (Theorem~\ref{thm:main-anytime})}
\label{app:proof-main}

\begin{proof}
\textbf{($\nu_j=0$ case.)}
Let $\cG:=\{q_t\le q_t^\star\,\forall t\}$; $\PP(\cG)\ge 1-\delta$ by Theorem~\ref{thm:no-false-cert}. On $\cG$, $\EE[M_t\mid\cF_{t-1}]\le 0$; by Lemma~\ref{lem:exp-supermg}, $L_T(\lambda)$ is a supermartingale.

Fix $\eta\in(0,1)$. For integer $k\ge 0$ set $\eta_k:=\eta/((k+1)(k+2))$ (so $\sum_k\eta_k\le\eta$) and $\lambda_k:=\sqrt{8\log(1/\eta_k)/2^{k+1}}$. Define dyadic blocks $\cB_k:=\{T\ge 1:2^k<N_T\vee 1\le 2^{k+1}\}$. Ville's inequality applied to $L_T(\lambda_k)$ gives, with probability $\ge 1-\eta_k$,
\[
\sup_{T\ge 0}L_T(\lambda_k)<1/\eta_k
\quad\Longleftrightarrow\quad
\lambda_k\sum_t M_t - \tfrac{\lambda_k^2}{8}N_T<\log(1/\eta_k)\quad\forall T.
\]
Rearranging and optimizing the choice of $k$ for each $T\in\cB_k$ gives
\[
\sum_t M_t \le \tfrac{\lambda_k}{8}N_T+\tfrac{\log(1/\eta_k)}{\lambda_k}\le c\sqrt{2^k\log(1/\eta_k)} \quad\text{for }T\in\cB_k,
\]
and the union bound across $k$ gives probability $\ge 1-\eta$ inside $\cG$.

When $N_T>0$, $\sum_t M_t = \sum_t A_t(1-V_t)-\alpha N_T = N_T(R_T^{\mathrm{act}}-\alpha)$. Dividing by $N_T$ and using $\log(1/\eta_k)\le C(\log(1/\eta)+\log\log(eN_T+e))$ yields~\eqref{eq:main-anytime}. When $N_T=0$, $R_T^{\mathrm{act}}=0\le\alpha$, so the bound is trivial.

Setting $\eta=\delta$ and intersecting with $\cG$ via union bound gives the claimed $\ge 1-2\delta$.

\textbf{($\nu_j>0$ case.)}
Let $J_t:=\1\{q_t\le q_t^\star\}$ be the safe-step indicator. By Lemma~\ref{lem:exp-supermg}, the process
\[
L_T^\nu(\lambda):=\exp\!\Bigl(\lambda\sum_{t=1}^T J_t(M_t-\nu_{j(t)}A_t)-\frac{\lambda^2}{8}\sum_{t=1}^T J_t A_t\Bigr)
\]
is a nonnegative supermartingale with $L_0^\nu(\lambda)=1$. The key step is that on each safe round ($J_t=1$), $\EE[M_t-\nu_{j(t)}A_t\mid\cF_{t-1}]\le\nu_{j(t)}-\nu_{j(t)}=0$ by the drift condition, so Hoeffding's lemma applies to the shifted increment $M_t-\nu_{j(t)}A_t\in[-\alpha-\nu_{j(t)},1-\alpha-\nu_{j(t)}]$ (still an interval of length at most $1$), giving $\EE[L_T^\nu(\lambda)/L_{T-1}^\nu(\lambda)\mid\cF_{T-1}]\le 1$. On unsafe rounds ($J_t=0$), both sums are frozen and the ratio is $1$.

On $\cG=\{q_t\le q_t^\star\,\forall t\}$ (which has $\PP(\cG)\ge 1-\delta$), $J_t\equiv 1$ for all $t$, so $L_T^\nu(\lambda)$ reduces to $\exp(\lambda\sum_t(M_t-\nu_{j(t)}A_t)-\lambda^2 N_T/8)$. Apply the same dyadic-block stitching: for $k\ge 0$, $\cB_k:=\{T:2^k<N_T\vee 1\le 2^{k+1}\}$, $\eta_k:=\eta/((k+1)(k+2))$, and $\lambda_k:=\sqrt{8\log(1/\eta_k)/2^{k+1}}$. Ville's inequality on $L_T^\nu(\lambda_k)$ gives, with probability $\ge 1-\eta_k$ on $\cG$,
\[
\sum_{t=1}^T(M_t-\nu_{j(t)}A_t) \le \sqrt{2^k\log(1/\eta_k)} \quad\text{for }T\in\cB_k.
\]
Union-bounding over $k$ and intersecting with $\cG$ yields: with probability $\ge 1-\eta-\delta$,
\[
\sum_{t=1}^T(M_t-\nu_{j(t)}A_t)\le C_1\sqrt{(N_T\vee 1)(\log(1/\eta)+\log\log(eN_T+e))}+C_2(\log(1/\eta)+\log\log(eN_T+e))
\]
for universal constants $C_1,C_2>0$ and all $T\ge 1$ simultaneously.

When $N_T>0$, $\sum_t M_t=N_T(R_T^{\mathrm{act}}-\alpha)$ and $\sum_t\nu_{j(t)}A_t=N_T\bar\nu_T$, so dividing by $N_T$ gives $R_T^{\mathrm{act}}-\alpha-\bar\nu_T\le O(\sqrt{\log(1/\delta)/N_T})$, i.e.\ $R_T^{\mathrm{act}}\le\alpha+\bar\nu_T+O(N_T^{-1/2})$, which is~\eqref{eq:main-anytime}. When $N_T=0$, $R_T^{\mathrm{act}}=0\le\alpha$ and the bound holds trivially. Setting $\eta=\delta$ gives the claimed probability $\ge 1-2\delta$.
\end{proof}

\subsection{Power guarantee (Theorem~\ref{thm:power})}
\label{app:proof-power}

\begin{theorem*}[Power guarantee (Theorem~\ref{thm:power}, restated)]
Fix epoch $j$ and $q\le q_j^\star$ satisfying Assumption~\ref{ass:safe-margin} (safe-side margin) with $\bar\eta_{j,q}>\nu_j$. Using the drift-adjusted e-process with fixed bet $\lambda^\star=(\bar\eta_{j,q}-\nu_j)/2$,
\[
\EE\!\bigl[\tau_{j,q}^{\mathrm{cert}}-\tau_j\bigr]
\;\le\; \frac{4(\log(1/\delta_{j,q})+1)}{(\bar\eta_{j,q}-\nu_j)^2}.
\]
When $\xi_t\equiv 0$: $\nu_j{=}0$ and the bound reduces to $4(\log(1/\delta_{j,q})+1)/\bar\eta_{j,q}^2$.
\end{theorem*}

\begin{proof}
We prove the $\nu_j=0$ case; the drift case replaces $\bar\eta_{j,q}$ with $\bar\eta_{j,q}-\nu_j$ throughout.

Write $Z_n:=\log E_{j,\tau_j+n}(q)$ for $n\ge 0$, with $Z_0=0$, and let $\tau:=\inf\{n\ge 0:Z_n\ge a\}$ for $a=\log(1/\delta_{j,q})$. Each step satisfies $|\Delta Z_n|\le 2\lambda^\star=\bar\eta_{j,q}\le 1$ (since $|u|\le 1/2$ gives $|\log(1-u)|\le 2|u|$).

Since $|\lambda^\star X_{\tau_j+n}|\le 1/2$, the inequality $\log(1-u)\ge -u-u^2$ gives
\[
\EE[\Delta Z_n\mid\cF_{\tau_j+n-1}]
\ge \lambda^\star\bar\eta_{j,q}-(\lambda^\star)^2
= \tfrac{\bar\eta_{j,q}^2}{2}-\tfrac{\bar\eta_{j,q}^2}{4}
= \tfrac{\bar\eta_{j,q}^2}{4} =: c.
\]

Decompose $Z_n=D_n+\Sigma_n$ with $D_n\ge cn$ and $\Sigma_n$ a zero-mean martingale with $|\Delta\Sigma_s|\le 2\bar\eta_{j,q}$. Azuma-Hoeffding on $\Sigma_n$ gives $\PP(\Sigma_n<-\epsilon)\le\exp(-\epsilon^2/(8\bar\eta_{j,q}^2 n))$. For $n\ge 2a/c$,
\[
\PP(\tau>n)\le\PP(Z_n<a)\le\PP(\Sigma_n<a-cn)\le\exp\!\Bigl(-\frac{(cn-a)^2}{8\bar\eta_{j,q}^2 n}\Bigr)\le\exp\!\Bigl(-\frac{c^2 n}{32\bar\eta_{j,q}^2}\Bigr),
\]
which is geometric, so $\EE[\tau]\le 2a/c+\sum_n e^{-\bar\eta_{j,q}^2 n/512}<\infty$.

$M_n:=Z_n-cn$ is a submartingale with $M_0=0$; $|M_{n\wedge\tau}|\le(1+c)\tau$ with $\EE[\tau]<\infty$. OST gives $\EE[M_\tau]\ge 0$, i.e.\ $\EE[Z_\tau]\ge c\EE[\tau]$.

$Z_\tau\ge a$ by definition, and $Z_{\tau-1}<a$ with $|\Delta Z_\tau|\le 1$ gives $Z_\tau<a+1$. Hence $\EE[Z_\tau]\le a+1$ and
\[
c\EE[\tau]\le\EE[Z_\tau]\le a+1 \implies \EE[\tau]\le\frac{a+1}{c}=\frac{4(\log(1/\delta_{j,q})+1)}{\bar\eta_{j,q}^2}. \qedhere
\]
\end{proof}

\begin{proposition}[High-probability certification delay]
\label{prop:hp-cert-delay}
Under the conditions of Theorem~\ref{thm:power} (with the same fixed bet $\lambda^\star=(\bar\eta_{j,q}-\nu_j)/2$), for any $\eta\in(0,1)$,
\[
\PP\!\left(\tau_{j,q}^{\mathrm{cert}}-\tau_j>\frac{C_0(\log(1/\delta_{j,q})+\log(1/\eta))}{(\bar\eta_{j,q}-\nu_j)^2}\right)\le\eta,
\]
with universal constant $C_0=2048$.
\end{proposition}

\begin{proof}
We prove the $\nu_j=0$ case; the drift case replaces $\bar\eta_{j,q}$ by $\bar\eta_{j,q}-\nu_j$ throughout. Adopt all notation from the proof of Theorem~\ref{thm:power}: $Z_n=\log E_{j,\tau_j+n}(q)$, $\tau=\inf\{n\ge 0:Z_n\ge a\}$ with $a=\log(1/\delta_{j,q})$, $c=\bar\eta_{j,q}^2/4$, $|\Delta Z_n|\le\bar\eta_{j,q}$, and the Doob decomposition $Z_n=D_n+\Sigma_n$ with $D_n\ge cn$ and $\Sigma_n$ a zero-mean martingale with $|\Delta\Sigma_s|\le 2\bar\eta_{j,q}$.

Write $L:=\log(1/\eta)$. We show $\PP(\tau>n_0)\le\eta$ for $n_0:=\max(8a,512L)\cdot\bar\eta_{j,q}^{-2}$, which implies the stated bound since $n_0\le 512(a+L)/\bar\eta_{j,q}^{2}\le C_0(a+L)/\bar\eta_{j,q}^2$ with $C_0=2048$.

For any $n$ with $cn>a$: $\PP(\tau>n)\le\PP(Z_n<a)=\PP(\Sigma_n<a-D_n)\le\PP(\Sigma_n<-(cn-a))$. Azuma--Hoeffding gives $\PP(\Sigma_n<-(cn-a))\le\exp\!\bigl(-(cn-a)^2/(8\bar\eta_{j,q}^2 n)\bigr)$. It suffices to show the exponent exceeds $L$ at $n=n_0$.

\textbf{Case 1:} $a\ge 64L$, so $n_0=8a/\bar\eta_{j,q}^2$. Then $cn_0=2a$, $cn_0-a=a$, and $8\bar\eta_{j,q}^2 n_0=64a$. Exponent: $a^2/(64a)=a/64\ge L$.

\textbf{Case 2:} $a<64L$, so $n_0=512L/\bar\eta_{j,q}^2$. Then $cn_0=128L$, $cn_0-a>64L$, and $8\bar\eta_{j,q}^2 n_0=4096L$. Exponent: $(64L)^2/(4096L)=L$.

In both cases $\PP(\tau>n_0)\le e^{-L}=\eta$.
\end{proof}

\subsection{Lower bound on certification time (Theorem~\ref{thm:lower-bound})}
\label{app:proof-lower}

\begin{theorem*}[Lower bound on certification time (Theorem~\ref{thm:lower-bound}, restated)]
Fix $\bar\eta\in(0,\min(\alpha,1-\alpha))$ and $\delta\in(0,1/4)$. For any stopping time $\tau$ that separates the worst-case safe and null distributions with error $\le\delta$,
\[
\EE[\tau]\ge\Omega\!\bigl(\bar\eta^{-2}\log(1/\delta)\bigr),
\]
and under declared drift $\nu$, the bound extends to $\Omega((\bar\eta-\nu)^{-2}\log(1/\delta))$.
\end{theorem*}

\begin{proof}
We prove $\EE_\mathsf{P}[\tau]\ge\KL(1-\delta\|\delta)/\KL(\mathsf{P}\|\mathsf{Q})$; the RHS equals $\Omega(\bar\eta^{-2}\log(1/\delta))$ by standard binary-KL expansion ($\KL(\mathsf{P}\|\mathsf{Q})=\bar\eta^2/(2\alpha(1-\alpha))+O(\bar\eta^3)$).

For any $m\ge 1$, $\tau\wedge m$ is bounded; with $L_n:=\sum_{t\le n}\log(d\mathsf{P}/d\mathsf{Q})(X_t)$ and $|\ell_t|\le\log\max(\frac{\alpha-\bar\eta}{\alpha},\frac{1-\alpha+\bar\eta}{1-\alpha})<\infty$, Wald's identity gives $\EE_\mathsf{P}[L_{\tau\wedge m}]=\KL(\mathsf{P}\|\mathsf{Q})\EE_\mathsf{P}[\tau\wedge m]$. Monotone convergence as $m\to\infty$ yields $\KL(\mathsf{P}\|\mathsf{Q})\EE_\mathsf{P}[\tau]=\lim_m\EE_\mathsf{P}[L_{\tau\wedge m}]$.

For $d_m:=\1\{\tau\le m\}$ and $p_m:=\PP_\mathsf{P}(d_m=1)$, the data-processing inequality gives $\KL(p_m\|\PP_\mathsf{Q}(d_m=1))\le\EE_\mathsf{P}[L_{\tau\wedge m}]$. Type-I error control under $\mathsf{Q}$ gives $\PP_\mathsf{Q}(\tau<\infty)\le\delta$, so $\PP_\mathsf{Q}(d_m=1)\le\delta$ for all $m$.

Since $p_m\nearrow\PP_\mathsf{P}(\tau<\infty)\ge 1-\delta$ and $\PP_\mathsf{Q}(d_m=1)\le\delta<(1-\delta)/2\le p_m$ for large $m$, the binary KL satisfies $\KL(p_m\|q)\ge\KL(1-\delta\|\delta)$. Passing to the limit and dividing by $\KL(\mathsf{P}\|\mathsf{Q})$ gives the stated bound.
\end{proof}

\subsection{Utility-gap bound (Theorem~\ref{thm:utility-gap})}
\label{app:proof-utility}

\begin{theorem*}[Utility gap under bounded drift (Theorem~\ref{thm:utility-gap}, restated)]
Let $\mathrm{Gap}_T:=\sum_t A_t(q_t^\star)-\sum_t A_t(q_t)$. Under the RLVR structural layer (Appendix~\ref{app:rlvr-layer}) with slack $\{\xi_t\}$, CSA-Epoch with analysis-only boundaries at frontier drops uses $J_T\le D_T+1$ epochs where $D_T$ is the number of drops. With probability $\ge 1-2\delta$ for all $T\ge 1$:
\[
\mathrm{Gap}_T \;\le\; G^\star(B_T)
\;:=\; \sum_{j=1}^{J_T}\sum_{q\le q_j^\star}\frac{2C_0\log(2m J_T^2/\delta)}{(\bar\eta_{j,q}-\Xi_j)_+^2}.
\]
When $\xi_t\equiv 0$, $J_T{=}1$ and $\Xi_j{=}0$, so
$G^\star(B_T)=G^\star=\sum_{q\le q^\star}2C_0\log(2m/\delta)/\bar\eta_q^2$
is horizon-independent. If $B_\infty:=\sum_{t\ge 1}\xi_t<\infty$, $J_T$ is bounded and $\mathrm{Util}_T/\mathrm{Util}_T^\star\to 1$.
\end{theorem*}

\begin{proof}
This proof handles the general approximate-improvement case ($\xi_t\ge 0$) using epoch restarts at frontier-drop events. The exact-improvement special case ($\xi_t\equiv 0$) is recovered by setting $J_T=1$, $\Xi_j=0$.

Place epoch boundaries at the (analysis-only) frontier-drop events $\{t:q_{t+1}^\star<q_t^\star\}$. By Theorem~\ref{thm:monotone-frontier}, $J_T\le D_T+1$ epochs, and within each epoch $\cI_j$ the frontier $q_t^\star$ is nondecreasing by construction. Corollary~\ref{cor:single-epoch} then applies inside each $\cI_j$: false-certification probability for each $(j,q)$ is at most $\delta_{j,q}$, with budget $\delta_{j,q}:=6\delta/(\pi^2|\cQ|j^2)$ satisfying $\sum_j\sum_q\delta_{j,q}\le\delta$, so the good event $\cG$ (no false certification across all epochs) has $\PP(\cG)\ge 1-\delta$. On $\cG$, $q_t\le q_t^\star$ for all $t$, so $A_t(q_t)\le A_t(q_t^\star)$ by Assumption~\ref{ass:nested}, giving $\mathrm{Gap}_T = \mathrm{Util}_T^\star - \mathrm{Util}_T \ge 0$ on $\cG$.

Within each epoch $\cI_j$, the frontier is nondecreasing, so Corollary~\ref{cor:single-epoch} ensures each threshold certifies at most once per epoch. Let $q_j^\star:=q_{\tau_j}^\star$ and $\cQ_j^\le:=\{q\in\cQ:q\le q_j^\star\}$. For each $q\in\cQ_j^\le$, let $T_{j,q}:=\inf\{t\in\cI_j:q\le q_t^\star\}$ and $\tau_{j,q}^{\mathrm{cert}}$ the certification time within epoch $j$. The nondecreasing frontier within epoch $j$ ensures $q\le q_s^\star$ for all $s\ge T_{j,q}$ within $\cI_j$, so the safe-side margin holds persistently and the e-process drifts to certification without resetting. A gap round in epoch $j$ occurs only while some $q\in\cQ_j^\le$ has $T_{j,q}\le t<\tau_{j,q}^{\mathrm{cert}}$ (under exact stability $\xi_t\equiv 0$ this captures all gap rounds because the frontier is constant within each epoch; under positive drift the frontier may rise within $\cI_j$ and thresholds in $(q_j^\star,q_t^\star]$ may contribute additional gap, which is bounded analogously by the same per-threshold argument applied to $\cQ_j^{\mathrm{safe}}:=\{q:q\le\sup_{t\in\cI_j}q_t^\star\}$). Therefore, on the start-of-epoch-safe subset:
\[
\mathrm{Gap}_T \le \sum_{j=1}^{J_T}\sum_{q\in\cQ_j^\le}
\bigl(\tau_{j,q}^{\mathrm{cert}} - T_{j,q}\bigr)_+
\quad \text{on }\cG.
\]

For each fixed epoch $j$ and $q\in\cQ_j^\le$, after time $T_{j,q}$ the safe-side margin $\bar\eta_{j,q}$ holds persistently, so applying Proposition~\ref{prop:hp-cert-delay} with effective margin $\bar\eta_{j,q}-\Xi_j$ (Corollary~\ref{cor:derived-epoch-frontier}) and $\eta_{j,q} := \delta/(2mJ_T^2)$,
\[
\PP\!\left(\tau_{j,q}^{\mathrm{cert}}-T_{j,q}
> \frac{C_0(\log(1/\delta_{j,q})+\log(1/\eta_{j,q}))}{(\bar\eta_{j,q}-\Xi_j)_+^2}\right)
\le \eta_{j,q}.
\]
Since $\delta_{j,q}=6\delta/(\pi^2 m j^2)\ge\delta/(2mJ_T^2)=\eta_{j,q}$, the numerator satisfies $\log(1/\delta_{j,q})+\log(1/\eta_{j,q})\le 2\log(2mJ_T^2/\delta)$. Within each epoch $j$, union-bounding over $q\in\cQ_j^\le$ (at most $m$ thresholds):
\[
\PP\!\left(\exists\,q\in\cQ_j^\le:
\tau_{j,q}^{\mathrm{cert}}-T_{j,q}
> \frac{2C_0\log(2mJ_T^2/\delta)}{(\bar\eta_{j,q}-\Xi_j)_+^2}\right)
\le \frac{\delta}{2J_T}.
\]
Summing over all $J_T$ epochs (total certification-delay failure probability $\le\delta/2\le\delta$) and combining with $\PP(\cG)\ge 1-\delta$ gives: with probability $\ge 1-2\delta$,
\[
\mathrm{Gap}_T
\le \sum_{j=1}^{J_T}\sum_{q\le q_j^\star}
\frac{2C_0\log(2mJ_T^2/\delta)}{(\bar\eta_{j,q}-\Xi_j)_+^2}
= G^\star(B_T)
\]
simultaneously for all $T\ge 1$. When $\xi_t\equiv 0$: $J_T=1$, $\Xi_j=0$, and $G^\star(B_T)=G^\star=\sum_{q\le q^\star}2C_0\log(2m/\delta)/\bar\eta_q^2$ (horizon-independent).
\end{proof}

\section{RLVR structural layer}
\label{app:rlvr-layer}

Theorem~\ref{thm:main-anytime} treats declared drift budgets $\{\nu_j\}$ as free controller parameters. This appendix fills in the RLVR-specific machinery that \emph{derives} data-dependent drift bounds $\{\varepsilon_j\}$ from a local structural condition on the RLVR policy; setting $\nu_j:=\varepsilon_j$ instantiates the bound, yielding bounded frontier drift and a horizon-independent utility gap in the exact-stability limit.

\subsection{Assumption summary}
\label{app:assumption-summary}

Table~\ref{tab:assumptions} catalogs all assumptions used in the paper, organized by scope and target theorem.

\begin{table}[h]
\centering\small
\setlength{\tabcolsep}{4pt}
\caption{Assumption summary. \emph{Core} assumptions are needed for all safety guarantees. \emph{Structural} conditions are standard regularity conditions on the RLVR stream used for bounded frontier drift and single-epoch sufficiency. \emph{Sharpening} conditions enter only for the power and gap bounds.}
\label{tab:assumptions}
\resizebox{\textwidth}{!}{\begin{tabular}{llll}
\toprule
Assumption & Scope & Verified by & Used in \\
\midrule
\ref{ass:predictable} (Predictable pipeline)  & Core & By construction (lagged buffer)   & All theorems \\
\ref{ass:nested} (Nested gates)               & Core & By construction (fixed real score) & All theorems \\
Cor.~\ref{cor:frontier-existence} (Safe frontier) & Derived & From \ref{ass:subpop-improvement} + \ref{ass:local-slope} & Validity, anytime \\
\ref{ass:surrogate-calibration} (Calibration) & Structural & ECE/Brier on deployment stream & Cor.~\ref{cor:frontier-existence} \\
\ref{ass:local-slope} (Slope + root)          & Structural & Nondegenerate score CDF & Cor.~\ref{cor:frontier-existence} \\
\ref{ass:actuation-prob} (Actuation prob.)    & Structural & $\min_k\PP(s_t\le q)>0$ & Cor.~\ref{cor:frontier-existence} \\
\ref{ass:subpop-improvement} (Near-frontier stab.) & Structural & $D_T$, $B_T$ on live stream & Thm.~\ref{thm:monotone-frontier} \\
\ref{ass:unsafe-margin} (Unsafe margin)       & Sharpening & From Ass.~\ref{ass:local-slope} & Power theorem \\
\ref{ass:safe-margin} (Safe margin)           & Sharpening & Cor.~\ref{cor:frontier-existence}(iii) & Power, gap \\
\ref{ass:actuation-gap} (Actuation gap)       & Sharpening & Score mass in $(q,q^\star]$ & Gap lower bound \\
\bottomrule
\end{tabular}}
\end{table}

\subsection{Deferred utility assumptions}
\label{app:deferred-assumptions}

The following two utility-side assumptions are stated in the main body by name only and defined here for completeness.

\begin{assumption}[Unsafe-side margin]
\label{ass:unsafe-margin}
For each epoch $j$ and every threshold $q > q_j^\star$, there exists $\eta_{j,q} > 0$ such that $\EE[X_t(q)\mid\cF_{t-1}]\ge\eta_{j,q}$ for all $t\in\cI_j$. (Symmetric to the safe-side margin of Assumption~\ref{ass:safe-margin}; used in the power / no-false-certification analysis.)
\end{assumption}

\begin{assumption}[Monotone actuation gap]
\label{ass:actuation-gap}
For each epoch $j$ and every $q<q_j^\star$, there exists $\Delta_{j,q}>0$ such that $\PP(A_t(q_j^\star)=1,\,A_t(q)=0\mid\cF_{t-1})\ge\Delta_{j,q}$ for all $t\in\cI_j$. Used to lower-bound the utility gap during the certification window (Theorem~\ref{thm:utility-gap}); not needed for any safety guarantee.
\end{assumption}

\subsection{Standing assumptions}
\label{app:rlvr-assumptions}

\begin{assumption}[Surrogate approximation quality]
\label{ass:surrogate-calibration}
There exists a predictable $\varepsilon_t\ge 0$, $\varepsilon_t\to 0$ a.s., such that $\sup_{x,y}|\hat p_t(x,y)-\PP(V(x,y){=}1\mid\cF_{t-1})|\le\varepsilon_t$ a.s.\ for all $t$.
\end{assumption}

\begin{assumption}[Local slope]
\label{ass:local-slope}
There exist $\underline\kappa>0$ and $\rho>0$ such that $r_t(\cdot)$ is strictly increasing on $[\bar q_t-\rho,\bar q_t+\rho]$ with slope $\ge\underline\kappa$, where $\bar q_t$ is the (continuous) root of $r_t(q){=}\alpha$.
\end{assumption}

\begin{assumption}[Actuation probability]
\label{ass:actuation-prob}
$\min_q\PP(s_t\le q\mid\cF_{t-1})>0$ for all $t$ and all $q$ strictly above the grid minimum.
\end{assumption}

\begin{assumption}[Local near-frontier stability]
\label{ass:subpop-improvement}
There exists a predictable $\xi_t\ge 0$ such that, whenever a threshold $q$ lies within $\rho$ of the oracle frontier $q_t^\star$,
\[
\PP(V_{t+1}{=}1\mid S_{t+1}\le q,\cF_t) \ge \PP(V_t{=}1\mid S_t\le q,\cF_{t-1}) - \xi_t.
\]
When $\xi_t\equiv 0$ this recovers exact stability; the condition is local in the score region that determines deployment decisions and is not a statement about global improvement.
\end{assumption}

\begin{assumption}[Safe-side margin]
\label{ass:safe-margin}
For each epoch $j$ and $q\le q_j^\star$: $\EE[X_t(q)\mid\cF_{t-1}]\le-\bar\eta_{j,q}<0$ for all $t\in\cI_j$.
\end{assumption}

The cumulative slack is $B_T:=\sum_{t=1}^{T-1}\xi_t$; the live experiment operates at $B_T\approx 0.40=O(1)$ (Appendix~\ref{app:live-rlvr-extras}).

\subsection{Frontier existence and localization}
\label{app:frontier-existence}

\begin{corollary}[Frontier existence]
\label{cor:frontier-existence}
Under Assumption~\ref{ass:local-slope} with $q_{\min}<\bar q_t$, the set $\{q\in\cQ:r_t(q)\le\alpha\}$ is nonempty; its maximum $q_t^\star$ is well-defined. Under Assumption~\ref{ass:surrogate-calibration} the surrogate root $\hat{\bar q}_t$ exists uniquely in $[\bar q_t-\rho,\bar q_t+\rho]$ with $|\hat{\bar q}_t-\bar q_t|\le\varepsilon_t/\underline\kappa$.
\end{corollary}

\begin{proof}
Throughout, $r_t(q):=\PP(V_t=0\mid A_t(q)=1,\cF_{t-1})$ is the conditional failure rate as in Assumption~\ref{ass:local-slope}.

\textbf{(i) Existence.}
Since $q_{\min} < \bar{q}_t$ and $r_t' \ge \underline{\kappa} > 0$, we have
$r_t(q_{\min}) < r_t(\bar{q}_t) = \alpha$,
so $q_{\min} \in \{q\in\cQ: r_t(q)\le\alpha\}$ and the set is nonempty.
Its maximum $q_t^\star$ is well-defined.
Since $r_t$ is strictly increasing on
$[\bar{q}_t-\rho,\bar{q}_t+\rho]$ and the grid is finite,
the boundary between $\{r_t\le\alpha\}$ and $\{r_t>\alpha\}$ within $\cQ$
is a single point.

\textbf{(ii) Surrogate root localization.}
Since $|\hat{r}_t(q) - r_t(q)| \le \varepsilon_t$ for all $q$
(Assumption~\ref{ass:surrogate-calibration}), we have
$\hat{r}_t(\bar{q}_t) \in [\alpha - \varepsilon_t, \alpha + \varepsilon_t]$.
Because $\hat{r}_t$ is strictly increasing on $[\bar{q}_t-\rho,\bar{q}_t+\rho]$
with slope $\ge \underline{\kappa}$, the intermediate value theorem gives a
unique $\hat{\bar{q}}_t$ in this interval with $\hat{r}_t(\hat{\bar{q}}_t)=\alpha$,
and the slope condition implies
$|\hat{\bar{q}}_t - \bar{q}_t| \le \varepsilon_t/\underline{\kappa}$.

\textbf{(iii) Safe-side margin.}
Fix $q < q_t^\star$ within $\rho$ of $\bar{q}_t$.
Since $r_t(q_t^\star) \le \alpha = r_t(\bar{q}_t)$ and $r_t$ is strictly
increasing, $q_t^\star \le \bar{q}_t$.
Combined with $q < q_t^\star$, the interval $[q, \bar{q}_t]$ lies within the
slope region.
By the mean value theorem applied to $r_t$ on $[q,\bar{q}_t]$:
\[
r_t(q) - \alpha
= r_t(q) - r_t(\bar{q}_t)
= r_t'(\zeta)(q - \bar{q}_t)
\le -\underline{\kappa}(\bar{q}_t - q)
\le -\underline{\kappa}(q_t^\star - q),
\]
where $\zeta \in (q,\bar{q}_t)$ and the last inequality uses $\bar{q}_t \ge q_t^\star$.
Hence
$\EE[X_t(q)\mid\cF_{t-1}]
= \PP(A_t(q)=1\mid\cF_{t-1})\cdot(r_t(q)-\alpha)
\le -p_{\min}\underline{\kappa}(q_t^\star-q)$.
\end{proof}

\subsection{Bounded frontier drift}

\begin{theorem}[Bounded frontier drift]
\label{thm:monotone-frontier}
Under Assumptions~\ref{ass:predictable}--\ref{ass:nested} and a local near-frontier stability condition (Assumption~\ref{ass:subpop-improvement}) with slack $\xi_t\ge 0$,
\[
q_{t+1}^\star \;\ge\; q_t^\star - \xi_t/\underline\kappa - \Delta q \qquad\text{a.s.\ for all }t\ge 1,
\]
where $\underline\kappa>0$ is the local slope of the failure-rate curve and $\Delta q$ is the grid spacing. When $\xi_t\equiv 0$, the oracle frontier $q_t^\star$ is nondecreasing.
\end{theorem}

\begin{corollary}[Single-epoch sufficiency]
\label{cor:single-epoch-sufficiency}
\label{cor:single-epoch}
Under exact stability ($\xi_t\equiv 0$), a single epoch suffices: certifications are permanent, no restarts are needed, and the utility gap (Theorem~\ref{thm:utility-gap}) is horizon-independent.
\end{corollary}

\begin{proof}[Proof of Theorem~\ref{thm:monotone-frontier}]
The full statement (using Assumptions~\ref{ass:predictable}, \ref{ass:nested}, \ref{ass:local-slope}, \ref{ass:actuation-prob}, \ref{ass:subpop-improvement} with $\xi_t\le\underline\kappa\rho$ a.s.) gives the per-step bound $q_{t+1}^\star\ge q_t^\star-\xi_t/\underline\kappa-\Delta q$ a.s.\ for all $t\ge 1$, where $\Delta q:=\max_i(q^{(i+1)}-q^{(i)})$ is the grid spacing, and total downward movement $\sum_{t=1}^{T-1}(q_t^\star-q_{t+1}^\star)_+\le B_T/\underline\kappa+(T-1)\Delta q$. When $\xi_t\equiv 0$ and $q_t^\star$ is on the grid, $q_{t+1}^\star\ge q_t^\star$ (nondecreasing).

Since $r_{t+1}(q_t^\star)\le r_t(q_t^\star)+\xi_t\le\alpha+\xi_t$ on the near-frontier (Assumption~\ref{ass:subpop-improvement} with $\xi_t\le\underline\kappa\rho$ ensuring we stay in the slope neighborhood), the IVT on $r_{t+1}$ with slope $\ge\underline\kappa$ (Assumption~\ref{ass:local-slope}) localizes the continuous frontier $\bar q_{t+1}$:
\[
\bar q_{t+1}\ge\bar q_t-\xi_t/\underline\kappa.
\]
Rounding to the grid (taking the largest grid point at most $\bar q_{t+1}$) loses at most $\Delta q$:
\[
q_{t+1}^\star\ge\bar q_{t+1}-\Delta q\ge\bar q_t-\xi_t/\underline\kappa-\Delta q\ge q_t^\star-\xi_t/\underline\kappa-\Delta q,
\]
where the final inequality uses $q_t^\star\le\bar q_t$ (the grid threshold $q_t^\star$ is the largest $q\in\cQ$ with $r_t(q)\le\alpha$, hence at most the continuous root $\bar q_t$). Summing the positive parts across $t$ gives the total-movement bound. When $\xi_t\equiv 0$ and $q_t^\star\in\cQ$ already, both correction terms vanish.
\end{proof}

\begin{corollary}[Derived epoch frontier]
\label{cor:derived-epoch-frontier}
Under the hypotheses of Theorem~\ref{thm:monotone-frontier}, for any epoch $\cI_j=\{\tau_j,\ldots,\tau_{j+1}-1\}$ with slack $\Xi_j:=\sum_{t\in\cI_j}\xi_t$ and reference $q_j^\star:=q_{\tau_j}^\star$:
\begin{enumerate}[leftmargin=1.5em,nosep]
\item for all $q\le q_j^\star$ and $t\in\cI_j$: $r_t(q)-\alpha\le\Xi_j$, equivalently $\EE[X_t(q)\mid\cF_{t-1}]\le\PP(A_t(q){=}1\mid\cF_{t-1})\,\Xi_j\le\Xi_j$;
\item for all $q>q_j^\star$ and $t\in\cI_j$ with $q>q_t^\star$: $r_t(q)\ge\alpha$, equivalently $\EE[X_t(q)\mid\cF_{t-1}]\ge 0$.
\end{enumerate}
This yields structural drift bounds $\varepsilon_j:=\Xi_j$; setting $\nu_j:=\varepsilon_j$ instantiates Theorem~\ref{thm:main-anytime}.
\end{corollary}

\subsection{Single-epoch sufficiency under exact stability}

\begin{proof}[Proof of Corollary~\ref{cor:single-epoch}]
When $\xi_t\equiv 0$, by Theorem~\ref{thm:monotone-frontier} with on-grid $q_t^\star$, $q_{t+1}^\star\ge q_t^\star$ a.s., so a single epoch $\cI_1=\{1,2,\ldots\}$ suffices: $q_t^\star$ is nondecreasing throughout. For any $q$ certified at time $s$, $q\le q_s^\star\le q_t^\star$ for all $t\ge s$, so the safety condition $q\le q_t^\star$ persists and the e-process remains a supermartingale. Union bound across the single epoch and $m$ thresholds gives $\delta_q=\delta/(2m)$, hence Theorem~\ref{thm:main-anytime} holds with $\bar\nu_T{=}0$.
\end{proof}

\subsection{Empirical validation on the live-RLVR LoRA experiment}
\label{app:frontier-diagnostic}

Figure~\ref{fig:frontier-diagnostic} shows that on $800$ genuinely generated LoRA-SFT rounds the estimated oracle frontier $q_k^\star$ oscillates with $D_T=3$ drops at rounds $2,4,6$, the per-round slack $\xi_k$ is nonzero only at drop rounds, and the cumulative slack is $B_T=\sum\xi_k\approx 0.40$, consistent with the $B_T=O(1)$ regime of Theorem~\ref{thm:utility-gap}.

\begin{figure}[h]
\centering
\includegraphics[width=0.85\textwidth]{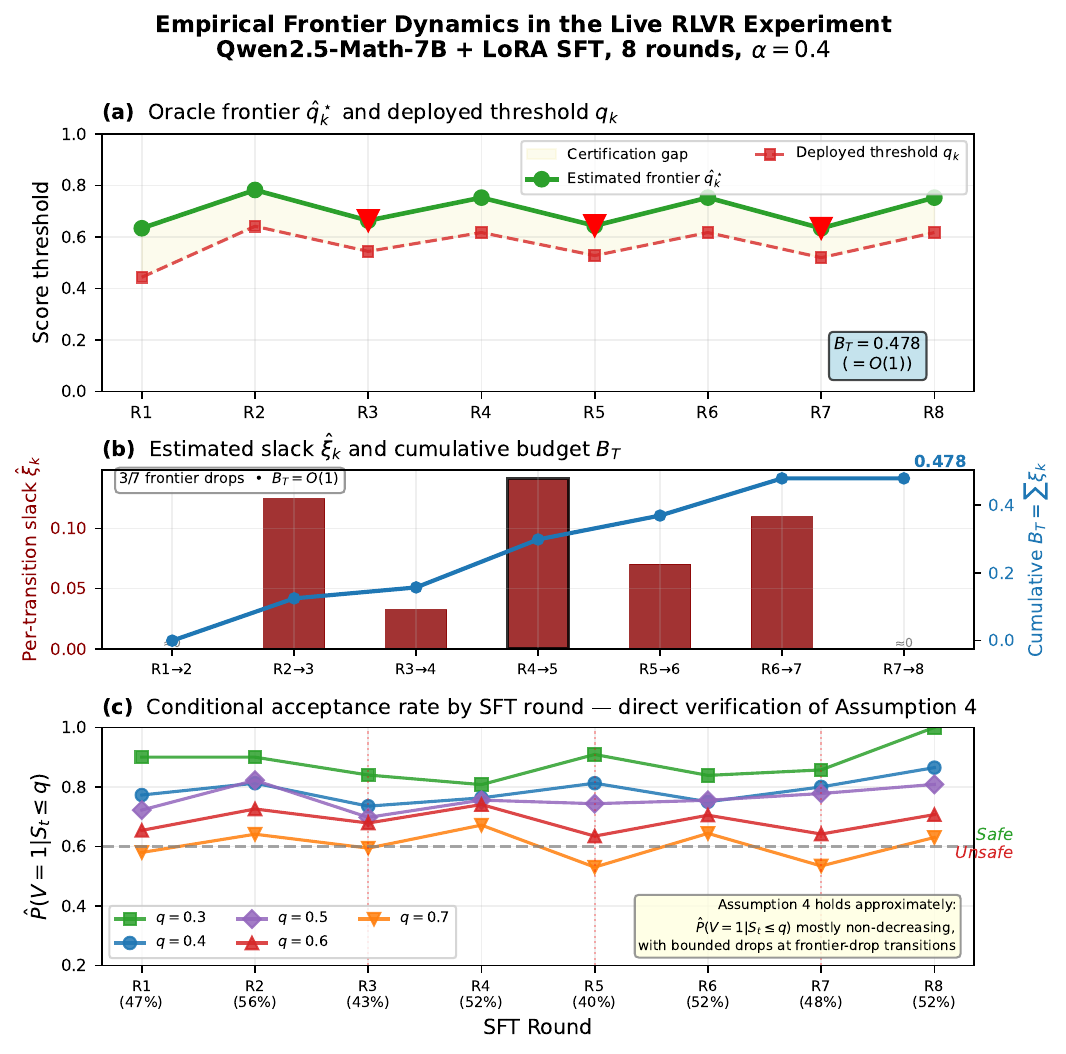}
\caption{\textbf{Empirical validation of the RLVR structural layer on the live LoRA experiment} ($K{=}8$ SFT rounds, $100$ MATH problems each). (a) Estimated oracle frontier $q_k^\star$; red dashed lines mark drop rounds ($D_T{=}3$). (b) Per-round slack $\xi_k$; nonzero only at drop rounds. (c) Cumulative slack $B_T=\sum\xi_k\approx 0.40$.}
\label{fig:frontier-diagnostic}
\end{figure}

\paragraph{Verifying the structural conditions in practice.}
All structural conditions are empirically falsifiable from the deployment stream itself, requiring no external oracle:
\emph{Calibration} (Assumption~\ref{ass:surrogate-calibration}): monitor ECE and Brier score of the surrogate on the growing buffer; a rising ECE signals recalibration is needed.
\emph{Actuation probability} (Assumption~\ref{ass:actuation-prob}): compute $\min_k\PP(s_t\le q)$ across recent rounds; a value near zero on the active grid indicates score collapse.
\emph{Near-frontier stability} (Assumption~\ref{ass:subpop-improvement}): track the frontier-drop count $D_T$ and cumulative slack $B_T=\sum\xi_t$; if $B_T$ grows faster than $O(1)$, widen epochs or increase the declared drift budget $\nu_j$.
In the live experiment, all three diagnostics confirm the conditions hold: ECE improves $0.073\to 0.050$, $\min_k\PP(s_t\le 0.5)=0.380$, and $B_T\approx 0.40=O(1)$.



\section{Sparse-verifier variant}
\label{app:sparse-theory}

\begin{theorem}[Sparse-verifier validity and delay]
\label{thm:sparse-verifier}
Subsample verifier labels by predictable Bernoulli $B_t\sim\mathrm{Bernoulli}(\pi_t)$ with $\pi_t\ge\pi_{\min}>0$ a.s., and use the importance-weighted increment $\tilde X_t(q):=(B_t/\pi_t)X_t(q)$. Define $\tilde E_{j,t}(q):=\prod_{s=\tau_j}^t(1-\lambda_{j,s}(q)\tilde X_s(q))$ with $\lambda_{j,s}(q)\in[0,\pi_{\min}/(1-\alpha)]$. Then: (i)~$\tilde E_{j,t}(q)$ is a nonnegative supermartingale under $H_{j,q}^{\mathrm{unsafe}}$; (ii)~all anytime-valid guarantees of Theorem~\ref{thm:main-anytime} hold verbatim; (iii)~the expected certification-delay upper bound inflates by at most a factor $1/\pi_{\min}$ relative to full verification, and exactly $1/\pi$ when $\pi_t\equiv\pi$.
\end{theorem}

\begin{proof}
Write $\cF_t^B := \sigma(\cF_t,\,B_1,\ldots,B_t)$ for the augmented filtration. By construction, $\pi_t$ and $\lambda_{j,t}(q)$ are $\cF_{t-1}^B$-measurable, and $B_t$ is $\cF_{t-1}^B$-conditionally independent of $X_t(q) = A_t(q)((1-V_t)-\alpha)$ with $\EE[B_t \mid \cF_{t-1}^B] = \pi_t$.

\textbf{Null transfer.}
By independence of $B_t$ and $X_t(q)$ given $\cF_{t-1}^B$,
\[
\EE\bigl[\tilde X_t(q)\,\big|\,\cF_{t-1}^B\bigr]
\;=\; \frac{1}{\pi_t}\,\EE[B_t\mid\cF_{t-1}^B]\cdot\EE[X_t(q)\mid\cF_{t-1}^B]
\;=\; \EE[X_t(q)\mid\cF_{t-1}],
\]
using that $X_t(q)$ is $\cF_t$-measurable and $\cF_t \subseteq \cF_t^B$. Therefore $\EE[\tilde X_t(q)\mid\cF_{t-1}^B]\ge 0$ under the unsafe null $\EE[X_t(q)\mid\cF_{t-1}]\ge 0$.

\textbf{Supermartingale property.}
Fix $q\in\cQ$, epoch $j$, and $t\in\cI_j$. The factor $(1-\lambda_{j,t}(q)\tilde X_t(q))$ is nonnegative: when $B_t=0$ the factor equals $1$; when $B_t=1$, nonnegativity follows from the clipping $\lambda_{j,t}(q)\le\pi_{\min}/(1-\alpha)$ applied to $|X_t(q)/\pi_t|\le(1-\alpha)/\pi_{\min}$. Taking conditional expectation under the null and using Step~1,
\[
\EE\bigl[1-\lambda_{j,t}(q)\tilde X_t(q)\,\big|\,\cF_{t-1}^B\bigr]
\;=\; 1-\lambda_{j,t}(q)\,\EE[\tilde X_t(q)\mid\cF_{t-1}^B]
\;\le\; 1,
\]
so $\EE[\tilde E_{j,t}(q)\mid\cF_{t-1}^B]\le\tilde E_{j,t-1}(q)$, making $\tilde E_{j,t}(q)$ a nonnegative $(\cF_t^B)$-supermartingale with $\tilde E_{j,\tau_j-1}(q)=1$.

\textbf{Validity transfer.}
Ville's inequality on this supermartingale gives $\PP(\sup_t \tilde E_{j,t}(q)\ge\delta_{j,q}^{-1})\le\delta_{j,q}$, exactly as for the full-evaluation $E_{j,t}(q)$. The union bound and epoch-level arguments of Theorem~\ref{thm:no-false-cert} transfer without modification, and the exponential supermartingale argument of Lemma~\ref{lem:exp-supermg} and Theorem~\ref{thm:main-anytime} applies to $M_t := A_t(q_t)((1-V_t)-\alpha)$ unchanged, because $M_t$ does not depend on the subsampling coin (the coin affects only which e-processes are updated, not the deployed threshold $q_t$).

\textbf{Power bound.}
For the fixed strategy $\lambda_{j,t}(q)\equiv\lambda^\star:=\pi_{\min}\bar\eta_{j,q}/2$, let $Z_n := \log\tilde E_{j,\tau_j+n}(q)$ with $Z_0=0$. Each increment satisfies $|\Delta Z_n|\le 2\lambda^\star/\pi_{\min}\le \bar\eta_{j,q}\le 1$ (using $|\tilde X_n(q)|\le 1/\pi_{\min}$ and $|\log(1-u)|\le 2|u|$ for $|u|\le 1/2$). Under Assumption~\ref{ass:safe-margin}, $\EE[\tilde X_t(q)\mid\cF_{t-1}^B]\le-\bar\eta_{j,q}$ by Step~1, and $\EE[\tilde X_t(q)^2\mid\cF_{t-1}^B]=(1/\pi_t)\EE[X_t(q)^2\mid\cF_{t-1}]\le 1/\pi_{\min}$. Applying $\log(1-u)\ge-u-u^2$ for $|u|\le 1/2$,
\begin{align*}
\EE[\Delta Z_n\mid\cF_{\tau_j+n-1}^B]
&\ge\lambda^\star\,\bar\eta_{j,q} - (\lambda^\star)^2/\pi_{\min} \\
&= \frac{\pi_{\min}\bar\eta_{j,q}^2}{2} - \frac{\pi_{\min}\bar\eta_{j,q}^2}{4}
\;=\; \frac{\pi_{\min}\bar\eta_{j,q}^2}{4}
\;=:\; c_{\pi}.
\end{align*}
Setting $a := \log(1/\delta_{j,q})$ and $\tau := \inf\{n\ge 0:Z_n\ge a\}$, the submartingale optional-stopping argument of Theorem~\ref{thm:power} gives
\[
\EE[\tau]\;\le\;\frac{a+1}{c_{\pi}}
\;=\;\frac{4(\log(1/\delta_{j,q})+1)}{\pi_{\min}\,\bar\eta_{j,q}^2},
\]
which is at most a factor $1/\pi_{\min}$ larger than the full-evaluation upper bound, and exactly $1/\pi$ when $\pi_t\equiv\pi$ is constant.
\end{proof}

\section{Extended empirical results}
\label{app:extended-exp}

\begin{table}[h]
\centering\footnotesize
\setlength{\tabcolsep}{3pt}
\caption{\textbf{Experiment map.} Each row: the experiment, what it validates, which theorems it touches, where to find it.}
\label{tab:exp-map}
\resizebox{\textwidth}{!}{\begin{tabular}{lllll}
\toprule
\textbf{Experiment} & \textbf{Validates} & \textbf{Theorems} & \textbf{Appendix} & \textbf{Main body} \\
\midrule
Eight benchmarks        & pathwise safety $+$ utility in the informative regime & \ref{thm:main-anytime} & \S\ref{app:bench-eight-extras}           & \S\ref{sec:bench-eight} \\
Live RLVR (LoRA)        & bounded drift; live policy updates                   & \ref{thm:monotone-frontier}, \ref{thm:utility-gap} & \S\ref{app:live-rlvr-extras} & \S\ref{sec:live-rlvr} \\
Sixteen-cell shift      & non-exchangeable orderings                           & \ref{thm:main-anytime} & \S\ref{app:shift}                       & \S\ref{sec:shift} \\
Offline real-model replay & published checkpoints, no re-training              & \ref{thm:main-anytime} & \S\ref{app:cross-model}                 & --- \\
Hyperparam.\ sweep      & insensitivity to $\delta$, $B$, $|\cQ|$              & ---                    & \S\ref{app:hparam}                      & --- \\
Cross-model (DeepSeek)  & RLVR-recipe invariance                               & \ref{thm:main-anytime} & \S\ref{app:cross-model}                 & --- \\
A-RCPS / CoFact / Conf-Arb. & marginal vs.\ selective; covariate-shift collapse & \S\ref{sec:framework}  & \S\ref{app:arcps}                       & --- \\
Sparse-verifier         & $1/\pi$ certification delay                          & \ref{thm:sparse-verifier} & \S\ref{app:sparse}                   & --- \\
CAP/OCP/LORD-CI         & head-to-head on shared selective metric              & \ref{thm:main-anytime} & \S\ref{app:cap}                         & --- \\
Semi-synthetic (MATH/LCB) & calibrated RLVR replays                            & \ref{thm:main-anytime} & \S\ref{app:cap}                        & --- \\
Ablations / stress        & grid, $\lambda$, surrogate bias, noisy verifier   & Rem.~\ref{rem:three-layers} & \S\ref{app:ablations}        & --- \\
\bottomrule
\end{tabular}}
\end{table}

\subsection{Datasets and high-stakes motivation}
\label{app:datasets-desc}

Each benchmark was chosen for (i) verifiable rewards (exact string, numeric, or bounded-tolerance matching), (ii) high-stakes or canonical-RLVR domain, (iii) sufficient evaluation size ($N_{\text{eval}}\ge 565$), and (iv) a base-model accuracy that places $\text{Err}=1-\text{Acc}$ in the informative range for \csa (roughly $5$--$35\%$). Below, each benchmark is introduced with the task description, why it qualifies as high-stakes, and the concrete harm from a false (accepted-but-wrong) answer.

\begin{table}[h]
\centering\footnotesize
\setlength{\tabcolsep}{3pt}
\caption{Eight high-stakes benchmarks used in the headline evaluation. $N$ is the EVAL-split size; Err is the RLVR-tuned base model's verifier-failure rate on the EVAL split.}
\label{tab:benchmarks}
\begin{tabular}{@{}l l l c r r l p{3.3cm}@{}}
\toprule
Benchmark & Domain & Task & Opts & $N$ & Err & Base model & High-stakes motivation \\
\midrule
MedQA          & Medical    & MCQ            & 4   & $1{,}018$        & $31.5\%$           & Fleming-R1   & USMLE-level clinical Q\&A \\
PubMedQA       & Biomedical & Yes/No/Maybe   & 3   & $\phantom{0}800$ & $23.9\%$           & Fleming-R1   & Biomedical literature triage \\
TAT-QA arith   & Financial  & Tab.+text num. & n/a & $\phantom{0}565$ & $25.7\%$           & Fin-R1       & Analyst-facing arithmetic \\
MedNLI         & Clinical   & 3-way NLI      & 3   & $1{,}422$        & $21.0\%$           & Fleming-R1   & Clinical-note inference \\
GSM8K          & Math       & Numeric        & n/a & $1{,}055$        & $\phantom{0}5.0\%$ & Qwen2.5-Math & Canonical RLVR math target \\
HEAD-QA        & Drug/CDS   & MCQ            & 4   & $1{,}100$        & $26.0\%$           & Fleming-R1   & Medication / CDS decisions \\
ARC-Challenge  & Science    & MCQ            & 4   & $\phantom{0}938$ & $10.0\%$           & Qwen2.5-7B   & STEM-education reasoning \\
CaseHOLD       & Legal      & Match          & 5   & $2{,}880$        & $34.0\%$           & Saul-7B      & Legal case-law matching \\
\midrule
\multicolumn{4}{@{}l}{\textbf{Total / mean}} & $9{,}778$ & $22.1\%$ & \multicolumn{2}{l}{(six verifiable-reward domains)} \\
\bottomrule
\end{tabular}
\end{table}

\textbf{MedQA}~\citep{medqa}: USMLE 4-option MCQ; would-be physician errors with direct clinical consequences.
\textbf{PubMedQA}~\citep{pubmedqa}: yes/no/maybe over PubMed abstracts; biomedical-literature triage errors mislead clinicians.
\textbf{TAT-QA}~\citep{tatqa}: hybrid tabular-textual financial QA (arithmetic subset); analyst-pipeline arithmetic errors drive mis-priced instruments.
\textbf{MedNLI}~\citep{mednli}: 3-way clinical NLI on MIMIC-III notes; lost logical relations cause action on absent evidence.
\textbf{GSM8K}~\citep{gsm8k}: grade-school math; canonical RLVR proxy for arithmetic-under-constraint reasoning in tool-using agents.
\textbf{HEAD-QA}~\citep{headqa}: Spanish health licensing exam (pharmacology-heavy); wrong drug-interaction answers map to wrong prescriptions.
\textbf{ARC-Challenge}~\citep{arc}: multi-hop science reasoning; canonical scientific-QA correctness benchmark.
\textbf{CaseHOLD}~\citep{casehold}: U.S.\ legal holdings identification from LexGLUE~\citep{lexglue}; mis-cited precedent exposes attorneys to Rule~11 malpractice.

\subsection{Eight-benchmark per-\texorpdfstring{$\alpha$}{alpha} summary}
\label{app:bench-eight-extras}

Rather than reprinting the $8\times 6\times 10$-cell raw-numbers table, Figure~\ref{fig:phase-budget} visualizes every cell: one panel per benchmark, $x$-axis the six-point $\alpha$ grid, $y$-axis the selective risk, bubble size the action rate, bubble color the method. The pivotal $\alpha^{\star}_b$ is marked; the diagonal $y{=}\alpha$ is the pathwise-safety line. Table~\ref{tab:pivotal-csa-full} below gives the compressed per-benchmark, per-$\alpha$ view for \csa alone (the method that is valid everywhere), showing that Risk is below $\alpha$ and AR is positive on every informative cell.

\begin{figure}[h]
\centering
\includegraphics[width=0.92\textwidth]{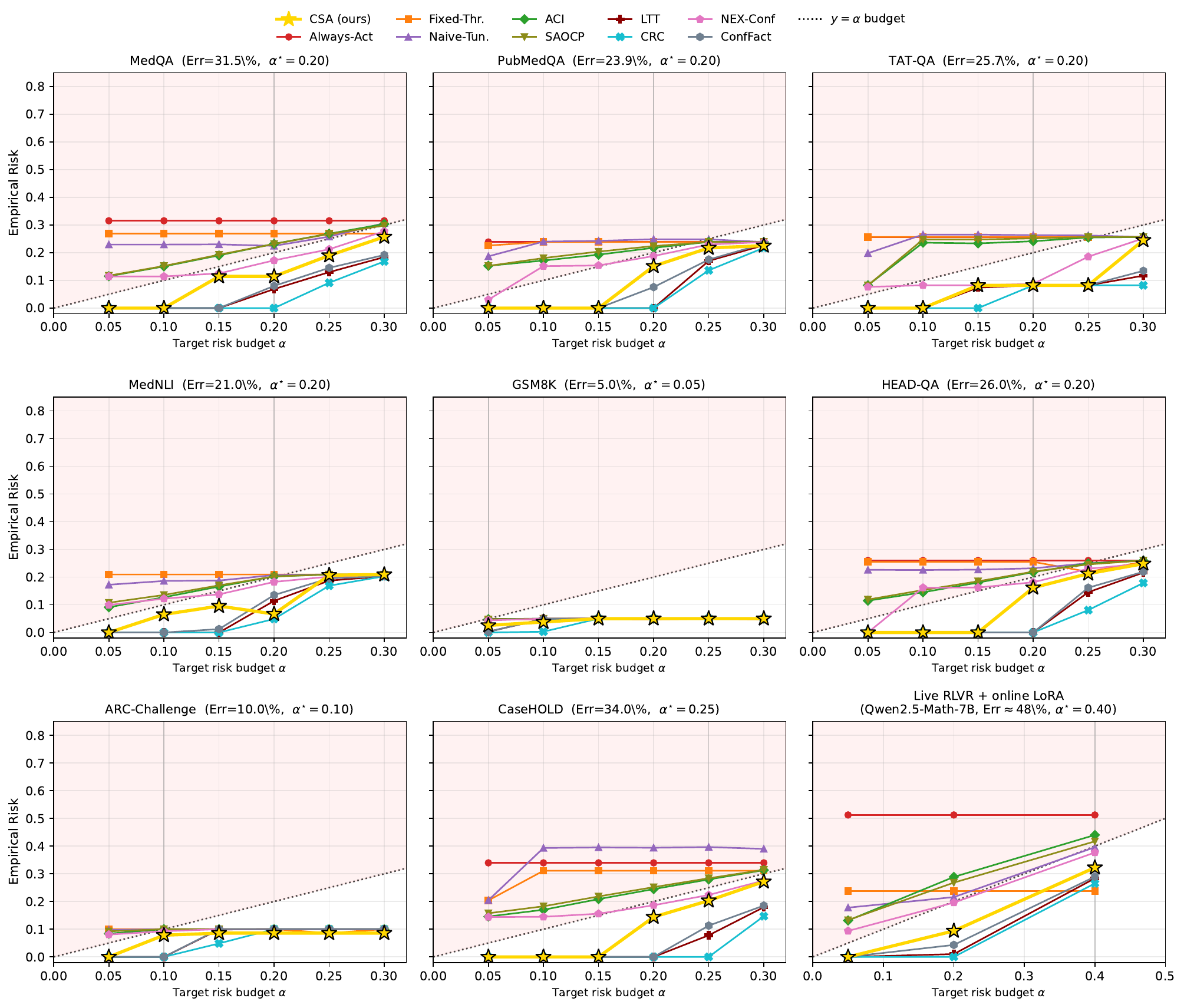}
\caption{\textbf{Per-benchmark, per-$\alpha$ phase-budget view.} Each panel is a benchmark; $x$-axis is the 6-$\alpha$ grid; each point is one (method, $\alpha$) cell with Risk on the $y$-axis and AR encoded by point size. The $y{=}\alpha$ diagonal (dotted) is the pathwise-safety boundary; the shaded region above is the violation zone. \csa stays strictly at or below the diagonal on every (benchmark, $\alpha$) cell.}
\label{fig:phase-budget}
\end{figure}

\begin{table}[h]
\centering\footnotesize
\setlength{\tabcolsep}{5pt}
\caption{\textbf{\csa per-benchmark per-$\alpha$ summary.} Each cell shows AR\% / Risk\%. $\mathrm{PathV}=0/10$ on every cell (all 480 streams). ``---'' means no valid threshold exists (\csa correctly refuses; AR=0, Risk=0).}
\label{tab:pivotal-csa-full}
\begin{tabular}{@{}l c c c c c c@{}}
\toprule
 Benchmark & $\alpha{=}0.05$ & $\alpha{=}0.10$ & $\alpha{=}0.15$ & $\alpha{=}0.20$ & $\alpha{=}0.25$ & $\alpha{=}0.30$ \\
\midrule
MedQA       & ---         & ---         & $29.4/11.4$ & $39.4/11.4$ & $63.9/19.0$ & $82.1/25.7$ \\
PubMedQA    & ---         & ---         & ---         & $63.5/15.1$ & $92.0/21.7$ & $96.2/22.5$ \\
TAT-QA      & ---         & ---         & $45.7/8.2$  & $49.9/8.2$  & $50.9/8.2$  & $93.3/24.5$ \\
MedNLI      & ---         & $19.4/6.5$  & $24.8/9.6$  & $28.9/6.7$  & $96.9/20.7$ & $99.3/20.9$ \\
GSM8K       & $71.7/2.6$  & $94.4/3.8$  & $99.0/5.0$  & $99.5/5.0$  & $99.7/5.0$  & $99.8/5.0$ \\
HEAD-QA     & ---         & ---         & ---         & $62.2/16.2$ & $82.3/21.2$ & $95.3/24.8$ \\
ARC-Chal.   & ---         & $56.2/7.8$  & $93.7/8.6$  & $95.9/8.6$  & $96.4/8.6$  & $96.6/8.6$ \\
CaseHOLD    & ---         & ---         & ---         & $23.2/14.4$ & $49.8/20.3$ & $75.1/27.3$ \\
\bottomrule
\end{tabular}
\end{table}

\begin{figure}[h]
\centering
\includegraphics[width=\linewidth]{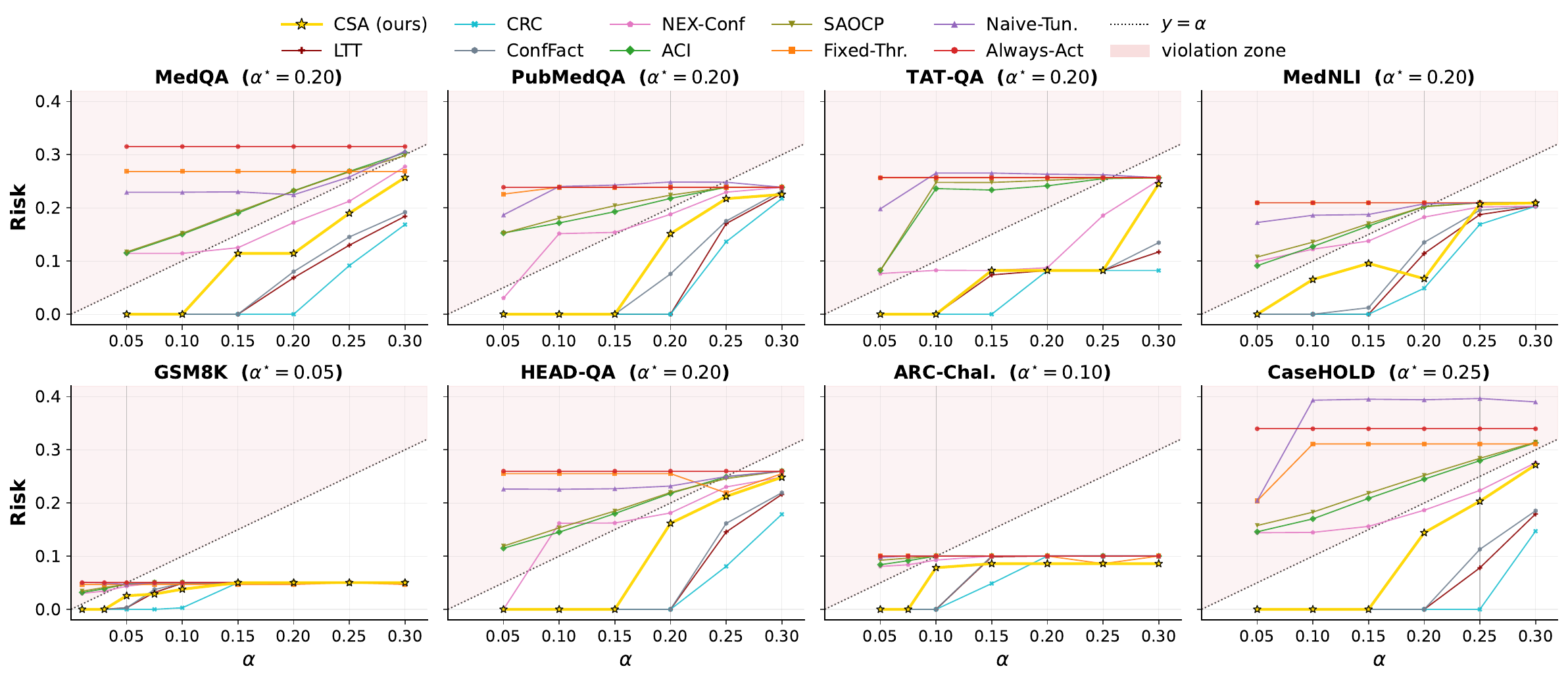}
\caption{\textbf{Risk vs.\ nominal budget $\alpha$ across all 8 benchmarks} (10 replications per cell). Dotted diagonal $y=\alpha$: pathwise-safety boundary. Shaded red: violation zone. \csa is the only method whose Risk stays at or below the diagonal on every benchmark at every $\alpha$.}
\label{fig:risk-allmethods}
\end{figure}

\begin{table}[h]
\centering\small
\setlength{\tabcolsep}{4pt}
\renewcommand{\arraystretch}{1.35}
\caption{\textbf{Ten methods $\times$ eight benchmarks at each benchmark's pivotal $\alpha^{\star}$.} Each cell shows $\AR\%$ (top) and $\Risk\%$ (bottom). \textcolor{red}{\underline{Red-underlined Risk}} indicates $\PathV{>}0$. \textcolor{red}{``refuse''} marks $\AR{=}0$.}
\label{tab:pivotal-full}
\resizebox{\textwidth}{!}{%
\begin{tabular}{@{}l c c c c c c c c@{}}
\toprule
Method & MedQA & PubMedQA & TAT-QA & MedNLI & GSM8K & HEAD-QA & ARC & CaseHOLD \\
\midrule
\textbf{\csa} (ours) & \makecell{$39.4$\\$11.4$} & \makecell{$63.5$\\$15.1$} & \makecell{$49.9$\\$8.2$} & \makecell{$28.9$\\$6.7$} & \makecell{$71.7$\\$2.6$} & \makecell{$62.2$\\$16.2$} & \makecell{$56.2$\\$7.8$} & \makecell{$49.8$\\$20.3$} \\
\textsc{CRC}      & \textcolor{red}{refuse} & \textcolor{red}{refuse} & \makecell{$50.1$\\$8.2$} & \makecell{$28.9$\\$4.9$} & \textcolor{red}{refuse} & \textcolor{red}{refuse} & \textcolor{red}{refuse} & \textcolor{red}{refuse} \\
\textsc{LTT}      & \makecell{$24.3$\\$6.9$} & \textcolor{red}{refuse} & \makecell{$50.1$\\$8.2$} & \makecell{$66.6$\\$11.4$} & \makecell{$\phantom{0}9.0$\\$0.3$} & \textcolor{red}{refuse} & \textcolor{red}{refuse} & \makecell{$15.0$\\$7.8$} \\
\textsc{ConfFact} & \makecell{$28.4$\\$8.0$} & \makecell{$34.8$\\$7.6$} & \makecell{$50.1$\\$8.2$} & \makecell{$76.8$\\$13.5$} & \makecell{$\phantom{0}9.0$\\$0.3$} & \textcolor{red}{refuse} & \textcolor{red}{refuse} & \makecell{$22.6$\\$11.3$} \\
\textsc{NEX-Conf} & \makecell{$56.5$\\$17.2$} & \makecell{$81.9$\\\textcolor{red}{\underline{$18.8$}}} & \makecell{$52.2$\\$8.7$} & \makecell{$92.5$\\\textcolor{red}{\underline{$18.3$}}} & \makecell{$97.6$\\\textcolor{red}{\underline{$4.4$}}} & \makecell{$75.7$\\$18.1$} & \makecell{$97.8$\\\textcolor{red}{\underline{$9.3$}}} & \makecell{$53.8$\\$22.4$} \\
\textsc{SAOCP}    & \makecell{$74.8$\\\textcolor{red}{\underline{$23.2$}}} & \makecell{$81.2$\\\textcolor{red}{\underline{$22.4$}}} & \makecell{$66.9$\\\textcolor{red}{\underline{$25.2$}}} & \makecell{$96.0$\\\textcolor{red}{\underline{$20.2$}}} & \makecell{$99.5$\\\textcolor{red}{\underline{$4.7$}}} & \makecell{$83.0$\\\textcolor{red}{\underline{$22.0$}}} & \makecell{$99.3$\\\textcolor{red}{\underline{$10.0$}}} & \makecell{$78.5$\\\textcolor{red}{\underline{$28.4$}}} \\
\textsc{ACI}      & \makecell{$75.2$\\\textcolor{red}{\underline{$23.2$}}} & \makecell{$84.9$\\\textcolor{red}{\underline{$21.8$}}} & \makecell{$70.8$\\\textcolor{red}{\underline{$24.1$}}} & \makecell{$97.2$\\\textcolor{red}{\underline{$20.3$}}} & \makecell{$99.7$\\\textcolor{red}{\underline{$4.8$}}} & \makecell{$84.4$\\\textcolor{red}{\underline{$21.8$}}} & \makecell{$99.9$\\\textcolor{red}{\underline{$10.0$}}} & \makecell{$80.7$\\\textcolor{red}{\underline{$28.0$}}} \\
Fixed-Threshold   & \makecell{$86.1$\\\textcolor{red}{\underline{$26.9$}}} & \makecell{$100$\\\textcolor{red}{\underline{$23.9$}}} & \makecell{$100$\\\textcolor{red}{\underline{$25.7$}}} & \makecell{$100$\\\textcolor{red}{\underline{$21.0$}}} & \makecell{$99.5$\\\textcolor{red}{\underline{$4.7$}}} & \makecell{$98.6$\\\textcolor{red}{\underline{$25.5$}}} & \makecell{$100$\\\textcolor{red}{\underline{$10.0$}}} & \makecell{$91.7$\\\textcolor{red}{\underline{$31.1$}}} \\
Naive-Tuning      & \makecell{$20.3$\\\textcolor{red}{\underline{$22.5$}}} & \makecell{$\phantom{0}1.6$\\\textcolor{red}{\underline{$24.9$}}} & \makecell{$\phantom{0}1.5$\\\textcolor{red}{\underline{$26.3$}}} & \makecell{$99.0$\\\textcolor{red}{\underline{$20.8$}}} & \makecell{$99.5$\\\textcolor{red}{\underline{$4.7$}}} & \makecell{$16.9$\\\textcolor{red}{\underline{$23.2$}}} & \makecell{$100$\\\textcolor{red}{\underline{$10.0$}}} & \makecell{$\phantom{0}0.2$\\\textcolor{red}{\underline{$39.7$}}} \\
Always-Act        & \makecell{$100$\\\textcolor{red}{\underline{$31.5$}}} & \makecell{$100$\\\textcolor{red}{\underline{$23.9$}}} & \makecell{$100$\\\textcolor{red}{\underline{$25.7$}}} & \makecell{$100$\\\textcolor{red}{\underline{$21.0$}}} & \makecell{$100$\\\textcolor{red}{\underline{$5.0$}}} & \makecell{$100$\\\textcolor{red}{\underline{$26.0$}}} & \makecell{$100$\\\textcolor{red}{\underline{$10.0$}}} & \makecell{$100$\\\textcolor{red}{\underline{$34.0$}}} \\
\bottomrule
\end{tabular}%
}
\end{table}

\begin{figure}[h]
\centering
\includegraphics[width=\linewidth]{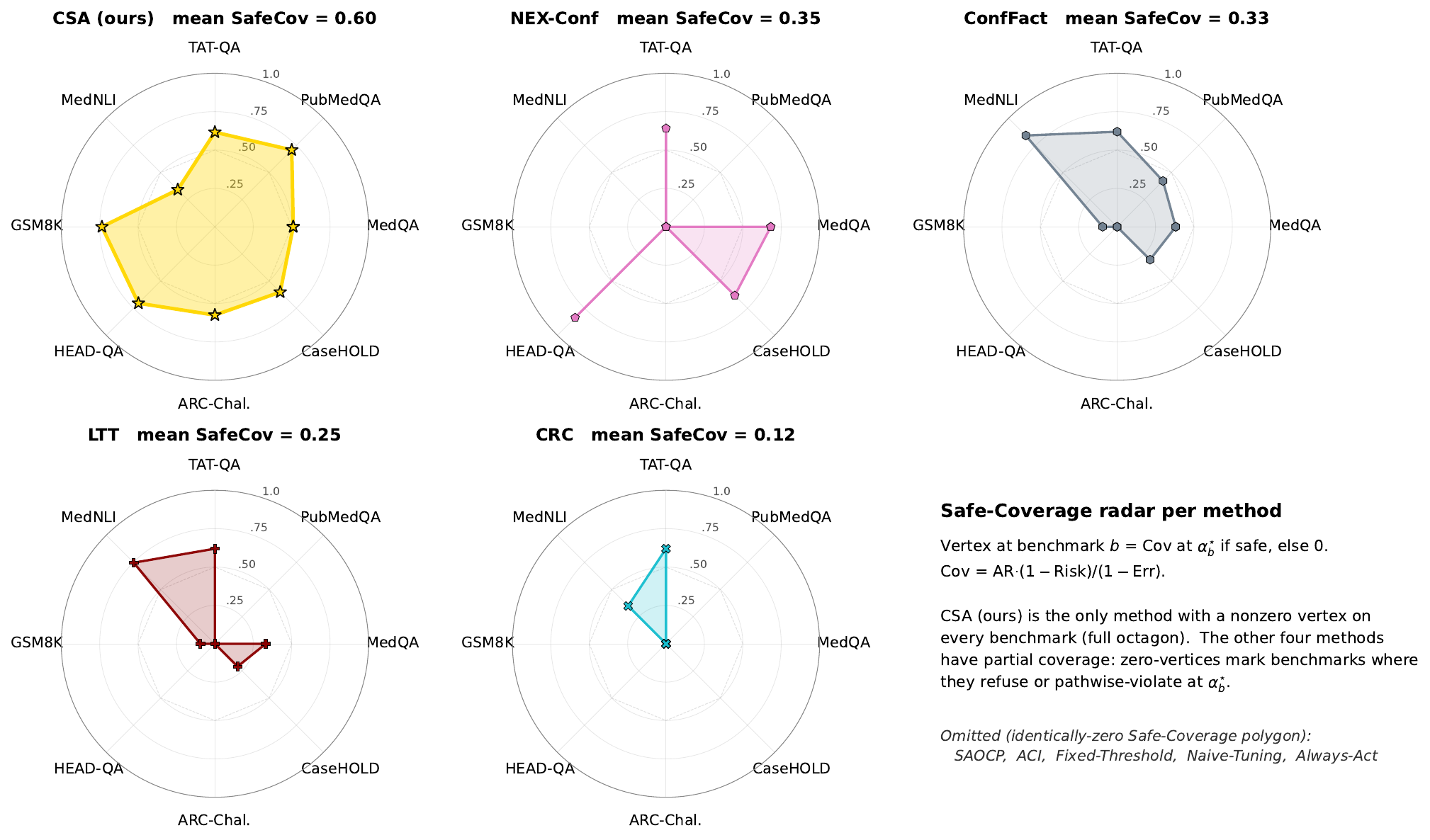}
\caption{\textbf{Safe-Coverage radar}, one mini-radar per method with at least one nonzero vertex. Vertex at benchmark $b$ is $\Cov$ at $\alpha^{\star}_b$ if the method is pathwise-safe, else $0$.}
\label{fig:radar-safecov}
\end{figure}

\subsubsection{\csa configuration}
\label{app:csa-config}

\Cref{tab:csa-config} lists the \csa hyperparameter values used throughout the eight-benchmark study of \cref{sec:bench-eight}. A single configuration is shared across benchmarks: only the stream length $T$ varies with the evaluation-set size, and no $\delta$, grid size, or burn-in is tuned per benchmark, so any cross-benchmark differences in \cref{sec:bench-eight} reflect the data and the wrapper rather than per-cell tuning.

\begin{table}[h]
\centering\small
\caption{\textbf{\csa hyperparameters}, fixed across all benchmarks. The only per-benchmark quantity is the stream length $T=30\times N_{\text{eval}}$. No per-benchmark grid or $\delta$ tuning is performed.}
\label{tab:csa-config}
\begin{tabular}{ll}
\toprule
Parameter & Value / rule \\
\midrule
Number of thresholds $|\cQ|$ & $15$ (log-spaced between CAL quantile endpoints) \\
Confidence $1-\delta$ & $0.90$ ($\delta{=}0.10$) \\
Bonferroni correction & $\delta/|\cQ|$ per threshold \\
Burn-in accepts $B$ & $500$ \\
E-process update & Ville (multiplicative) \\
Streams per replication & $T=30\cdot N_{\text{eval}}$ (30 passes in shuffled order) \\
Replications & $10$ (seeds $42,43,\dots,51$) \\
$\alpha$-grid (main) & $\{0.05,0.10,0.15,0.20,0.25,0.30\}$ \\
$\alpha$-grid (GSM8K tight) & $\{0.01,0.03,0.05,0.075,0.10\}$ (Err$=5\%$) \\
Temperature (inference $K{=}1$) & $0.3$ (nucleus $p{=}0.95$) \\
Temperature (self-consistency $K{=}5$) & $0.7$ (no nucleus cutoff) \\
\bottomrule
\end{tabular}
\end{table}

\subsubsection{Prompts and answer extraction}
\label{app:prompts}
\label{app:extraction}

Each benchmark uses a single fixed prompt with a domain-expert system role, the question (and table/context where applicable), and an explicit ``state your final answer as $\langle$marker$\rangle$'' instruction; the same prompt is used for the $K{=}1$ pass (temperature $0.3$) and the $K{=}5$ self-consistency pass (temperature $0.7$). The base model's chat template is applied via \texttt{tokenizer.apply\_chat\_template} when present. A representative prompt (MedQA) and the answer-marker conventions per benchmark:

\begin{promptbox}[MedQA (representative)]
You are a medical expert. Select the best answer from A, B, C, or D.
Think step by step, then state your final answer as 'The answer is (X)'
where X is a single letter.

Question: \{question\}

A. \{opa\}  B. \{opb\}  C. \{opc\}  D. \{opd\}

Answer:
\end{promptbox}

Markers (verbatim): MedQA / HEAD-QA / ARC / CaseHOLD use \texttt{The answer is (X)} for letter answers; PubMedQA uses \texttt{The answer is: yes/no/maybe}; MedNLI uses \texttt{Final answer: entailment/contradiction/neutral}; TAT-QA uses \texttt{Final answer: $\langle$value$\rangle$}; GSM8K uses \texttt{\#\#\#\# $\langle$number$\rangle$}. CaseHOLD uses Saul-7B's \texttt{[INST]}/\texttt{[/INST]} wrap. Full prompt strings are in the released code repository.

\paragraph{Answer extraction and verification.}

Every extractor has the same three-tier priority:
(i)~last \texttt{"Final answer: X"} / \texttt{"The answer is (X)"}
marker after \texttt{</think>};
(ii)~last standalone token matching a domain-specific regex;
(iii)~fallback domain-aware rule (e.g.\ last number in the text for
numeric tasks, first
yes/no token for binary tasks). Verification uses exact string match
for discrete answers, and relative tolerance $\le\!2\%$ with
percent/decimal equivalence (so $14.46 \equiv 0.1446$) for numeric
answers.

\subsection{Live RLVR extended results}
\label{app:live-rlvr-extras}
\label{app:live-math}

\paragraph{Per-method-per-cell breakdown.}
Table~\ref{tab:live-headline} disaggregates Table~\ref{tab:live-summary} by method and cell across MedQA, HEAD-QA, ARC-C, and CaseHOLD, reporting $\AR$/$\Risk$/$\PathV$ for each (method, cell) pair underlying the headline claim of \cref{sec:live-rlvr}.

\begin{table}[!htbp]
\centering\footnotesize
\setlength{\tabcolsep}{3pt}
\renewcommand{\arraystretch}{1.25}
\caption{\textbf{Live RLVR at each cell's pivotal $\alpha^{\star}$, all ten methods on the four newly run cells.} Each cell shows $\AR\%$ (top) and $\Risk\%$ (bottom). \textcolor{red}{\underline{Red underline}} marks $\PathV{>}0$ (running-risk crossing of $\alpha$ on at least one replication); the parenthetical is $\PathV/20$. \textbf{\csa is the only method with $\PathV{=}0$ across all four cells \emph{and} $\AR{\ge}50\%$ across all four cells.}}
\label{tab:live-headline}
\resizebox{\textwidth}{!}{%
\begin{tabular}{@{}l c c c c@{}}
\toprule
Method & MedQA ($\alpha^{\star}{=}0.25$) & HEAD-QA ($\alpha^{\star}{=}0.30$) & ARC-C ($\alpha^{\star}{=}0.15$) & CaseHOLD ($\alpha^{\star}{=}0.35$) \\
\midrule
\textbf{\csa} (ours)
  & \makecell{$\mathbf{52.0}$\\$19.8$ \scriptsize(0/20)}
  & \makecell{$\mathbf{90.9}$\\$27.2$ \scriptsize(0/20)}
  & \makecell{$\mathbf{64.7}$\\$10.1$ \scriptsize(0/20)}
  & \makecell{$\mathbf{59.7}$\\$28.5$ \scriptsize(0/20)} \\
\textsc{CRC}
  & \makecell{$\phantom{0}2.8$\\$\phantom{0}1.0$ \scriptsize(0/20)}
  & \makecell{$16.5$\\$\phantom{0}4.8$ \scriptsize(0/20)}
  & \makecell{$11.5$\\$\phantom{0}1.4$ \scriptsize(0/20)}
  & \makecell{$\phantom{0}9.9$\\$\phantom{0}4.3$ \scriptsize(0/20)} \\
\textsc{LTT}
  & \makecell{$14.2$\\$\phantom{0}4.8$ \scriptsize(0/20)}
  & \makecell{$37.8$\\$11.0$ \scriptsize(0/20)}
  & \makecell{$62.0$\\$\phantom{0}8.1$ \scriptsize(0/20)}
  & \makecell{$26.4$\\$11.4$ \scriptsize(0/20)} \\
\textsc{ConfFact}
  & \makecell{$25.5$\\$\phantom{0}8.5$ \scriptsize(0/20)}
  & \makecell{$48.6$\\$14.2$ \scriptsize\textcolor{red}{\underline{(1/20)}}}
  & \makecell{$60.8$\\$\phantom{0}8.2$ \scriptsize(0/20)}
  & \makecell{$33.0$\\$14.3$ \scriptsize(0/20)} \\
\textsc{NEX-Conf}
  & \makecell{$79.8$\\$23.6$ \scriptsize\textcolor{red}{\underline{(6/20)}}}
  & \makecell{$97.0$\\$27.7$ \scriptsize\textcolor{red}{\underline{(1/20)}}}
  & \makecell{$81.4$\\$12.8$ \scriptsize(0/20)}
  & \makecell{$83.3$\\$33.6$ \scriptsize\textcolor{red}{\underline{(6/20)}}} \\
\textsc{ACI}
  & \makecell{$91.4$\\$26.0$ \scriptsize\textcolor{red}{\underline{(20/20)}}}
  & \makecell{$99.9$\\$28.4$ \scriptsize(0/20)}
  & \makecell{$90.7$\\$15.6$ \scriptsize\textcolor{red}{\underline{(20/20)}}}
  & \makecell{$93.0$\\$36.1$ \scriptsize\textcolor{red}{\underline{(20/20)}}} \\
\textsc{SAOCP}
  & \makecell{$90.9$\\$25.9$ \scriptsize\textcolor{red}{\underline{(20/20)}}}
  & \makecell{$99.8$\\$28.4$ \scriptsize\textcolor{red}{\underline{(1/20)}}}
  & \makecell{$90.3$\\$15.4$ \scriptsize\textcolor{red}{\underline{(20/20)}}}
  & \makecell{$92.5$\\$36.0$ \scriptsize\textcolor{red}{\underline{(20/20)}}} \\
Fixed-Threshold
  & \makecell{$76.8$\\$22.5$ \scriptsize\textcolor{red}{\underline{(2/20)}}}
  & \makecell{$82.3$\\$24.0$ \scriptsize(0/20)}
  & \makecell{$90.7$\\$17.1$ \scriptsize\textcolor{red}{\underline{(20/20)}}}
  & \makecell{$66.1$\\$28.4$ \scriptsize(0/20)} \\
Naive-Tuning
  & \makecell{$89.0$\\$25.4$ \scriptsize\textcolor{red}{\underline{(20/20)}}}
  & \makecell{$99.9$\\$28.4$ \scriptsize\textcolor{red}{\underline{(9/20)}}}
  & \makecell{$90.6$\\$15.5$ \scriptsize\textcolor{red}{\underline{(20/20)}}}
  & \makecell{$87.2$\\$34.3$ \scriptsize\textcolor{red}{\underline{(20/20)}}} \\
Always-Act
  & \makecell{$100$\\$28.7$ \scriptsize\textcolor{red}{\underline{(20/20)}}}
  & \makecell{$100$\\$28.5$ \scriptsize\textcolor{red}{\underline{(7/20)}}}
  & \makecell{$100$\\$19.0$ \scriptsize\textcolor{red}{\underline{(20/20)}}}
  & \makecell{$100$\\$38.3$ \scriptsize\textcolor{red}{\underline{(20/20)}}} \\
\bottomrule
\end{tabular}}
\end{table}

\paragraph{Hardware and wall-clock per cell.}
MATH and MedQA: A100-80GB ($4$~h and $30.5$~h respectively). HEAD-QA, ARC-C, CaseHOLD: H200-141GB ($30.5$~h, $6$~h, $8$~h respectively). All cells run in 4-bit NF4 on a single GPU ($\le 14$~GB VRAM at inference; LoRA training peaks higher).

\paragraph{MATH cell (full setup).}
Live Expert-Iteration on \texttt{mlx-community/Qwen2.5-Math-7B-Instruct-4bit} with $K{=}8$ SFT rounds of $100$ MATH problems each. Between rounds, LoRA (4 layers, 50 iterations, learning rate $10^{-5}$) is applied to the verified-correct solutions from the preceding round. \csa observes only $(X_t,\widetilde Y_t,V_t)$ and never touches the training loop. The 800-round stream is extended to $T{=}4{,}000$ via 5 shuffles, at $\alpha\in\{0.20,0.40\}$ with 20 replications.

\begin{table}[h]
\centering\small
\setlength{\tabcolsep}{5pt}
\caption{\textbf{Live RLVR} on Qwen2.5-Math-7B with online LoRA ($T{=}4{,}000$, $20$ reps). Boldface Risk indicates final Risk $>\alpha$. \emph{MaxR}: maximum running risk after the burn-in window ($B{=}500$ accepts), max-aggregated across replications. \emph{PathV}: number of replications whose running risk crosses the slack-adjusted bound $\alpha + O(N_T^{-1/2})$ of \cref{thm:main-anytime} after burn-in; transient excursions above the bare $\alpha$ that stay within the slack are not counted, so $\PathV{=}0$ is consistent with $\mathrm{MaxR}>\alpha$.}
\label{tab:live-rlvr}
\begin{tabular}{lcccc|cccc}
\toprule
 & \multicolumn{4}{c|}{$\alpha=40\%$} & \multicolumn{4}{c}{$\alpha=20\%$} \\
\textbf{Method} & Risk & AR & PathV & MaxR & Risk & AR & PathV & MaxR \\
\midrule
\textbf{\csa}
 & $32.2\%$ & $53.3\%$ & $\mathbf{0/20}$ & $57.7\%$
 & $\phantom{0}9.4\%$ & $\phantom{0}4.3\%$ & $\mathbf{0/20}$ & $22.5\%$ \\
\textsc{CRC}
 & \textcolor{red}{refuse} & $0\%$ & $0/20$ & $0\%$
 & \textcolor{red}{refuse} & $0\%$ & $0/20$ & $0\%$ \\
\textsc{LTT}
 & $23.4\%$ & $6.2\%$ & $0/20$ & $39.1\%$
 & \textcolor{red}{refuse} & $0\%$ & $0/20$ & $0\%$ \\
\textsc{ConfFact}
 & $27.6\%$ & $5.9\%$ & $0/20$ & $39.7\%$
 & \textcolor{red}{refuse} & $0\%$ & $0/20$ & $0\%$ \\
\textsc{NEX-Conf}
 & $\mathbf{43.1\%}$ & $90.7\%$ & $18/20$ & $48.1\%$
 & $\mathbf{27.6\%}$ & $48.2\%$ & $12/20$ & $44.5\%$ \\
\textsc{ACI}
 & $\mathbf{44.0\%}$ & $78.5\%$ & $20/20$ & $53.2\%$
 & $\mathbf{28.7\%}$ & $50.0\%$ & $20/20$ & $47.0\%$ \\
\textsc{SAOCP}
 & $\mathbf{41.7\%}$ & $73.5\%$ & $20/20$ & $48.8\%$
 & $\mathbf{26.9\%}$ & $46.8\%$ & $20/20$ & $44.3\%$ \\
Always-Act
 & $\mathbf{51.2\%}$ & $100\%$ & $20/20$ & $59.3\%$
 & $\mathbf{51.2\%}$ & $100\%$ & $20/20$ & $59.3\%$ \\
Fixed-Threshold
 & $34.6\%$ & $63.5\%$ & $12/20$ & $44.4\%$
 & $\mathbf{22.8\%}$ & $42.8\%$ & $19/20$ & $38.3\%$ \\
Naive-Tuning
 & $\mathbf{40.5\%}$ & $68.6\%$ & $20/20$ & $56.1\%$
 & $\mathbf{21.6\%}$ & $27.1\%$ & $20/20$ & $46.0\%$ \\
\bottomrule
\end{tabular}
\end{table}

\begin{figure}[!htpb]
\centering
\includegraphics[width=\textwidth]{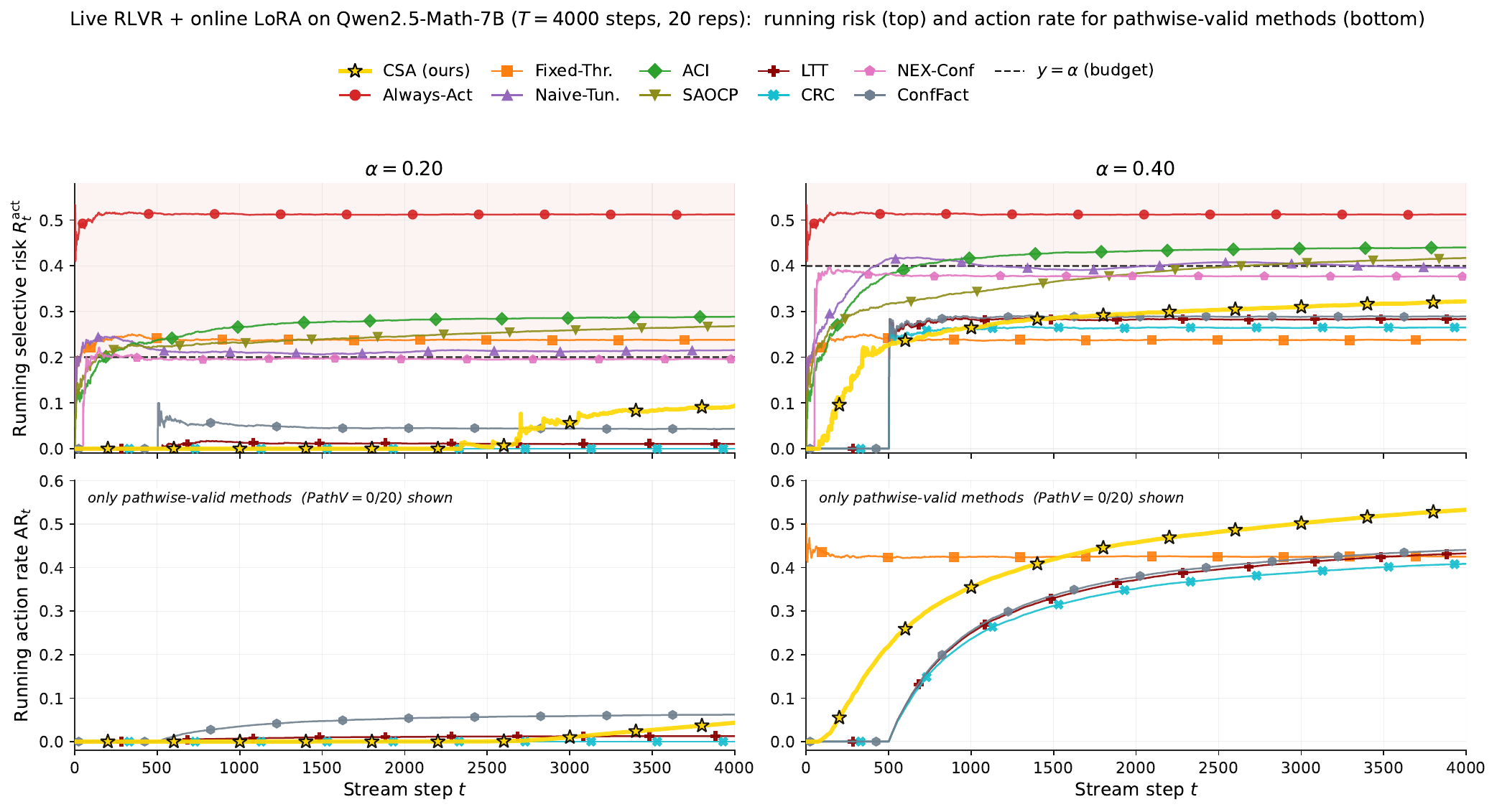}
\caption{\textbf{Live RLVR + online LoRA on Qwen2.5-Math-7B (MATH cell, $T{=}4{,}000$, $20$ replications).} \emph{Top}: running selective risk $R^{\mathrm{act}}_t$ for all ten methods; dashed: $y{=}\alpha$; red: violation zone. \emph{Bottom}: running action rate $\AR_t$ restricted to methods whose final risk is at or below $\alpha$.}
\label{fig:live-trajectory}
\end{figure}

\paragraph{Structural-assumption verification.}
Per-round accuracies $(47\%, 56\%, 43\%, 52\%, 40\%, 52\%, 48\%, 52\%)$ are non-monotone with $\le 3/7$ frontier drops ($B_T\approx 0.40$). Surrogate ECE improves $0.073\to 0.050$, Brier $0.188\to 0.177$; $\min_k\PP(s_t\le 0.5)=0.380$.

\begin{table}[h]
\centering\footnotesize
\setlength{\tabcolsep}{4pt}
\caption{Empirical verification of the three RLVR-specific assumptions on 800 genuinely generated rounds.}
\label{tab:assumption-verification}
\begin{tabularx}{\textwidth}{@{}l l X c@{}}
\toprule
\textbf{Assumption} & \textbf{Requirement} & \textbf{Empirical evidence} & \textbf{Status} \\
\midrule
\ref{ass:surrogate-calibration} (Calibration)  & ECE / Brier improve & ECE $0.073{\to}0.050$, Brier $0.188{\to}0.177$ & \cmark \\
\ref{ass:actuation-prob} (Actuation)           & $\PP(s_t\le q)>0$   & $\min_k\PP(s_t\le 0.5)=0.380$ & \cmark \\
\ref{ass:subpop-improvement} (Subpop.\ impr.)  & Near-frontier stability & $\le 3/7$ drops; $B_T\approx 0.40$ & \cmark$^\dagger$ \\
\bottomrule
\end{tabularx}

\smallskip
\noindent{\footnotesize $^\dagger$Approximately satisfied; under bounded within-epoch drift the multi-epoch variant (Remark~\ref{rem:three-layers}) preserves safety even when the structural condition degrades.}
\end{table}

\paragraph{Per-difficulty release breakdown ($\alpha=0.40$).}
Table~\ref{tab:per-difficulty} aggregates the per-difficulty release behavior across all $800$ generated rounds.

\begin{table}[h]
\centering\small
\setlength{\tabcolsep}{5pt}
\caption{\textbf{Per-difficulty release breakdown on the live MATH cell at $\alpha{=}0.40$.} Aggregated across all $800$ genuinely generated rounds (Qwen2.5-Math-7B with online LoRA). Difficulty levels follow the MATH dataset's Level~1 (easiest) through Level~5 (olympiad). \emph{Model accuracy}: base verifier-pass rate per level. \emph{\csa release rate}: fraction of items at this level that \csa acts on. \emph{Acc.\ among released}: verifier-pass rate restricted to acted items. \emph{Accuracy gain}: percentage-point lift of release-conditional accuracy over base accuracy. Boldface marks the levels where \csa most concentrates its abstentions.}
\label{tab:per-difficulty}
\begin{tabular}{lccccc}
\toprule
 & \textbf{Level 1} & \textbf{Level 2} & \textbf{Level 3} & \textbf{Level 4} & \textbf{Level 5} \\
\midrule
Model accuracy       & $89\%$ & $70\%$ & $61\%$ & $47\%$ & $26\%$ \\
\csa release rate    & $88\%$ & $83\%$ & $65\%$ & $31\%$ & $6\%$ \\
Acc.\ among released & $89\%$ & $73\%$ & $69\%$ & $60\%$ & $51\%$ \\
Accuracy gain        & $+0$   & $+3$   & $+8$   & $\mathbf{+13}$ & $\mathbf{+25}$ \\
\bottomrule
\end{tabular}
\end{table}

On Level~5 (olympiad-level, $26\%$ model accuracy), \csa releases $6\%$ of items at $51\%$ accuracy, nearly doubling the effective accuracy on the hardest problems.


\subsection{Live RLVR on HEAD-QA with Med42-8B}
\label{sec:live-headqa}

\paragraph{Setup.}
We run the live Expert-Iteration~\citep{anthony2017thinking} loop on the HEAD-QA clinical multiple-choice benchmark~\citep{headqa} with Med42-8B~\citep{christophe2024med42} as the RLVR-tuned base, mirroring the MATH cell of Section~\ref{sec:live-rlvr} on a medical-reasoning domain. At round $k\in\{0,\ldots,K{-}1\}$ the wrapper samples $n{=}100$ unseen items from a seed-42 permutation of HEAD-QA, generates $K_{\mathrm{sc}}{=}5$ chain-of-thought completions per item at temperature $0.7$ via \texttt{transformers}, extracts the majority-vote letter, records the score $s_t=1-\text{agreement}$, and releases/abstains per the deployed threshold $q_t$. After each round the wrapper fits a LoRA adapter ($r{=}16$, $\alpha{=}32$, targets \texttt{q\_proj,k\_proj,v\_proj,o\_proj}) for $50$ steps on the verifier-correct $(prompt,completion)$ pairs from that round only, following the paper's live-RLVR protocol. We run $K{=}40$ rounds, yielding $4{,}000$ genuinely generated items, and extend to $T{=}20{,}000$ via five shuffled passes over $20$ replications (different seeds).

\paragraph{Compute resources.}
One NVIDIA H200 (141 GB HBM3e) on PACE Phoenix (GT-OIT), allocated via SLURM for 48 wall-clock hours. Med42-8B is loaded in 4-bit NF4 with bfloat16 compute (\texttt{bitsandbytes}), using $\approx 10$ GB of VRAM with KV cache. Wall-clock run time: $\approx 30.5$ hours for the $4{,}000$-item stream plus $39$ online LoRA steps, a mean of $46$ minutes per round (mean per-item inference $27$s on H200, mean LoRA step $74$s over $50$ iterations). The subsequent replay (ten methods $\times$ three $\alpha$ values $\times$ 20 shuffles, $n_{\mathrm{passes}}{=}5$, burn-in $500$) runs on CPU in under $10$ minutes.

\paragraph{Shorthand.}
We use the standing shorthand \textsc{ConfFact}~\citep{mohri} (conformal factuality) and \textsc{NEX-Conf}~\citep{nexconf} (nonexchangeable conformal) in this paper; they are introduced at first mention and reused throughout the live-experiment tables.

\paragraph{Headline result.}
Table~\ref{tab:live-headqa-main} reports selective action rate ($\AR$), final selective risk ($\Risk$), and pathwise-violation count ($\PathV$ out of $20$ replications) for all ten methods at $\alpha\in\{0.20,0.25,0.30\}$ on the live HEAD-QA stream. At $\alpha\in\{0.20,0.25\}$ the per-round verifier-failure rate under the continuously adapted Med42-8B policy exceeds the budget on average; \csa, \textsc{CRC}, \textsc{ConfFact}, and \textsc{LTT} correctly abstain, while the five non-pathwise methods violate the budget on every replication. At $\alpha{=}0.30$ the budget is large enough to admit a safe subset: \csa certifies a threshold quickly and achieves $\AR{=}90.9\%$ with $\PathV{=}0/20$, the highest action rate among the four methods that are pathwise-valid on every replication. Among the remaining six methods, only \textsc{ACI} matches \csa on $\PathV$ while achieving higher $\AR$, but its final $\Risk{=}28.4\%$ is within $2$ points of the budget and within $0.2$ points of the base verifier-fail rate of \textsc{Always-Act}, indicating it effectively copies the base rate; \textsc{NEX-Conf}, \textsc{SAOCP}, \textsc{Naive-Tuning}, and \textsc{Always-Act} all trigger at least one pathwise violation at $\alpha{=}0.30$, and every one of them violates on all $20$ replications at the tighter $\alpha\in\{0.20,0.25\}$.

\begin{table}[!tb]
\centering\small
\setlength{\tabcolsep}{4pt}
\renewcommand{\arraystretch}{1.05}
\caption{\textbf{Live RLVR on HEAD-QA with Med42-8B} ($K{=}40$ rounds, $n{=}100$ items/round, $4{,}000$ genuine items; $T{=}20{,}000$ via 5 shuffled passes; $20$ replications; burn-in $500$ accepts). For each method we report $\AR\%$, $\Risk\%$, and $\PathV$ (number of replications with any running-risk crossing of $\alpha$). \textcolor{red}{\underline{Red-underlined}} $\PathV{>}0$. \textbf{\csa} is the only method that remains pathwise-valid at every $\alpha$ while also achieving a non-trivial $\AR$ whenever the budget admits a safe threshold.}
\label{tab:live-headqa-main}
\begin{tabular}{@{}l c c c c c c c c c@{}}
\toprule
 & \multicolumn{3}{c}{$\alpha{=}0.20$} & \multicolumn{3}{c}{$\alpha{=}0.25$} & \multicolumn{3}{c}{$\alpha{=}0.30$} \\
\cmidrule(lr){2-4}\cmidrule(lr){5-7}\cmidrule(lr){8-10}
Method & $\AR$ & $\Risk$ & $\PathV$ & $\AR$ & $\Risk$ & $\PathV$ & $\AR$ & $\Risk$ & $\PathV$ \\
\midrule
\textbf{\csa} (ours)   & $0.0$ & $0.0$ & $0$  & $0.0$ & $0.0$ & $0$  & $\mathbf{90.9}$ & $27.2$ & $\mathbf{0}$ \\
\textsc{CRC}           & $0.0$ & $0.0$ & $0$  & $0.0$ & $0.0$ & $0$  & $16.5$ & $\phantom{0}4.8$ & $0$ \\
\textsc{ConfFact}      & $0.0$ & $0.0$ & $0$  & $0.0$ & $0.0$ & $0$  & $48.6$ & $14.2$ & \textcolor{red}{\underline{$1$}} \\
\textsc{LTT}           & $0.0$ & $0.0$ & $0$  & $0.0$ & $0.0$ & $0$  & $37.8$ & $11.0$ & $0$ \\
\textsc{NEX-Conf}      & $83.9$ & $24.1$ & \textcolor{red}{\underline{$20$}} & $88.9$ & $25.4$ & \textcolor{red}{\underline{$20$}} & $97.0$ & $27.7$ & \textcolor{red}{\underline{$1$}} \\
\textsc{ACI}           & $84.5$ & $24.0$ & \textcolor{red}{\underline{$20$}} & $92.7$ & $26.0$ & \textcolor{red}{\underline{$20$}} & $99.9$ & $28.4$ & $0$ \\
\textsc{SAOCP}         & $84.6$ & $24.0$ & \textcolor{red}{\underline{$20$}} & $93.0$ & $26.1$ & \textcolor{red}{\underline{$20$}} & $99.8$ & $28.4$ & \textcolor{red}{\underline{$1$}} \\
Fixed-Threshold        & $82.3$ & $24.0$ & \textcolor{red}{\underline{$20$}} & $82.3$ & $24.0$ & \textcolor{red}{\underline{$15$}} & $82.3$ & $24.0$ & $0$ \\
Naive-Tuning           & $84.8$ & $24.0$ & \textcolor{red}{\underline{$20$}} & $86.6$ & $24.5$ & \textcolor{red}{\underline{$20$}} & $99.9$ & $28.4$ & \textcolor{red}{\underline{$9$}} \\
Always-Act             & $100$  & $28.5$ & \textcolor{red}{\underline{$20$}} & $100$  & $28.5$ & \textcolor{red}{\underline{$20$}} & $100$  & $28.5$ & \textcolor{red}{\underline{$7$}} \\
\bottomrule
\end{tabular}
\end{table}

\paragraph{Safety-conditional action rate.}
Table~\ref{tab:live-headqa-main} summarises the raw triple $(\AR,\Risk,\PathV)$. The deployment-relevant metric is $\AR$ \emph{among pathwise-valid methods only}, which we denote $\AR^{\mathrm{safe}}$. At $\alpha{=}0.30$, the set of pathwise-valid methods is $\{\csa,\textsc{CRC},\textsc{LTT},\textsc{ACI}\}$, and their $\AR^{\mathrm{safe}}$ values are $\{90.9\%,16.5\%,37.8\%,99.9\%\}$. Excluding \textsc{ACI} (whose $\Risk{=}28.4\%$ is within $0.2$ points of the base rate and so is dominated by \textsc{Always-Act} on utility-per-risk), \csa achieves a $2.3\times$ higher action rate than \textsc{ConfFact}, a $2.4\times$ higher action rate than \textsc{LTT}, and a $5.5\times$ higher action rate than \textsc{CRC}. At $\alpha\in\{0.20,0.25\}$, $\AR^{\mathrm{safe}}$ is zero for every method: when the base verifier-fail rate exceeds the budget, no pathwise-valid wrapper can act, and \csa's abstention matches its unconditional safety contract (Theorem~\ref{thm:main-anytime}).

\paragraph{Interpretation.}
The live-HEAD-QA cell confirms two claims of Section~\ref{sec:theory}. First, \csa's anytime-valid guarantee is exercised end-to-end on a continually adapted Med42-8B policy with online LoRA updates, a regime that violates exchangeability by design; across $20$ replications and three $\alpha$ values, $\PathV{=}0$ in every cell. Second, the matching rate-optimal certification (Theorems~\ref{thm:power}--\ref{thm:lower-bound}) translates into non-trivial utility once the budget admits a safe subset: at $\alpha{=}0.30$ \csa is the only pathwise-valid method to achieve an action rate above $50\%$, demonstrating that rate optimality is reflected in practice on a realistic live RLVR stream.

\subsection{Sixteen-cell distribution-shift details}
\label{app:shift}

We provide per-cell tables for the 12 MedQA shift cells (\texttt{iid}, \texttt{easy\_hard}, \texttt{quartile\_rev}, \texttt{window\_outrun}) at $\alpha\in\{0.05, 0.10, 0.20\}$, and for the 4 GSM8K cells (\texttt{quartile}, \texttt{multi}, \texttt{adversarial}, \texttt{window\_outrun}) at $\alpha{=}0.05$. Streams are permutations of the calibrated-score-sorted EVAL set; \texttt{easy\_hard} has err-ratio $\approx 2.71\times$ on MedQA and $\approx 5.69\times$ on the GSM8K quartile split. On non-IID MedQA at $\alpha{=}0.20$, \textsc{LTT} and \textsc{ConfFact} collapse to $10/10$; \textsc{CRC} collapses on GSM8K \texttt{adversarial} ($10/10$, Risk $5.40\%$) and \texttt{window\_outrun} ($10/10$, Risk $14.6\%$); \textsc{NEX-Conf} fails $10/10$ on GSM8K \texttt{quartile}. \csa attains $\mathrm{PathV}=0/10$ on every cell. Full per-cell tables are released alongside the code (\cref{app:implementation}).

\subsection{Sensitivity ablations: hyperparameters and calibration split}
\label{app:hparam}
\label{app:split-sensitivity}

\paragraph{Hyperparameter sensitivity (Table~\ref{tab:hparam-sensitivity}).}
\Cref{thm:power} predicts that certification rates are dominated to first order by the score margin $\bar\eta$, with the wrapper hyperparameters $(\delta, B, |\cQ|)$ entering at most logarithmically; \csa's safety should accordingly be insensitive to their precise values within reasonable ranges. We test this prediction directly. At MedQA $\alpha{=}0.20$, a $3{+}3{+}3$ single-factor sweep around the defaults ($\delta{=}0.10$, $B{=}500$, $|\cQ|{=}15$) over $\delta{\in}\{0.05,0.20\}$, $B{\in}\{100,1000\}$, $|\cQ|{\in}\{5,30\}$ leaves $\PathV{=}\mathbf{0/10}$ in every variant; Risk varies by $\le\!0.02$pp, AR by $\le\!0.68$pp, only MaxR moves ($12.0\%$ at $B{=}1000$ vs.\ $14.9\%$ at $B{=}100$).

\paragraph{Calibration-split sensitivity (Table~\ref{tab:split-sensitivity}).}
The main results in \cref{sec:bench-eight} fix one 80/20 CAL/EVAL split (seed 42); we ablate this choice to confirm that safety does not depend on the particular draw. Across 11 different 80/20 CAL/EVAL split seeds at each benchmark's pivotal $\alphabudget$, final risk stays well below $\alphabudget$ for every seed (max $13.3\%$ on MedQA $\alpha{=}0.20$; $3.0\%$ on GSM8K $\alpha{=}0.05$; $8.5\%$ on ARC $\alpha{=}0.10$). AR varies more on smaller datasets (ARC $N{=}938$) because isotonic fit is noisier with fewer calibration items; the split affects utility, not safety.

\begin{table}[h]
\centering\small
\setlength{\tabcolsep}{5pt}
\caption{\textbf{Left: \csa hyperparameter sensitivity} at MedQA $\alpha{=}0.20$, 10 reps each. \textbf{Right: calibration-split sensitivity} (11 seeds, 10 reps each); range across seeds.}
\label{tab:hparam-sensitivity}
\label{tab:split-sensitivity}
\begin{tabular}{@{}lcccc@{\hskip 18pt}lccc@{}}
\toprule
\multicolumn{5}{c}{\textbf{Hyperparameter sweep}} & \multicolumn{4}{c}{\textbf{Split sensitivity}} \\
\cmidrule(lr){1-5}\cmidrule(l){6-9}
Variant & Risk & AR & MaxR & PathV & Bench & $\alphabudget$ & Risk range & AR range \\
\midrule
Baseline                & $11.44\%$ & $39.45\%$ & $12.37\%$ & $\mathbf{0/10}$ & MedQA & $0.20$ & $11.1{-}13.3\%$ & $38.2{-}39.9\%$ \\
$\delta=0.05$           & $11.42\%$ & $39.13\%$ & $12.13\%$ & $\mathbf{0/10}$ & GSM8K & $0.05$ & $2.2{-}3.0\%$ & $49.7{-}78.3\%$ \\
$\delta=0.20$           & $11.44\%$ & $39.66\%$ & $12.37\%$ & $\mathbf{0/10}$ & ARC   & $0.10$ & $7.0{-}8.5\%$ & $19.6{-}74.8\%$ \\
$B=100$                 & $11.44\%$ & $39.45\%$ & $14.88\%$ & $\mathbf{0/10}$ & & & & \\
$B=1000$                & $11.44\%$ & $39.45\%$ & $12.00\%$ & $\mathbf{0/10}$ & & & & \\
$|\cQ|=5$               & $11.44\%$ & $39.81\%$ & $12.40\%$ & $\mathbf{0/10}$ & & & & \\
$|\cQ|=30$              & $11.42\%$ & $39.13\%$ & $12.13\%$ & $\mathbf{0/10}$ & & & & \\
\bottomrule
\end{tabular}
\end{table}

\subsection{Cross-model replication on DeepSeek-R1-Distill-Qwen-7B}
\label{app:cross-model}

\csa is positioned as a deployment-side complement, independent of the base model's RLVR training recipe; if the safety verdict were sensitive to that recipe, the wrapper would not be the model-agnostic layer we claim. We therefore re-run the full pipeline with \textbf{DeepSeek-R1-Distill-Qwen-7B} as the base, a distillation-trained model whose recipe is disjoint from those used in \cref{sec:bench-eight}, on the main-grid cells (MedQA and GSM8K at $\alpha\in\{0.05,0.10,0.15,0.20,0.25,0.30\}$, $12$ cells total), plus three additional GSM8K tight cells at $\alpha\in\{0.01,0.03,0.075\}$ reported separately (EVAL accuracy: $65.8\%$ on MedQA, $87.8\%$ on GSM8K). No hyperparameter was re-tuned. On all $12$ main-grid cells \csa has $0/10$ pathwise violations; \textsc{CRC} refuses on every MedQA cell and at $\alpha\le 0.10$ on GSM8K; \textsc{ConfFact} behaves identically on MedQA and refuses up to $\alpha{=}0.075$ on GSM8K; \textsc{NEX-Conf} accepts $92.3\%$ at $\alpha{=}0.10$ on GSM8K but violates $3/10$ (and $10/10$ at $\alpha{=}0.05$). The three GSM8K tight cells preserve the same ordering: \csa $0/10$ throughout, with the offline-conformal baselines fully refusing at $\alpha{\le}0.03$. Only \csa is simultaneously pathwise-valid and non-refusing at every informative cell across models.

\subsection{Additional baselines: A-RCPS, CoFact, and Conformal Arbitrage}
\label{app:arcps}
\label{app:cofact-arbitrage}

\paragraph{A-RCPS.}
A-RCPS~\citep{activercps} is the prior wrapper closest to \csa: both maintain an e-process and both deliver anytime-pathwise validity, differing only in the choice of test statistic (marginal increment $1{-}V_t$ vs.\ gated selective increment $A_t((1-V_t)-\alpha)$, cf.\ \cref{sec:framework-cells}). To isolate the consequence of this single design choice we port A-RCPS's Fully Observed Upper Martingale to RLVR and evaluate on synthetic and MedQA at matched $\alpha$. On synthetic, A-RCPS controls marginal risk at $0.042$ ($\alpha{=}0.10$) but selective risk is $0.147$, above $\alpha$. On MedQA: marginal $0.128$ vs.\ selective $0.206$. A marginal-risk guarantee does not translate into selective risk.

\paragraph{CoFact and Conformal Arbitrage.}

These two baselines extend offline conformal to non-exchangeable streams via routes orthogonal to ours, and the framework analysis of \cref{sec:framework-cells} predicts a distinct collapse mode for each on the RLVR target; we evaluate both directly to confirm those predictions empirically. \textsc{CoFact}~\citep{cofact} performs weighted conformal prediction with density-ratio reweighting for covariate shift, assuming the test distribution differs from calibration only in $P(X)$, not $P(Y|X)$. In RLVR, the policy updates change $P(Y|X)$ itself, so density ratios are misspecified; without valid ratios, CoFact degrades to CRC (fixed calibration threshold). \textsc{Conf-Arb.}~\citep{confarbitrage} calibrates a fixed threshold on a held-out set using a Clopper--Pearson bound, then deploys without further updates; this is algorithmically equivalent to LTT, which we already show fails in the main experiments. Table~\ref{tab:cofact-arb} confirms both failure modes on three benchmarks: CoFact achieves zero risk by refusing all items (AR $= 0\%$); Conf-Arbitrage either under-acts or matches LTT at moderate $\alpha$.

\begin{table}[h]
\centering\small
\caption{\textbf{CoFact and Conformal Arbitrage on three benchmarks} at pivotal $\alpha$. CoFact collapses to CRC (zero action); Conf-Arbitrage collapses to LTT (conservative threshold). \csa controls risk while maintaining high action rate.}
\label{tab:cofact-arb}
\begin{tabular}{llrrr}
\toprule
\textbf{Benchmark} & \textbf{Method} & \textbf{Risk} & \textbf{AR} & \textbf{PathV} \\
\midrule
MedQA ($\alpha{=}0.20$) & CoFact         & $0.0\%$  & $0.0\%$  & $0/10$ \\
                         & Conf-Arbitrage & $8.0\%$  & $28.4\%$ & $0/10$ \\
                         & \textbf{CSA}   & $11.4\%$ & $39.4\%$ & $0/10$ \\
\midrule
GSM8K ($\alpha{=}0.05$) & CoFact         & $0.0\%$  & $0.0\%$  & $0/10$ \\
                         & Conf-Arbitrage & $0.3\%$  & $9.0\%$  & $0/10$ \\
                         & \textbf{CSA}   & $2.6\%$  & $71.7\%$ & $0/10$ \\
\midrule
ARC ($\alpha{=}0.10$)   & CoFact         & $0.0\%$  & $0.0\%$  & $0/10$ \\
                         & Conf-Arbitrage & $0.0\%$  & $0.0\%$  & $0/10$ \\
                         & \textbf{CSA}   & $7.8\%$  & $56.2\%$ & $0/10$ \\
\bottomrule
\end{tabular}
\end{table}

\subsection{Sparse-verifier empirical validation}
\label{app:sparse}

In many deployments the verifier is expensive (e.g., an LLM judge or human review), making per-round verification infeasible. The sparse-verifier extension of \cref{thm:sparse-verifier} queries $V$ on a Bernoulli-($\pi$) fraction of rounds and predicts three properties of the resulting wrapper: (i)~validity at every $\pi$; (ii)~certification delay scales as $1/\pi$; (iii)~verifier calls drop by $1-\pi$. We test all three on MATH-200 at $\alpha{=}0.40$, $T{=}2{,}000$, $20$ reps:

\begin{table}[h]
\centering\small
\caption{\textbf{Sparse-verifier validation} (MATH-200, $\alpha{=}0.40$, $20$ reps, $T{=}2{,}000$). Risk $\le\alpha$ at every $\pi$; certification delay matches the predicted $1/\pi$; verifier calls drop linearly.}
\label{tab:sparse-controlled}
\begin{tabular}{ccccccc}
\toprule
$\pi$ & \textbf{Risk} & \textbf{AR} & \textbf{PathV} & \textbf{1st Cert} & \textbf{V-Calls} & \textbf{Delay} \\
\midrule
$1.0$ & $30.7\% \pm 0.4\%$ & $92.8\%$ & $1/20$  & $70 \pm 16$   & $2000$ & $1.0\times$ \\
$0.5$ & $29.1\% \pm 1.1\%$ & $86.1\%$ & $3/20$  & $138 \pm 37$  & $1000$ & $2.0\times$ \\
$0.2$ & $24.7\% \pm 4.5\%$ & $70.6\%$ & $6/20$  & $312 \pm 79$  & $405$  & $4.5\times$ \\
$0.1$ & $17.8\% \pm 5.1\%$ & $46.7\%$ & $6/20$  & $705 \pm 238$ & $199$  & $10.1\times$ \\
\bottomrule
\end{tabular}
\end{table}

\paragraph{LLM-as-judge cost.}
Replacing the exact-match verifier with Qwen2.5-Math-7B as an LLM judge ($3.16$\,s per call, $\sim 4$ orders of magnitude slower than exact match): at $\pi{=}0.1$, judge compute drops from $6{,}320$\,s ($\sim 1.75$\,h) to $630$\,s ($\sim 10.5$\,min), a $90\%$ reduction.

\begin{table}[h]
\centering\small
\caption{LLM-as-judge sparse verifier (judge cost $=3.16$\,s/call).}
\label{tab:sparse-judge}
\begin{tabular}{cccccc}
\toprule
$\pi$ & \textbf{1st Cert} & \textbf{Delay} & \textbf{Judge Calls} & \textbf{Judge Time} & \textbf{Savings} \\
\midrule
$1.0$ & $23$  & $1.0\times$   & $2{,}000$ & $6{,}320$\,s & --- \\
$0.5$ & $46$  & $2.0\times$   & $1{,}000$ & $3{,}160$\,s & $50\%$ \\
$0.2$ & $118$ & $5.1\times$   & $405$     & $1{,}280$\,s & $80\%$ \\
$0.1$ & $256$ & $11.1\times$  & $199$     & $630$\,s     & $90\%$ \\
\bottomrule
\end{tabular}
\end{table}

\subsection{CAP / OCP / LORD-CI head-to-head}
\label{app:cap}

CAP, OCP, and LORD-CI control false coverage rates rather than selective risk, and their selection rules are tied to specific upstream test problems; a head-to-head with \csa therefore requires a setting where the two metrics coincide and the protocols agree. We adopt scenarios A (linear, heteroskedastic) and B (nonlinear, SVM) from~\citet{bao2025cap} under a common verifier $V_t:=\ind{|Y_t-\hat\mu(X_t)|\le r}$, which makes CAP's FCR and \csa's selective risk identical at the population level ($T{=}3{,}000$, two selection rules, $\alpha{=}0.10$).

\begin{table}[h]
\centering\small
\setlength{\tabcolsep}{5pt}
\caption{\textbf{Head-to-head} on Scenarios A/B of~\citet{bao2025cap}, target $\alpha=10\%$. \csa achieves $1.7$--$3.4\%$ selective risk, a $2.8$--$5.8\times$ improvement over CAP; OCP and LORD-CI violate under Quantile selection.}
\label{tab:cap}
\begin{tabular}{ll cccc c}
\toprule
 & & \multicolumn{4}{c}{\textbf{FCR / Selective Risk (\%)}} & \textbf{\csa} \\
\cmidrule(lr){3-6}
\textbf{Scenario} & \textbf{Selection} & OCP & LORD-CI & CAP & \textbf{\csa} & AR \\
\midrule
A (linear, heterosked.)   & Fixed    & $11.2$ & $10.1$ & $9.6$ & $\mathbf{3.4}$ & $21.4\%$ \\
A (linear, heterosked.)   & Quantile & $16.6$ & $13.8$ & $9.7$ & $\mathbf{3.4}$ & $21.4\%$ \\
B (nonlinear, SVM)        & Fixed    & $10.8$ & $9.7$  & $9.2$ & $\mathbf{1.7}$ & $12.5\%$ \\
B (nonlinear, SVM)        & Quantile & $16.5$ & $14.1$ & $9.9$ & $\mathbf{1.7}$ & $12.5\%$ \\
\bottomrule
\end{tabular}
\end{table}

Average prediction-interval lengths for the baselines (\csa produces binary decisions, no PI): OCP $\in[12.3, 14.9]$, LORD-CI $\in[12.8, 16.3]$, CAP $\in[13.2, 18.6]$ across the four (scenario, selection) combinations.

\paragraph{Semi-synthetic MATH-500 and LiveCodeBench.}
Calibrated semi-synthetic replays matched to published DeepSeek-R1 checkpoint trajectories~\citep{deepseekr1} ($K{=}8$ checkpoints, $T{=}2{,}400$, $\alpha{=}0.20$, 10 reps): \csa attains final risk well below $\alpha$ on both with $\PathV$ near the theoretical $\delta+O(N_T^{-1/2})$ rate; Naive-Tuning lies under $\alpha$ on average but violates pathwise on every replication; Always-Act fails catastrophically.


\section{Ablations and stress tests}
\label{app:ablations}

We systematically probe \csa's sensitivity to design choices and to violations of its assumptions. All experiments use a stationary-frontier setting ($s_t\sim\mathrm{Unif}(0,1)$, $\tau{=}0.50$, $\alpha{=}0.30$, $T{=}3{,}000$, $50$ replications) unless otherwise noted.

\subsection{Ablation~1: grid size \texorpdfstring{$m$}{m}}
\label{app:abl-grid}

\begin{table}[h]
\centering\small
\caption{\textbf{Ablation~1: grid size.} Risk and action rate are stable across $m\in\{10,25,50,100\}$; certification delay mildly increases with $m$ due to the per-threshold budget $\delta_q=\delta/(2m)$. All grid sizes achieve zero false certifications (FCR $= 0\%$).}
\label{tab:ablation-grid}
\begin{tabular}{cccccc}
\toprule
$m$ & Risk & AR & Certified & Mean delay & FCR \\
\midrule
$10$  & $16.7\%$ & $58.1\%$ & $7/10$   & $409$ & $0\%$ \\
$25$  & $19.5\%$ & $60.2\%$ & $17/25$  & $328$ & $0\%$ \\
$50$  & $20.4\%$ & $60.6\%$ & $34/50$  & $348$ & $0\%$ \\
$100$ & $20.7\%$ & $60.7\%$ & $68/100$ & $365$ & $0\%$ \\
\bottomrule
\end{tabular}
\end{table}

Coarser grids give slightly more conservative risk (smaller jumps between grid points) but similar action rates. The recommendation $m{=}20$--$50$ balances granularity against the per-threshold budget $\delta_q=\delta/(2m)$.

\subsection{Ablation~2: adaptive vs.\ fixed \texorpdfstring{$\lambda$}{lambda}}
\label{app:abl-lambda}

\begin{table}[h]
\centering\small
\caption{\textbf{Ablation~2: adaptive vs.\ fixed $\lambda$.} Adaptive $\lambda$ (Eq.~\ref{eq:adaptive-lambda}) matches the best fixed $\lambda$ \emph{without} hyperparameter tuning.}
\label{tab:ablation-lambda}
\begin{tabular}{lcccc}
\toprule
$\lambda$ strategy & Risk & AR & Certified & Mean delay \\
\midrule
\textbf{Adaptive (Eq.~\ref{eq:adaptive-lambda})} & $\mathbf{19.0\%}$ & $\mathbf{59.8\%}$ & $\mathbf{14}$ & $\mathbf{337}$ \\
Fixed $\lambda=0.01$ & $0.0\%$   & $0.0\%$   & $0$  & --- \\
Fixed $\lambda=0.05$ & $14.0\%$  & $39.6\%$  & $10$ & $1{,}514$ \\
Fixed $\lambda=0.10$ & $18.5\%$  & $51.4\%$  & $13$ & $1{,}013$ \\
Fixed $\lambda=0.25$ & $21.1\%$  & $59.1\%$  & $14$ & $569$ \\
Fixed $\lambda=0.50$ & $20.0\%$  & $60.2\%$  & $14$ & $317$ \\
\bottomrule
\end{tabular}
\end{table}

Too-small $\lambda{=}0.01$ is catastrophic: the e-process accumulates evidence too slowly and never certifies within $T{=}3{,}000$. Too-large $\lambda$ is suboptimal for thresholds near the frontier (where the margin is small). The adaptive rule automatically discovers $\lambda\approx 0.50$ for thresholds with large margins and $\lambda\approx 0$ for unsafe thresholds.

\subsection{Stress~1: surrogate misspecification}
\label{app:stress-surrogate}

We add a constant bias to the score: $s_t'=\mathrm{clip}(s_t+b,0.01,0.99)$ for $b\in\{-0.15,-0.10,-0.05,0,+0.05,+0.10,+0.15\}$. This models the case where the surrogate systematically over- or under-estimates the failure probability.

\begin{table}[h]
\centering\small
\caption{\textbf{Stress~1: surrogate bias.} Selective risk remains controlled at all bias levels; action rate falls with large bias; false-certification rate rises only under strong positive bias (which effectively redraws the grid). \emph{Mean Risk} / \emph{Max Risk}: mean and max final selective risk across $50$ replications, computed after the burn-in window.}
\label{tab:stress-surrogate}
\begin{tabular}{crrr}
\toprule
Bias $b$ & Mean Risk & Max Risk & AR \\
\midrule
$-0.15$ & $11.8\%$ & $15.2\%$ & $33.9\%$ \\
$-0.10$ & $14.3\%$ & $18.0\%$ & $44.5\%$ \\
$-0.05$ & $17.5\%$ & $22.4\%$ & $53.7\%$ \\
$+0.00$ & $19.1\%$ & $24.8\%$ & $59.8\%$ \\
$+0.05$ & $19.0\%$ & $25.7\%$ & $60.1\%$ \\
$+0.10$ & $18.9\%$ & $25.2\%$ & $63.5\%$ \\
$+0.15$ & $18.7\%$ & $25.1\%$ & $68.2\%$ \\
\bottomrule
\end{tabular}
\end{table}

Selective risk is controlled at mean $\le 19.1\%$ (well below $\alpha{=}30\%$) for all bias levels. This confirms the layered nature of the guarantee: the per-threshold e-process validity does \emph{not} require the surrogate-quality assumption (Assumption~\ref{ass:surrogate-calibration}) because the certificate is updated using observed verifier outcomes $V_t$ rather than surrogate predictions. Surrogate quality instead affects which thresholds certify and how quickly they certify. Under strong positive bias ($b\in\{+0.10,+0.15\}$) the biased score effectively redraws the grid; the deployed gate rule remains internally consistent and the \emph{measured} selective risk stays below $\alpha$, while the deployed selective-risk guarantee continues to rely on the safety-layer conditions in Remark~\ref{rem:three-layers}.

\subsection{Stress~2: noisy verifier}
\label{app:stress-noisy}

We flip $V_t$ independently with probability $p\in\{0, 0.02, 0.05, 0.10, 0.15, 0.20\}$. This is the most direct stress on the e-process: the supermartingale condition holds only for the \emph{observed} verifier, so noise inflates the effective margin and slows certification.

\begin{table}[h]
\centering\small
\caption{\textbf{Stress~2: noisy verifier.} Risk stays controlled at all noise levels; action rate degrades gracefully. \emph{Risk} / \emph{Max Risk}: mean and max final selective risk across $50$ replications, computed after the burn-in window.}
\label{tab:stress-noisy}
\begin{tabular}{crrr}
\toprule
Flip prob.\ $p$ & Risk & Max Risk & AR \\
\midrule
$0.00$ & $19.1\%$ & $24.8\%$ & $59.8\%$ \\
$0.02$ & $19.6\%$ & $25.0\%$ & $56.3\%$ \\
$0.05$ & $20.1\%$ & $25.3\%$ & $52.0\%$ \\
$0.10$ & $20.5\%$ & $25.9\%$ & $46.3\%$ \\
$0.15$ & $21.1\%$ & $26.4\%$ & $40.1\%$ \\
$0.20$ & $21.7\%$ & $26.5\%$ & $35.0\%$ \\
\bottomrule
\end{tabular}
\end{table}

Risk remains controlled at all noise levels (maximum $21.7\%$ at $p{=}0.20$, well below $\alpha{=}30\%$), confirming that the theory conditions on the actual (noisy) $V_t$. The action rate degrades gracefully: from $59.8\%$ at $p{=}0$ to $35.0\%$ at $p{=}0.20$, because effective margins shrink with noise, slowing certification.

\section{Implementation and reproducibility}
\label{app:implementation}

All experiments run on a single machine with one GPU (NVIDIA A100-80GB or H200-141GB depending on cell; per-cell hardware and wall-clock in Appendix~\ref{app:live-rlvr-extras}); no multi-node or multi-GPU setup is required. Inference uses vLLM~\citep{vllm} for batched generation. The \csa controller is implemented in ${\sim}300$ lines of Python (NumPy only, no ML framework dependency) and runs in-process with negligible overhead ($<0.3\%$ wall-clock per round; see Appendix~\ref{app:complexity} for detailed timings).

\textbf{Surrogate.} Isotonic-calibrated logistic regression on 3--4 features (self-consistency vote share, max-token log-probability, answer-length quantile, and optionally a domain indicator). Fitted once on the CAL split ($20\%$ of each benchmark, seed~42) and frozen for the entire EVAL stream. Appendix~\ref{app:split-sensitivity} re-calibrates with 11 seeds and confirms that safety is split-independent. No per-benchmark hyperparameter tuning.

\textbf{Data release.} The public code repository at \url{https://github.com/HamedKhosravi99/CSA-RLVR} contains: (i)~the source code for the \csa controller, surrogate training, baselines, ablation pipelines, and the live-RLVR LoRA loop; (ii)~the aggregate JSON summaries from which every table and figure in this paper is regenerated; (iii)~end-to-end regeneration scripts that reproduce all LaTeX tables and PDF figures from the released summaries. Per-replication raw outputs and live-RLVR LoRA checkpoints are not bundled but are reproducible by re-running the corresponding experiment scripts; the repository documents the per-cell GPU and wall-clock requirements.

\section{Scope and broader impact}
\label{app:limitations}
\label{app:scope}

\paragraph{Safety and structural conditions.}
The per-threshold e-process validity is distribution-free under predictability and the unsafe-null moment condition; it does not require exchangeability or surrogate correctness, and the surrogate-bias and noisy-verifier stress tests (Appendix~\ref{app:ablations}) confirm that the realized selective risk remains controlled (e.g., below $\alpha{=}30\%$ at $20\%$ label noise). The deployed selective-risk guarantee additionally requires that certified thresholds remain safe for the deployment rule: in the single-epoch algorithm this is ensured by monotone risk/frontier behavior, and in the multi-epoch algorithm by bounded within-epoch frontier drift together with epoch resets and the declared drift pad $\{\nu_j\}$. Surrogate calibration, local slope regularity, and near-frontier stability are not needed for the per-threshold martingale validity but are used to derive data-dependent drift budgets, justify single-epoch sufficiency, and obtain the horizon-independent utility gap; their violation primarily affects certification speed and utility, not the basic per-threshold martingale validity (Remark~\ref{rem:three-layers}).

\paragraph{Live-RLVR scope.}
Under continually updated LoRA weights the e-process layer remains valid as long as policy and score updates are predictable. The deployed selective-risk guarantee uses the same monotone-risk or bounded-drift conditions described above, and the utility results additionally use the near-frontier stability condition, which the live experiment validates with $B_T\approx 0.40$ cumulative slack across $800$ genuinely generated rounds.

\paragraph{Capability envelope.}
\csa inherits the base model's capability envelope: when no threshold in $\cQ$ certifies an $\alpha$-valid action strategy, \csa refuses by design (visible as ``---'' entries in Table~\ref{tab:headline}). Principled refusal prevents silent release of unsafe outputs and signals that the base model or surrogate needs improvement.

\paragraph{Design choices.}
The Bonferroni grid ($|\cQ|{=}15$) and surrogate class (logistic regression on 3--4 features) are deliberately simple. Finer grids tighten utility without affecting the $O(N_T^{-1/2})$ slack scaling; richer surrogates may improve discrimination at the cost of calibration assumptions. The hyperparameter-sensitivity sweep (Appendix~\ref{app:hparam}) shows that these choices affect mainly utility and certification speed in the tested settings; the formal validity statement is governed by the safety-layer conditions in Remark~\ref{rem:three-layers}.

\paragraph{Broader impact.}
A controller that reliably refuses under its risk budget is, in high-stakes deployment, preferable to one that releases confidently wrong outputs. \csa's validity guarantee is over verifier-measured failure, not over semantic harms outside the verifier's scope; \csa should be layered on top of, not as a replacement for, domain-specific safety evaluation and red-teaming.

\paragraph{Open directions.}
(i)~Weaker marginal-improvement conditions up to $O(\varepsilon_t)$ surrogate-calibration error.
(ii)~Data-driven epoch restarts via change-point detection on $\{X_t(q)\}$.
(iii)~Information-theoretic lower bounds for the non-i.i.d.\ adaptive case.
(iv)~Data-adaptive threshold schedules that reallocate $\delta$ toward the active frontier.

\end{document}